\documentclass{amsart}
\usepackage[numbers, compress]{natbib}

\usepackage[utf8]{inputenc} 
\usepackage[T1]{fontenc}    
\usepackage{hyperref}       
\usepackage{url}            
\usepackage{booktabs}       
\usepackage{amsfonts}       
\usepackage{nicefrac}       
\usepackage{microtype}      

\usepackage{tikz}
\usepackage{enumitem}
\usepackage{stfloats}

\usepackage{mathtools}
\usepackage{amsfonts}

\usepackage{xspace}
\newcommand{\norm}[1]{\left\| #1 \right\|}
\usepackage{dsfont}
\newcommand{\RR}{\mathds{R}}
\newcommand{\xv}{\mathbf{x}}

\newcommand{\wv}{\mathbf{w}}
\newcommand{\zv}{\mathbf{z}}
\newcommand{\Xc}{\mathcal{X}}

\newcommand{\Fs}{\mathsf{F}}
\newcommand{\fb}{\mathbf{f}}

\newcommand{\rsf}{\mathsf{r}}
\newcommand{\dsf}{\mathsf{d}}
\newcommand{\cl}{\mathop{\mathrm{cl}}}
\newcommand{\intr}{\mathop{\mathrm{int}}}
\newcommand{\bd}{\mathop{\mathrm{bd}}}
\newcommand{\eg}{e.g.\xspace}
\newcommand{\ie}{i.e.\xspace}
\newcommand{\cf}{c.f.\xspace}
\newcommand{\vs}{vs.\xspace}
\newcommand{\wrt}{w.r.t.\xspace}
\newcommand{\PP}{\mathbb{P}}

\newcommand{\EE}{\mathbb{E}}
\newcommand{\Xb}{\mathbf{X}}

\newcommand{\Rc}{\mathcal{R}}
\newcommand{\Rs}{\mathsf{R}}

\newcommand{\Lc}{\mathcal{L}}
\newcommand{\where}{~\mathrm{where}~}
\newcommand{\yhat}{\hat{y}}
\newcommand{\one}{\mathbf{1}}
\newcommand{\msf}{\mathsf{m}}
\newcommand{\sign}{\mathop{\mathrm{sign}}}
\newcommand{\argmax}{\arg\max}

\newcommand{\var}{\mathrm{VaR}}
\newcommand{\cvar}{\mathrm{CVaR}}
\newcommand{\srm}{\mathrm{SRM}}
\newcommand{\rmd}{\mathrm{d}}
\usepackage{stmaryrd}

\newcommand{\parencite}{\citep}
\newcommand{\textcite}{\citet}

\usepackage{amsthm}
\newtheorem{theorem}{Theorem}

\newtheorem{definition}{Definition}

\usepackage{graphicx}
\graphicspath{{./Figure/}}
\usepackage[skip=0pt]{caption}
\usepackage{subcaption}
\usepackage{multirow}
\usepackage{rotating}

\usepackage{wrapfig}

\usepackage{booktabs} 
\usepackage{thmtools}
\usepackage{thm-restate}

\usepackage{cleveref}
\usepackage{amsmath}

\begin{document}
\title[Trade-off between Minimum and Average Margin]{Understanding Adversarial Robustness: The Trade-off between Minimum and Average Margin}

\author{Kaiwen Wu}
\address{University of Waterloo}
\email{kaiwen.wu@uwaterloo.ca}

\author{Yaoliang Yu}
\address{University of Waterloo}
\email{yaoliang.yu@uwaterloo.ca}

\begin{abstract}
Deep models, while being extremely versatile and accurate, are vulnerable to adversarial attacks: slight perturbations that are imperceptible to humans can completely flip the prediction of deep models. Many attack and defense mechanisms have been proposed, although a satisfying solution still largely remains elusive. In this work, we give strong evidence that during training, deep models maximize the minimum margin in order to achieve high accuracy, but at the same time decrease the \emph{average} margin hence hurting robustness. Our empirical results highlight an intrinsic trade-off between accuracy and robustness for current deep model training. To further address this issue, we propose a new regularizer to explicitly promote average margin, and we verify through extensive experiments that it does lead to better robustness. Our regularized objective remains Fisher-consistent, hence asymptotically can still recover the Bayes optimal classifier. 
\end{abstract}

\maketitle


\section{Introduction}
\label{sec:intro}

Traditionally, training more accurate classifiers has been much of the focus of machine learning research. More recently, as machine learning models start to penetrate into safety-critical applications, their robustness against random or even adversarial manipulations has drawn a lot of attention. The work of \textcite{SzegedyZSBEGF14} demonstrated surprisingly that it is possible to craft very minimal changes to an input image that (a) is clearly not perceivable by humans but (b) can completely flip the predictions of state-of-the-art deep models. Such detrimental perturbations are called adversarial examples and have raised serious concerns on the safety of deep models.

Since the work of \citet{SzegedyZSBEGF14}, a sequence of works have proposed new attack algorithms to generate adversarial examples \citep[\eg][]{GoodfellowSS15, moosavi2016deepfool, CarliniWagner17a,ChenZSYH17,HashemiCK18,DuFYCT18,PapernotMGJCS17,ZantedeschiNR17}, as well as defensive techniques \citep[\eg][]{PapernotMWJS16, cisse2017parseval, MadryMSTV18,GoodfellowMP18,GuoZZC18} to train more robust models. 
There even appears to be an arm race between attack and defense techniques: new defense techniques are shown nonrobust under stronger attacks shortly after being proposed \parencite{AthalyeCW18}. Thus, a line of research focusing on provable certification of (non)robustness has emerged. 
For example, one can use the Lipschitz constant of the network to provide a lower bound on the amount of needed adversarial perturbations \parencite[\eg][]{HeinAndriushchenko17, WengZCYSGHD18, zhang2018efficient, WengZCSHBDD18}. More recently, certification methods base on mixed integer programming (MIP) \citep{tjeng2018evaluating, singh2018robustness} have been proposed to provide more accurate lower bounds. However, due to the NP-hardness of MIP, convex relaxations \parencite{wong2018scaling, WongKolter18, raghunathan2018semidefinite, raghunathan2018certified, SinghGMPV18,GowalDSBQUAMK18} have been popular for scaling these provable certification methods to larger models. 

In this paper, we study adversarial robustness from a margin perspective. The notion of margin distribution  dates back to \citep{Breiman99,SchapireFBL98,garg2002generalization, garg2003distribution}, which provide generalization bounds based on the margin distribution. Later on, algorithms that control margin distribution for better generalization performance have been proposed \citep{zhang2017multi}. In this work, however, we connect the notion of margin with adversarial robustness, and we reveal a surprising trade-off between minimum and average margin in standard deep model training. In particular, the recent result in \cite{soudry2018the} implies that on a linearly separable dataset a linear classifier trained with (stochastic) gradient descent converges to the SVM solution, under mild conditions on the loss function. Since SVM explicitly maximizes the minimum margin, any linear classifier would eventually achieve the same. However, our experiments reveal that the average margin is greatly reduced during training. In other words, models maximize the minimum margin at the expense of reducing the average margin hence becoming more susceptible to adversarial attacks. This surprising phenomenon was also confirmed on linearly nonseparable datasets and a variety of nonlinear classifiers. Interestingly, a few defense methods based on margin maximization have been proposed \citep{lee2018towards,croce2019provable,yan2018deep}, which through our work, should be understood as maximizing the average margin, as opposed to the minimum margin.

To overcome the trade-off between minimum margin and average margin in standard training,
we propose the standard training objective with an average-margin promoting regularizer. We prove that the regularized objective remains Fisher consistent, if the original loss is so. This means as sample size grows, the classifier obtained through optimizing our regularized objective still approaches the Bayes optimal classifier, while potentially is much more robust than standard training. Our regularizer can be easily extended to multiclass and nonlinear classifiers. We confirm the latter point through extensive experiments on two standard datasets and a variety of different models and settings. 

To summarize, we make the following contributions in this work:
\begin{itemize}[leftmargin=*, noitemsep, nolistsep]
	\item We reveal the intrinsic trade-off between minimum margin and average margin;
	\item We propose a new regularizer that explicitly promotes average margin and retains Fisher consistency;
	\item We perform extensive experiments to confirm the effectiveness of our new regularizer.
\end{itemize}


\section{Background}
\label{sec:bg}

We consider the multi-category classification problem. Given a set of $n$ training instances $\{(\xv_i, y_i) \in \Xc \times [c]: i=1, \ldots, n\}$, where the feature domain $\Xc\subseteq \RR^{d}$ and $[c] := \{1, \ldots, c\}$ with $c \geq 2$ the number of categories, we are interested in finding a classifier $h: \Xc \to [c]$, or equivalently, a set  partition $\Fs_1, \ldots, \Fs_c$ of the domain $\Xc$ such that $\cup_k \Fs_k = \Xc$ and for all $k\ne l$, $\Fs_k \cap \Fs_l = \emptyset$. Often in practice, the classifier $h$ is learned through a vector-valued function $\fb: \Xc \to \RR^c$, with the argmax prediction rule employed to predict the label $\hat y$ of a test sample $\xv$ as follows:
	\begin{align}
	\label{eq:pred}
	\hat y = \hat y(\xv) := \argmax_{k=1,\ldots, c} f_k(\xv), \qquad
	f_k: \Xc \to \RR, \qquad \fb = (f_1, \ldots, f_c).
	\end{align}
	As is well-known, the two representations of a classifier, \ie either through a set partition $\{\Fs_k\}$ or through a vector-valued function $\fb$, are equivalent. 
	We will use both representations interchangeably.
	
	In their seminal work, \textcite{SzegedyZSBEGF14} defined the robustness of a classifier $\fb$ on a test sample $\xv$ as the minimum perturbation needed to flip its prediction (\cf \eqref{eq:pred}):
	\begin{align}
	\label{eq:rob_org}
	\rsf(\xv) = \inf \{ \|\zv\|: \xv+\zv\in\Xc, ~
	\hat y(\xv+\zv) \ne \hat y(\xv)\},
	\end{align}
	where $\|\cdot\|$ is an abstract norm (\eg the familiar $\ell_p$ norm) for measuring the amount of perturbation. This notion of robustness turns out to have a very intuitive geometric meaning, as we detail below.
We remind that we also represent a classifier $\fb$ as a set partition $\{\Fs_1, \ldots, \Fs_c\}$.

For any $\xv \in \Xc$, define $\yhat(\xv) \in [c]$ so that $\xv \in \Fs_{\yhat(\xv)}$, \ie $\yhat(\xv)$ is the predicted label of our classifer. 
We define the distance from a point $\xv\in\Xc$ to a set $\Fs\subseteq\Xc$ as:
\begin{align}
\label{eq:distance}
\dsf(\xv, \Fs) := \inf\{ \|\xv - \zv\|: \zv\in \Fs\}.
\end{align} 
Recall that the notations $\cl \Fs$, $\intr\Fs$, $\bar \Fs$ and $\bd\Fs$ denote the closure, interior, complement and the boundary\footnote{Of course, $\bd\Fs := \cl \Fs \setminus \intr \Fs$.} of a point set $\Fs$, respectively. We remark that for any set $\Fs$:
\begin{align}
\label{eq:dist}
\dsf(\xv, \Fs) = \begin{cases}
0,& ~ \mbox{ if } \xv \in \cl \Fs \\
\dsf(\xv, \bd \Fs),& ~ \mbox{ otherwise }
\end{cases}
,
\end{align}
which will be important when we aim to \emph{optimize} some function of the distance.

Note that $\dsf(\xv, \Fs_{\hat{y}})$ is the distance to the decision boundary, which captures the definition  \eqref{eq:rob_org} exactly:
\begin{restatable}{proposition}{geometric}
\label{thm:geometric}
Let $\fb: \Xc \to \RR^c$ be a vector-valued function that, if augmented with the argmax prediction rule \eqref{eq:pred}, leads to the following sets (classifier): for all $k\in[c]$, 
 \begin{align}
 \label{eq:class_set}
\Fs_k := \{\xv\in \Xc: f_k(\xv) = \max_{l=1,\ldots,c} f_l(\xv) \},
\end{align}
where we break ties arbitrarily. Then the robustness definition in \Cref{eq:rob_org} equals to $\dsf(\xv, \Fs_{\hat{y}})$.
\end{restatable}

We are now ready to define the \emph{individual} margin of a sample pair $(\xv, y)$ \wrt the classifier $\{\Fs_k\}$ as:
\begin{equation}
\label{eq:rob}
\msf(\xv, y) := \msf(\xv, y; \{\Fs_k\}) := \sign(\hat y(\xv), y) \cdot \dsf(\xv, \bd \Fs_{\hat{y}}), 
\end{equation}
where we define $\sign(\hat y, y) = 1$ if $\hat y = y$ and $\sign(\hat y, y) = -1$ otherwise. Namely, when predicting correctly, $\msf(\xv, y)$ is the distance to the decision boundary, \ie the minimum perturbation needed to change the prediction; whereas when predicting wrongly, it is the \emph{negation} of the distance to the decision boundary, \ie the minimum perturbation needed to correct the prediction. We then define
\begin{itemize}[leftmargin=*, noitemsep, nolistsep]
\item \textbf{minimum margin:} $\min_{1 \leq i \leq n} \msf(\xv_i, y_i)$
\item \textbf{average margin:} $\frac{1}{n}\sum_{i = 1}^{n} | \msf(\xv_i, y_i) |$, simply the average distance to the decision boundary.
\end{itemize}
Although the definition \eqref{eq:distance} applies for any norm, we will focus on $l_2$ margin later, as some current results of implicit minimum margin maximization \citep{soudry2018the, gunasekar2018implicit, ji2018gradient} only holds for $l_2$ norm, which will be discussed in \S\ref{sec:study}. From now on, we use $\| \cdot \|$ to denote the Euclidean norm.

\if01

\section{An axiomatic definition of robustness}
\label{sec:def}
Having defined the margin of a classifier $\{\Fs_k\}$ \wrt any individual sample $\xv$, we now  axiomatically formalize the notion of (adversarial) robustness of a classifier \wrt a population. Our definition unifies and extends the existing ones, and provides a common framework for future comparisons in this field.

Let $(\Xb, Y)$ be a pair of random variables following some distribution $\PP$, then $\msf_l(\Xb)$ is a random quantity as well. While it is common to simply average the individual margins over the distribution $\PP$, we propose the following axiomatic definition that is more flexible. 

First, we introduce an individual robustness measure $\Rs: \RR^c \times [c] \to \RR$ such that $\Rs(\msf(\xv), y)$ aggregates the individual margins. A simple choice that has been used (implicitly) in previous works is 
\begin{align}
\Rs(\msf(\xv), y) = \max\{ \msf_y(\xv), 0 \},
\end{align}
\ie, when predicting correctly we gain a reward  proportional to the minimum perturbation needed to change the correct prediction while we gain 0 reward when predicting incorrectly.
However, in many applications misclassifying an instance $\xv$ with true label $y$ into different classes $l$ may result in totally different consequences. For instance, classifying a yellow light as red is much less severe than classifying a red light as green. This asymmetry in misclassifications can be reflected in an appropriately chosen individual robustness measure $\Rs$. 

Second, we introduce a population robustness measure $\varrho$ that aggregates the individual robustness over the population. 
Let $\Lc$ denote the space of  real-valued random variables and consider the functional $\varrho : \Lc \to \RR\cup\{\infty\}$ that yields the population robustness measure:
\begin{align}
\Rc(\{\Fs_k\}; \PP):= \varrho\big(\Rs(\msf(\Xb), Y)\big), \where (\Xb, Y)\sim\PP,
\end{align}
It is natural to require the following axioms on the robustness measure $\varrho$:
\begin{itemize}[leftmargin=*]
\item (L)aw invariant: If $U$ and $V$ follow the same distribution, then $\varrho(U) = \varrho(V)$.
\item (M)onotonic: If $U \geq V$, then $\varrho(U) \geq \varrho(V)$.
\item (N)ormalized: $\varrho(0) = 0$.
\end{itemize}
These properties are self-evident; any reasonable robustness measure should satisfy them.
Additionally, in some applications we might also be interested in forcing some extra ones, such as
\begin{itemize}[leftmargin=*]
\item (C)onvex: For all $0\leq \lambda \leq 1$, $\varrho(\lambda U + (1-\lambda)V) \leq \lambda \varrho(U) + (1-\lambda)\varrho(V)$,
\item (T)ranslation equivariant: For any real value $r \geq 0$ and random variable $U$, $\varrho(U + r) = \varrho(U) + r$.
\end{itemize}
The convex property basically means the population robustness measure $\Rc$ is smaller when $\Xb$ follows a mixture distribution than when $\Xb$ follows each of the component distributions. Translation equivariance  means if we (uniformly) increase each individual reward $\Rs(\msf(\xv), y)$ then we increase the population robustness measure by the same amount. 

With all of the above properties we call $\varrho$ a convex robustness measure, which corresponds to essentially the convex risk measure in finance \parencite[Chapter 4]{FollmerSchied16}, see also the upper expectation in \textcite[Chapter 10]{HuberRonchetti09}.
%
We mention a few examples of population robustness measures: 
\begin{align}
\var_\alpha &:= \inf\{ t: \PP( \Rs(\msf(\Xb), Y) \leq t ) \geq \alpha\} \\
\cvar_\alpha &:= \EE( \Rs(\msf(\Xb), Y) | \Rs(\msf(\Xb), Y) \leq \var_\alpha )\\
&= \frac{1}{\alpha} \int_{0}^\alpha \var_\gamma \rmd \gamma\\
\label{eq:srm}
\srm &:= \int_{0}^1 \cvar_\alpha \rmd \mu(\alpha) = \int_{0}^1 \var_\gamma \cdot s(\gamma)\rmd \gamma,
\end{align}
where $\mu$ is a probability distribution and $s$ is a density function. Clearly, $\var_\alpha$ and $\cvar_\alpha$ are special cases of the Spectral Robustness Measure $\srm$, with $s(\gamma) = \delta_{\alpha}(\gamma)$ and $s(\gamma) = \tfrac{1}{\alpha}\one_{[0,\alpha]}(\gamma)$, respectively. Note that $\var_\alpha$ is the $\alpha$-quantile, \ie, the robustness value $t$ such that 100$\alpha$\% of the data distribution have robustness measure smaller than $t$. $\srm$ is a convex robustness measure iff the weighting function $s$ is decreasing, \ie, we put more weights on smaller robustness values (since these are easier to be adversarially perturbed). In particular, $\cvar$ is a convex robustness measure while $\var$ is not. Another popular choice is to set $s \equiv 1$, in which case we recover the average robustness measure $\EE[ \Rs(\msf(\Xb), Y) ]$, which has been predominantly used so far in the field of adversarial machine learning. We argue that in certain applications it may be crucial to use a tailored weighting function $s$, based on domain knowledge. We note that variants of $\var_\alpha$ and $\cvar_\alpha$ (where $\var_\alpha$  instead of $\alpha$ is fixed) have been used in \parencite{BastaniILVNC16}  to measure robustness. Notice that another measure, the robust accuracy $\EE\left[ \min_{\norm{\delta} \leq \epsilon} \one\left[\hat Y(\Xb + \delta) = Y\right]\right]$ used in \cite{MadryMSTV18}, is also included in our framework as $\PP[ \msf(\Xb) \geq \epsilon ]$, although it is not convex.
 
 \fi


\section{Minimum \vs Average Margin: A Case Study}
\label{sec:study}

In this section, we study a simple example to demonstrate the trade-off between minimum and average margin. 
In statistical learning theory, it is well-known that we can bound the generalization error of a classifier using empirical margins \cite{KoltchinskiiPanchenko02}. In particular, the celebrated support vector machines (SVM) explicitly maximize the minimum margin (on any linearly separable dataset). In adversarial learning, however, we argue that average margin is more indicative of the robustness of a classifier. 

Our main observations in this section are: (a) Current machine learning models, especially when they become deep and powerful, implicitly maximize the minimum margin to achieve high accuracy; (b) there appears to be an inherent trade-off between minimum margin and average margin. In particular, by maximizing the minimum margin, the model also (unconsciously) minimizes the average margin, hence becomes susceptible to adversarial attacks. Note that we do not claim minimum margin always contradicts average margin. It does not, as can be easily shown through carefully constructed toy datasets (\eg a hightly symmetric dataset). Our empirical observation is that there does appear to be some tension between the two notions of margin on real datasets.

To begin with, we first demonstrate that minimum margin is maximized during standard training. In fact, this can be formally established through the  implicit bias of optimization algorithms, such as (stochastic) gradient descent which is ubiquitously used in training deep models. 
Recall that on a (linearly) separable dataset, (hard-margin) linear SVM explicitly maximizes the minimum margin. The following result, due to \citet{soudry2018the}, confirms the same for most models currently used in machine learning, including logistic regression.


\begin{theorem}[\citep{soudry2018the}]
\label{thm:implicit_bias}
For almost all linearly separable binary datasets and any smooth decreasing loss with an exponential tail, gradient descent with small constant step size and any starting point $\wv_0$ converges to  the (unique) solution $\widehat{\wv}$ of hard-margin SVM, \ie
$
\lim_{t \to \infty} \frac{\wv_t}{\norm{\wv_t}} = \frac{\widehat{\wv}}{\norm{\widehat{\wv}}}.
$
\end{theorem}
Note that the margin of a linear classifier $f(\xv) = \wv^\top \xv$, with decision boundary $\bd \Fs := \{ \xv: \wv^\top \xv = 0 \}$, can be computed in closed-form:
$
\dsf(\xv, \bd\Fs) = 
\tfrac{|\wv^\top \xv|}{\|\wv\|}$,
which only depends on the direction of the weight vector $\wv$. Thus, \Cref{thm:implicit_bias} implies in particular that if we optimize logistic regression by gradient descent on a linearly separable dataset, then we implicitly maximize the minimum margin, just as in SVM. Besides linear classifiers, the same implicit bias towards maximizing the minimum margin has also been discovered  for deep networks \cite{gunasekar2018implicit, ji2018gradient}.

\begin{figure}[t]
\centering
\begin{subfigure}{0.28\textwidth}
\centering
\includegraphics[width = \textwidth, height = 0.28\textheight]{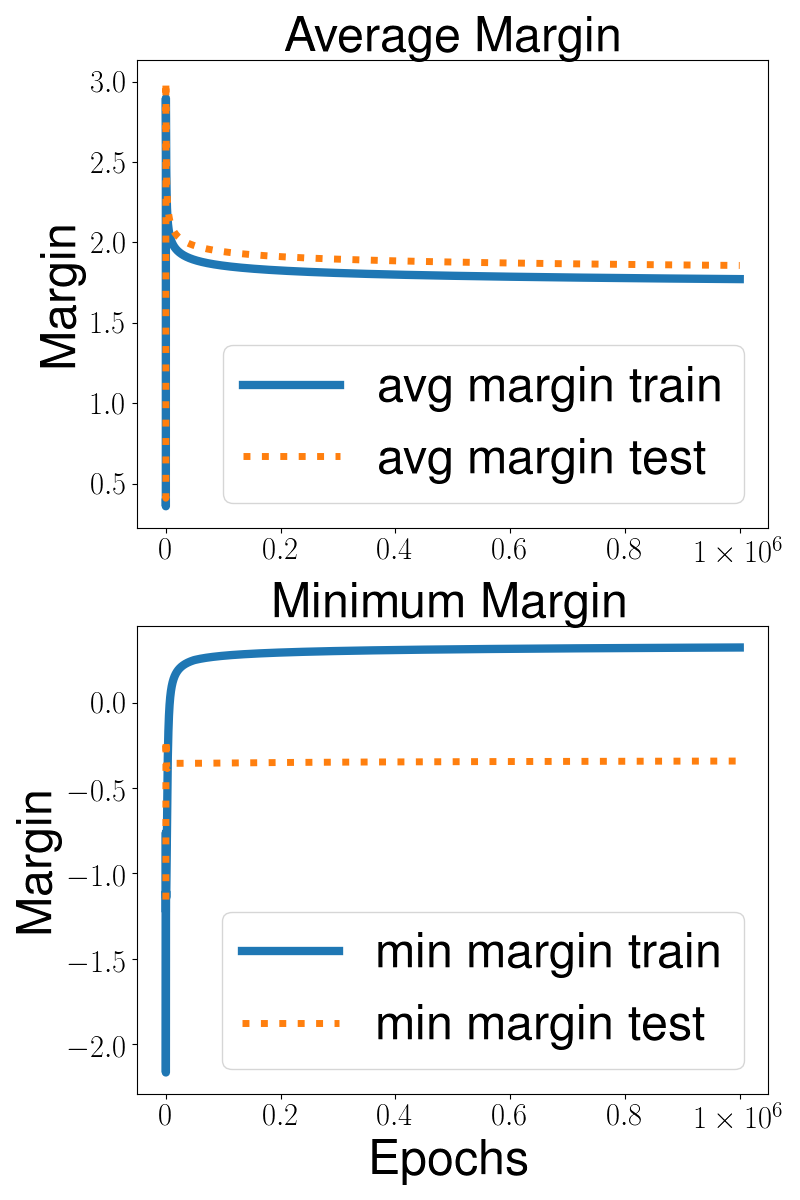}
\caption{Margins during training of LR. \\ \textbf{Upper:} Average margin on training and test set. \\ \textbf{Lower:} Minimum margin on training and test set.}
\label{fig:lr_dist}
\end{subfigure}
~
\begin{subfigure}{0.68\textwidth}
\begin{subfigure}{\textwidth}
\centering
\includegraphics[width = \textwidth, height = 0.09\textheight]{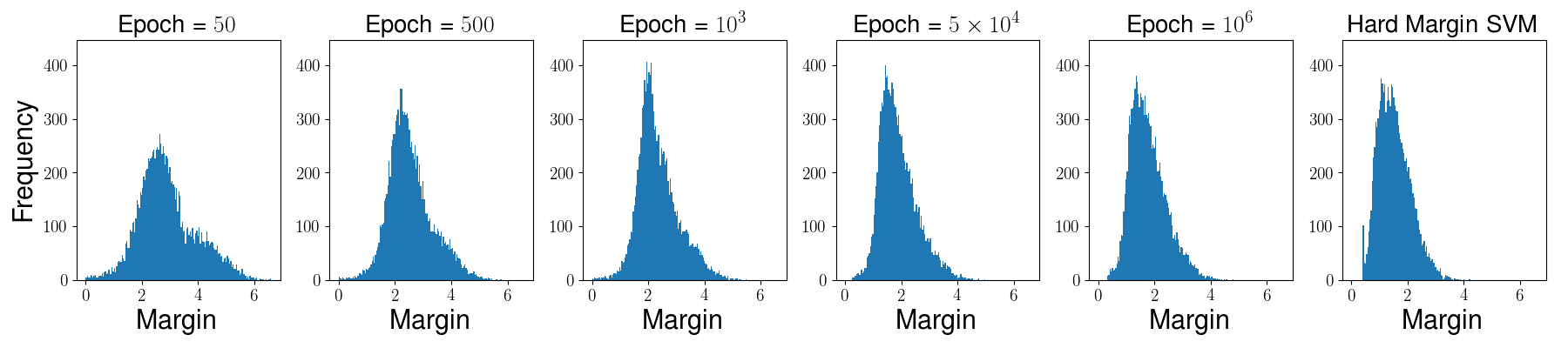}
\caption{Margin histograms on the training set. \textbf{First 5 plots:} Margin histograms of LR during training. \textbf{Last plot:} Margin histogram of SVM, which LR margins converge to. During training, the histogram continues shifting towards left.}
\label{fig:lr_hist}
\end{subfigure}

\begin{subfigure}{\textwidth}
\centering
\includegraphics[width = \textwidth, height = 0.12\textheight]{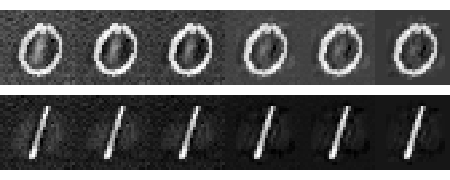}
\caption{Visualization of adversarial examples at different epochs during training. \textbf{First row:} adversarial examples for $0$. \textbf{Second row:} adversarial examples for $1$. Adversarial examples in the same column are generated in the same epoch. 
Adversarial examples gradually become imperceptible. Zoom the figure for better visualization.}
\label{fig:lr_adv}
\end{subfigure}
\end{subfigure}

\end{figure}

It is then natural to ask the next question:
\textbf{How does average margin change during training?}
To answer this question, we train a binary logistic regression (LR) using gradient descent on MNIST to discriminate $0$'s from $1$'s. Note that the subset of MNIST consisting of only $0$'s and $1$'s is indeed linearly separable, as LR achieves zero training error (see \Cref{fig:lr_loss} in appendix). All conditions of \Cref{thm:implicit_bias} are thus satisfied, and we expect LR to maximize minimum margin. Indeed, \Cref{fig:lr_dist} confirms that the minimum margin continues to increase during training until it approaches that of hard-margin SVM, as predicted by \Cref{thm:implicit_bias}. Meanwhile, the average margin decreases drastically after a few epochs at the very beginning, and then keeps decreasing. Interestingly, the minimum margin continues to increase while the average margin continues to decrease even after the training error reaches zero, which means the decision boundary   still changes even after training error diminishes.

To gain further insight, in \Cref{fig:lr_hist} we plot the histogram of margins at different training epochs. We observe that the margin distribution shifts towards the left during training and eventually approaches that of hard-margin SVM. In other words, during training the majority of data is pushed towards the decision boundary, leading LR to become more and more vulnerable to adversarial attacks. Indeed, \Cref{fig:lr_adv} confirms this by visualizing the adversarial examples constructed at different training epochs. Note that for a linear classifier the adversarial example with minimum perturbation can be explicitly determined as
\begin{equation}
\label{eq:linear_adv}
\textstyle
\xv^{adv} = \xv - \frac{\wv^\top\xv}{\|\wv\|^2} \wv ,
\end{equation}
which is the point on the decision boundary that is closest to the training example $\xv$.
As shown in \Cref{fig:lr_adv}, the adversarial examples gradually become more and more imperceptible, indicating that LR becomes more and more vulnerable to adversarial attacks.

A few remarks are in order. First, we have carefully designed our experiment so that (a) the theoretical results in \Cref{thm:implicit_bias} apply; (b) the margins and adversarial examples can be explicitly computed. However, similar phenomenon is also observed for deep models on different datasets where the conditions of \Cref{thm:implicit_bias} may be violated or the margins can only be approximately computed; see \S\ref{sec:exp} for more of these experiments. Second, we remark that the decrease of average margin cannot be caused by overfitting. This is because the test accuracy continues to decrease during training (see training curves in  \Cref{fig:lr_loss} of the appendix). Moreover, we observe that the average margin decreases on \emph{both} training and test sets.
To summarize, we conclude that practically, deep models try to maximize the minimum margin during training at the expense of sacrificing the average margin, hence become susceptible to adversarial attacks. To address this issue, in the next section we propose to explicitly optimize the average margin through appropriate regularization.




\section{An Average Margin Regularizer}
\label{sec:avg}

In this section, we propose a regularization function to explicitly promote average margin. We first handle linear classifiers in \S\ref{sec:blc} through maximizing the average margin in the input space directly. For nonlinear classifiers, maximizing the margin in the input space directly is intractable. Instead, we maximize a lower bound of input space margins in \S\ref{sec:multiclass}, through simultaneously maximizing feature space margin and controlling the Lipschitz constant of the network.

The most straightforward way to improve robustness of a classifier is through explicit regularization. In particular, we consider the following regularized problem:
\begin{equation}
\label{eq:RERM}
\min_{\fb:\Xc \to \RR^c} ~ \frac{1}{n} \sum_{i=1}^n \phi(y_i, \fb(\xv_i)) - \lambda \cdot \frac{1}{n} \sum_{i = 1}^{n} \one_{y_i = \hat{y}_i} \cdot \dsf_\tau(\xv_i, \bd\Fs_{y_i})
\end{equation}
where $\phi$ is the loss function we use to measure the accuracy of our classifier $\fb$, $\{\Fs_k\}$ are the sets induced by $f$ (using the argmax rule)
, $\lambda \geq 0$ is the regularization constant that balances the two objectives, and $\dsf_\tau = \min\{ \tau, \dsf \}$ is the truncated distance. In particular, we only maximize the margin when the classifier makes a correct prediction and the margin is capped at $\tau$ to avoid being dominated by some outliers. We show next how to solve \eqref{eq:RERM} when $\fb$ is a binary linear classifier.

\subsection{Binary Linear Classifier}
\label{sec:blc}
In this section we assume there are two classes, \ie $c=2$ and we consider the linear classifier $f(\xv) = \wv^\top \xv$ (w.l.o.g. we omit the bias term). Correspondingly, let $\Fs_+ = \{ \xv\in \Xc: \wv^\top \xv \geq 0 \}$ and 
the distance 
$
\dsf(\xv, \bd \Fs_+) = \tfrac{|\wv^\top \xv|}{\|\wv\|}.
$
The regularized problem \eqref{eq:RERM} reduces to
\begin{align}
\label{eq:regobj}
\min_{\norm{\wv} = 1} ~~ \sum_{i = 1}^{n}\phi(y_i \wv^\top\xv_i) - \lambda \cdot \sum_{i = 1}^{n}[y_i\wv^\top\xv_i]_0^\tau,
\end{align}
where $[t]_0^\tau = \min(\max(t, 0), \tau)$.
The second regularization term can be written as a difference of two convex functions. Indeed, define $H_{s}(x) = \max(0, s - x)$, then 
\begin{align}
\label{eq:dc}
-[t]_{0}^\tau = H_{\tau}(t) - H_{0}(t) + \tau,
\end{align}
and the objective function simplifies to 
\begin{align}
\min_{\norm{\wv} = 1} ~~ \sum_{i = 1}^{n}\phi(y_i \wv^\top\xv_i) + \lambda \cdot \sum_{i = 1}^{n} [H_\tau-H_0](y_i\wv^\top\xv_i),
\end{align}
where $\lambda \geq 0$ and $\tau \geq 0$ are tuning hyperparameters.

Quite pleasantly, our average margin regularizer not only promotes robustness, it also retains Fisher consistency (aka classification-calibrated), namely that the classifier we obtain by minimizing the \emph{regularized} objective in \eqref{eq:regobj} still approaches the Bayes optimal classifier, as sample size $n$ increases to infinity.
\begin{definition}[Fisher consistency \parencite{bartlett2006convexity}]
Suppose $\phi: \RR \to \RR$ is a loss function and $\eta \in \left[0, 1\right]$. For any $\alpha \in \mathbb{R}$, define the conditional $\phi$-risk as
$
C_{\eta}^{\phi}(\alpha) = \eta \phi(\alpha) + (1 - \eta) \phi(1 - \alpha).
$
We say a loss $\phi$ is Fisher consistent if for any $\eta \neq \frac12$,
\begin{align}
\label{eq:cc}
\inf_{\alpha : \alpha(2\eta - 1) \leq 0}C_{\eta}^{\phi}(\alpha) > \inf_{\alpha \in \RR} C_{\eta}^{\phi}(\alpha) .
\end{align}
\end{definition}
For a Fischer consistent loss functions $\phi$, \eqref{eq:cc} implies that to minimize the $\phi$-risk, $\alpha$ should satisfy $\alpha (2\eta -1) > 0$, \ie the decision $\alpha$ of our classifier should match the decision $\sign(2\eta - 1)$ of the Bayes classifier (which is optimal under the 0-1 loss). Fisher-consistency is a necessary condition for any reasonable loss, if our goal is to approximate the Bayes optimal classifier. The following result confirms the Fisher-consistency of our \emph{regularized} objective \eqref{eq:regobj} (see the proof in appendix).
\begin{theorem}
\label{lma:regularizer_calibrated}
Suppose the loss function $\phi$ is decreasing, continuous, bounded below and $\phi^{\prime}(0) < 0$. Let $\psi = H_{\tau} - H_{0}$ be our average margin regularizer. Then the regularized loss $\ell = \phi + \lambda \psi$ is Fisher consistent for any $\lambda, \tau \geq 0$.
\end{theorem}
The above condition on the loss $\phi$ is reasonable: it basically guarantees $\phi$ itself to be Fisher consistent. Common loss functions, such as the logistic loss, exponential loss and hinge loss, all satisfy this condition. Thus, our average margin regularizer, when combined with these typical, Fisher consistent loss functions, remains Fisher consistent.


\subsection{Extension to Multiclass Deep Models}
\label{sec:multiclass}
For deep neural networks, even computing the margin in the input space is already NP-hard \cite{KatzBDJK17,WengZCSHBDD18}, let alone optimizing it. However, the margin in the feature space provides a lower bound of the margin in the input space. For example, let $\fb$ be a deep neural network with $L$ layers, $\fb(\xv) = W_L \cdot \sigma(W_{L - 1} \cdot \sigma( \cdots \sigma(W_1 \cdot \xv)))$, where $\sigma$ is the activation function. Let $\Phi(\xv)$ be the output of the second last fully connected layer, i.e. $\Phi(\xv) = \sigma(W_{L - 1} \cdot \sigma( \cdots \sigma(W_1 \cdot \xv)))$, then
\begin{equation}
\label{eq:Lip}
\norm{\Phi(\xv_1) - \Phi(\xv_2)}\leq Lip(\Phi) \norm{\xv_1 - \xv_2},
\end{equation}
where $Lip(\Phi)$ is the Lipschitz constant of the feature map $\Phi$. The tightness of this bound is determined by the Lipschitz constant of the model. For example, if the model is $1$-Lipschitz with $Lip(\phi) = 1$, then the margin in the input space $\Xc$ is exactly lower bounded by the margin in the feature space $\left\{\Phi(\xv) : \xv \in \Xc \right\}$. The bound \eqref{eq:Lip} motivates a natural way to optimize the margin in input space: we simultaneously fix the Lipschitz constant and maximize the margin in feature space. Note that this perspective allows us to treat any deep network as first performing a nonlinear feature transform through $\Phi$, and then applying a linear classifier on top.



The Lipschitz constant can be bounded by the product of norms of the weight matrices $\prod_i \norm{W_i}_2$, where $W_i$ is the weight matrix in layer $i$, assuming the activation function is 1-Lipschitz, for instance ReLU \parencite{SzegedyZSBEGF14}. To control the overall  Lipschitz constant $Lip(\Phi)$, ideally, we want to enforce the Lipschitz constant of each layer to be exactly $1$, namely constraining the spectral norm of each weight matrix $W_i$ to be approximately 1. 
Inspired by \cite{cisse2017parseval,lin2018defensive}, we add an orthogonal penalty $\beta \norm{W_iW_i^\top - I}_{\mathrm{F}}^2,
$ into our training objective, 
which encourages the weight matrices to be orthogonal hence having spectral norm close to $1$. For convolutional layers, we first flatten the convolutional filter and then apply the above penalty. 

Once the Lipschitz constant is bounded, we directly maximize the average margin in the feature space, as a computationally efficient lower bound for promoting the average margin in the input space. Since our classifier is linear in the feature space, the distance to the decision boundary can be computed as in \S\ref{sec:blc}:
\begin{equation}
\dsf(\xv, \bd \Fs_{\hat{y}}) = \min_{k \neq \hat{y}} (\wv_{\hat{y}} - \wv_k)^\top \Phi(\xv),
\end{equation}
where $\wv_k$ is the (normalized) weight vector in the final layer for the $k$-th class. Similar to the binary case, we only maximize the margin for correctly classified examples and truncate the margin if it exceeds the threshold $\tau$.
In the end, our regularized objective is,
\begin{equation}
    \sum_{i = 1}^{n} \phi\left(y_i, \fb(\xv_i)\right) - \lambda \left[ \min_{k \neq y_i} (\wv_{y_i} - \wv_k)^\top \Phi(\xv_i) \right]_{0}^{\tau} + \beta \sum_{1 \leq l \leq L} \| W_lW_l ^ \top - I \|_{\mathrm{F}}^2,
\end{equation}
where the first term is the standard training loss (such as cross-entropy), the third term is the orthogonal constraint for controlling the Lipschitz constant of the network, and the second term is the average margin penalty in the feature space, which also maximizes the input space average margin, provided that the Lipschitz constant of the network is indeed close to unity.

\if01
Distance to decision boundary plays a key role in adversarial robustness. Currently, most definitions of adversarial robustness measure more or less rely on distance from a single data point to decision boundary.

In the previous experiments, we are able to observe that average distance can decrease drastically during training. Thus, it is natural to add a average distance regularizer into objective function.

Previously, we got the following formulation:
\begin{align}
\max_{\|w\|_2 = 1} & ~ \sum_{i = 1}^{n} (y_i \hat{y}_i)_+\\
\text{s.t.} & ~ y_i \hat{y}_i > 0
\end{align}
Change the hard constraint into soft constraint, we get
\begin{equation}
\min_{\|w\|_2 = 1} ~ - \mu \sum_{i = 1}^{n} (y_i \hat{y}_i)_+ + \sum_{i = 1}^{n}(- y_i \hat{y}_i)
\end{equation}
We can view the objective function as a sum of the following loss function:
\begin{equation}
l(\alpha) = 
\begin{cases*}
-\alpha & if $\alpha \leq 0$ \\
-\mu \alpha & if $\alpha > 0$
\end{cases*}
\end{equation}
A more general form, considering truncation is shown below:
\begin{equation}
l(\alpha) = 
\begin{cases*}
\xi + \tau & if $\alpha \leq -\tau$ \\
-\xi - \alpha & if $-\tau < \alpha \leq 0$ \\
\xi - (1 + \mu) \alpha & if $0 < \alpha < \xi$ \\
-\mu \alpha & if $\xi < \alpha\leq \tau$ \\
-\mu \tau & if $\alpha > \tau$ \\
\end{cases*}
\end{equation}
Note that the loss function can be written in a fairly simple form. If we define 
\begin{equation}
H_s(x) = \max(0, s - x)
\end{equation}
Then the loss is equal to
\begin{equation}
H_{\xi}(x) - H_{-\tau}(x) + \mu (H_{\tau}(x) - H_{0}(x)) - \mu \tau
\end{equation}
Ignoring the constant, the loss is
\begin{equation}
H_{\xi}(x) - H_{-\tau}(x)
\end{equation}
while the regularization term is
\begin{equation}
H_{\tau}(x) - H_{0}(x)
\end{equation}
Note that both two terms has close relation to robust truncated hinge loss, which is $H_1(x) - H_0(x)$. Writing the loss in this way makes the decomposition trivial. Moreover, it makes the analysis of calibration concise. By the following lemma, the preceding loss function is calibrated for any $\xi > 0$. The method is to extend the lemma for proving truncated hinge loss (Wu and Liu, 2007), such that it can be applied in our case. (Note that Wu and Liu's proof is incorrect. One counterexample can be $\min(0, -x)$. But it can be fixed by additional assumptions.)

We introduce the definition of Fischer-consistency for loss functions, a.k.a. classification calibrated.
\begin{definition}[Classification calibrated]
Suppose $\phi: R \to R$ is a loss function and $\eta \in \left[0, 1\right]$. For any $\alpha \in R$, define conditional $\phi$-risk as
\begin{equation}
C_{\eta}^{\phi}(\alpha) = \eta \phi(\alpha) + (1 - \eta) \phi(1 - \alpha).
\end{equation}
We say loss $\phi$ is calibrated if for any $\eta \neq \frac12$,
\begin{equation}
\inf_{\alpha : \alpha(2\eta - 1) \leq 0}C_{\eta}^{\phi}(\alpha) > \inf_{\alpha \in R} C_{\eta}^{\phi}(\alpha) .
\end{equation}
\end{definition}
Classification calibrated is a basic condition on loss functions. If a loss function is classification calibrated, then minimizing the surrogate loss over the true distribution will lead to Bayes classifier.

\begin{lemma}[Wu and Liu, 2007]
Suppose a nonincreasing loss function $\phi$ is differentiable at zero and $\phi^\prime (0) < 0$. Then the truncated loss $\psi(x)=\min(\phi(x), \phi(s))$ is calibrated for any $s \leq 0$.
\end{lemma}

By the following lemma, for a wide range of loss functions, including logistic loss, hinge loss and exponential loss, the objective function remains calibrated after adding our regularizer.
\begin{lemma}
Suppose a nonincreasing continuous function $\phi$ is bounded below and $\phi^{\prime}(0) < 0$. $\psi$ is the regularizer $H_{\tau} - H_{0}$. Then $\ell = \phi + \mu \psi$ is calibrated.
\end{lemma}

\begin{proof}
Since both $\phi$ and $\psi$ are bounded below, $\ell$ is also bounded below. Denote $m = \inf C_{\eta}^{\ell}(\alpha)$. We need to show $\inf_{\alpha(2\eta - 1) \leq 0}C_{\eta}^{\ell}(\alpha)$ is strictly greater than $m$ for all $\eta \neq \frac12$. We first consider the case $\eta > \frac12$. For the case $\eta < \frac12$, the proof is similar.

The derivative of $C_{\eta}^{\phi}(\alpha)$ at zero is $(2\eta - 1)\phi^{\prime}(0)$, which is negative. Thus there exists $\delta_1 > 0$ such that any $\alpha \in (0, \delta_1)$ satisfies $C_{\eta}^{\phi}(\alpha) < C_{\eta}^{\phi}(0)$.

The right hand derivative of $\psi$ at zero is negative, thus there exist $\delta_2 > 0$ such that any $\alpha \in (0, \delta_2)$ satisfies. $\psi(\alpha) < \psi(0)$. Notice that $\psi(\alpha)$ is constant when $\alpha \leq 0$. Thus any $\alpha \in (0, \delta_2)$ satisfies $C_{\eta}^{\psi}(\alpha) < C_{\eta}^{\psi}(0)$.

Combining above arguments, there exists a sufficiently small $\alpha_0 > 0$, such that $C_{\eta}^{\ell}(\alpha_0) < C_{\eta}^{\ell}(0)$. Splitting the interval $(-\infty, 0]$ into $(-\infty, -\alpha_0]$ and $[-\alpha_0, 0]$, it is sufficient to show that both
\begin{equation}
\inf_{\alpha \leq - \alpha_0} C_{\eta}^{\ell}(\alpha) > m
\end{equation}
and
\begin{equation}
\inf_{- \alpha_0 \leq \alpha \leq 0} C_{\eta}^{\ell}(\alpha) > m
\end{equation}
hold.

For $\alpha \leq - \alpha_0$,
\begin{align}
C_{\eta}^{\ell}(\alpha) - C_{\eta}^{\ell}(- \alpha) & = (2\eta - 1)(\ell(\alpha) - \ell(-\alpha)) \\
& \geq (2\eta - 1)(\ell(0) - \ell(- \alpha)) \\
& \geq (2\eta - 1)(\ell(0) - \ell(\alpha_0))
\end{align}
Namely, flipping the sign of $\alpha$ can strictly decrease the value of $C_{\eta}^{\ell}(\alpha)$. Thus,
\begin{align}
\inf_{\alpha \leq -\alpha_0} C_{\eta}^{\ell}(\alpha) \geq & (2\eta - 1)(\ell(0) - \ell(\alpha_0)) + m \\
> & m
\end{align}
For $-\alpha_0 \leq \alpha \leq 0$, by continuity there exists a minimizer $\alpha^\star$ for $C_{\eta}^{\ell}(\alpha)$ on this compact set. If $\alpha^\star \neq 0$, then we can again flip the sign of $\alpha$ to get a strictly smaller value. Thus $C_{\eta}^{\ell}(\alpha^\star) > C_{\eta}^{\ell}(- \alpha^\star) \geq m$.

If $\alpha^\star = 0$,
\begin{equation}
C_{\eta}^{\ell}(0) > C_{\eta}^{\ell}(\alpha_0) \geq m
\end{equation}

\end{proof}

By above lemma, we can combine it with most loss functions without changing the calibrated property.

A few points to be highlighted:
\begin{itemize}
\item When $\tau = 1$, the regularizer is \textbf{exactly} the truncated hinge.
\item The regularizer includes accuracy into it. We won't be surprise that mininizing the regularizer alone can lead to a reasonable classifier. If we want to decouple accuracy, we should use 
\item It is the regularizer, \textbf{combining} with the norm constraint $\|w\|_2 = 1$, that leads to a classifier with large average distance. Minimizing the regularizer alone can still lead to hard margin SVM. A concrete counterexample is truncate hinge SVM with sufficiently small $\lambda$:
\begin{equation}
\min ~ \lambda \|w\|_2 + \sum_{i = 1}^{n} H_1(y_i\hat{y}_i) - H_0(y_i\hat{y}_i)
\end{equation}
\end{itemize}

\subsection{Multiclass}
Suppose $x$ is correctly classified as label $c$, then the distance the the decision boundary is
\begin{equation}
\min_{j \neq c}\frac{(w_c - w_j)^T x + b_c - b_j}{\|w_c - w_j\|}
\end{equation}
Here we do a relaxation. We only maximize the numerator, while forcing the denominator bounded. If denote
\begin{equation}
\alpha = \min_{j \neq c}\frac{(w_c - w_j)^T x + b_c - b_j}{\|w_c - w_j\|}
\end{equation}
then again we add a regularization term $H_{\tau}(\alpha) - H_{0}(\alpha)$ to the objective function.

\subsection{Optimization}
Note that our regularizer is nonconvex and nondifferentiable. In addition, the gradient of the regularizer is zero in $(-\infty, 0)$ and $(\tau, +\infty)$, meaning gradient descent cannot perform any update in these intervals. Both increase the difficulty of optimization, even for linear classifiers. However, the regularizer can be written as difference of two convex functions. Thus we use DCA to optimize the objective function. For the inner loop of DCA, we use (stochastic) projected subgradient descent to minimize the subproblem. In practice, it converges very fast, comparable to normal training.

\fi


\section{Experiments}
\label{sec:exp}

In this section, we conduct experiments to (a) verify the minimum-average margin trade-off on a variety of (nonlinear) deep models; (b) demonstrate the effectiveness of our average margin regularizer. First, we train six models with various architectures on MNIST and CIFAR10 to verify the trade-off again. Then, we retrain these models with our average margin regularizer, and confirm that our regularizer indeed can effectively promote the average margin. Finally, we compare our average margin regularizer with other state-of-the-art defense techniques in the literature and find that our regularizer achieves comparable performance in terms of robustness and accuracy.


\subsection{Approximating Distance}
In general, exactly computing the distance to the decision boundary of deep models is intractable \cite{KatzBDJK17,WengZCSHBDD18}. Instead, we use the following Lipschitz lower bound as a reasonable approximation:
\begin{theorem}[\citep{HeinAndriushchenko17}]
\label{thm:approx_margin}
Suppose $\fb: \RR^d \to \RR^c$ is a multiclass classifier (with argmax prediction rule). Then, for any $r > 0$:
\begin{align}
\label{eq:lb}
\dsf(\xv, \bd \Fs_{\hat{y}}) \geq \min\left\{\min_{k \neq \hat{y}} \tfrac{f_{\hat{y}}(\xv) - f_k(\xv)}{L^k_{\xv}}, r\right\},
\end{align}
where $L^k_{\xv}$ is the local Lipschitz constant of $f_{\hat{y}}(\xv) - f_k(\xv)$ over the ball $B(\xv, r)$
\end{theorem}
The lower bound \eqref{eq:lb} is in fact exact when $\fb$ is linear, and for general nonlinear $\fb$, it is reasonably close to the true distance, as shown empirically in \cite{WengZCYSGHD18}.
Here, our estimate for the Lipschitz constant $L^k_{\xv}$ is simply the maximum norm of the gradient (difference) of many random samples in the ball $B(\xv, r)$, as we found this simple strategy is reasonably efficient and accurate. For space limits, we defer the experimental details on approximating this distance to the appendix.


\subsection{Verifying the Margin Trade-off again on Deep Nonlinear Models}
Due to space limits, here we only present the experimental results on the CIFAR10 dataset, while results on MNIST can be found in the appendix. For each model, we plot the average margin and minimum margin \wrt the number of training epochs
in \Cref{fig:cifar_avg_min_dist}. It is clear that similar phenomenon as those in \S\ref{sec:study} can be observed: the minimum margin continues increasing while at the same time the average margin keeps decreasing. This again highlights an intrinsic trade-off between minimum margin and average margin. In addition, we consistently observed that the average margin continues to decrease even after the training/test error has saturated (see training curves \Cref{fig:training_curves} in the appendix).


To investigate how the margin distribution changes during training, we plot the histograms of margins in \Cref{fig:cifar_hist} (appendix). As training proceeds, we observe again the margin distribution shifts towards left, meaning the majority of data points is pushed closer to the boundary. To some extent, this provides compelling explanation of the non-robustness of deep models: although achieving high accuracy by maximizing the minimum margin, the average margin, which is a better indicator of robustness, is at serious jeopardy. Note that early stopping, while helps preventing the average margin to decrease unnecessarily, is not sufficient by itself to promote average margin. Instead, an explicit average margin regularizer is more effective, as we show next.

\begin{figure}
\begin{subfigure}{0.32\textwidth}
\includegraphics[width = \textwidth, height = 0.10\textheight]{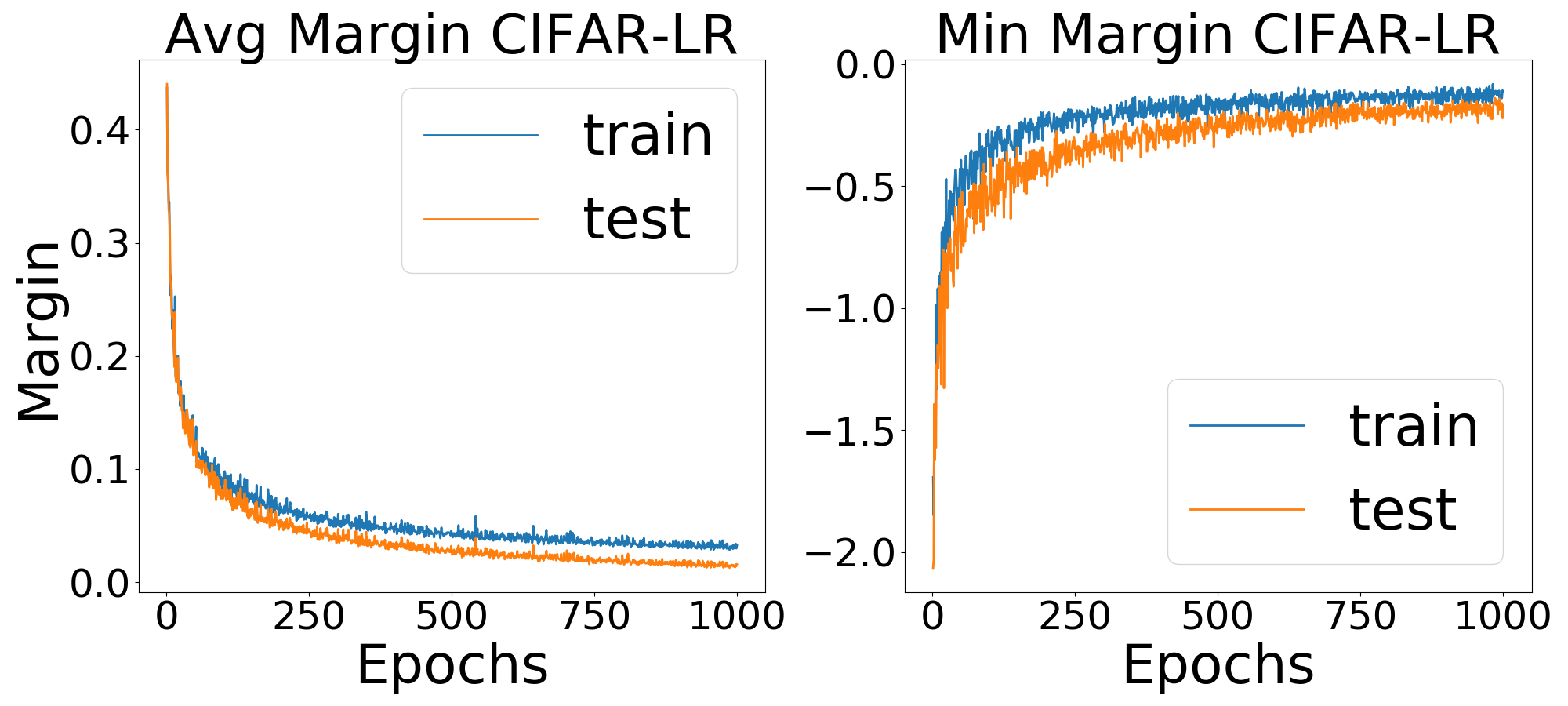}
\caption*{Avg vs Min CIFAR-LR}
\end{subfigure}
\begin{subfigure}{0.32\textwidth}
\includegraphics[width = \textwidth, height = 0.10\textheight]{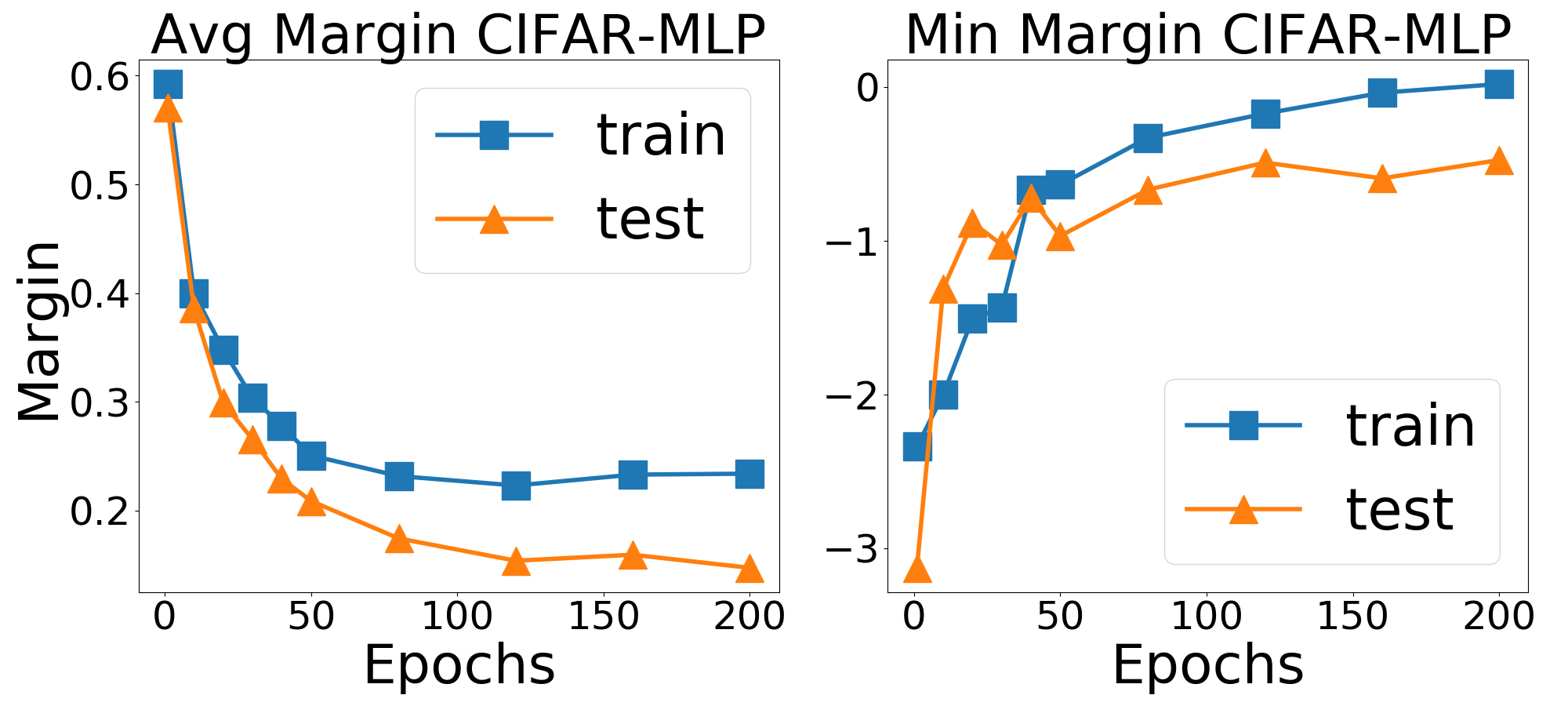}
\caption*{Avg vs Min CIFAR-MLP}
\end{subfigure}
\begin{subfigure}{0.32\textwidth}
\includegraphics[width = \textwidth, height = 0.10\textheight]{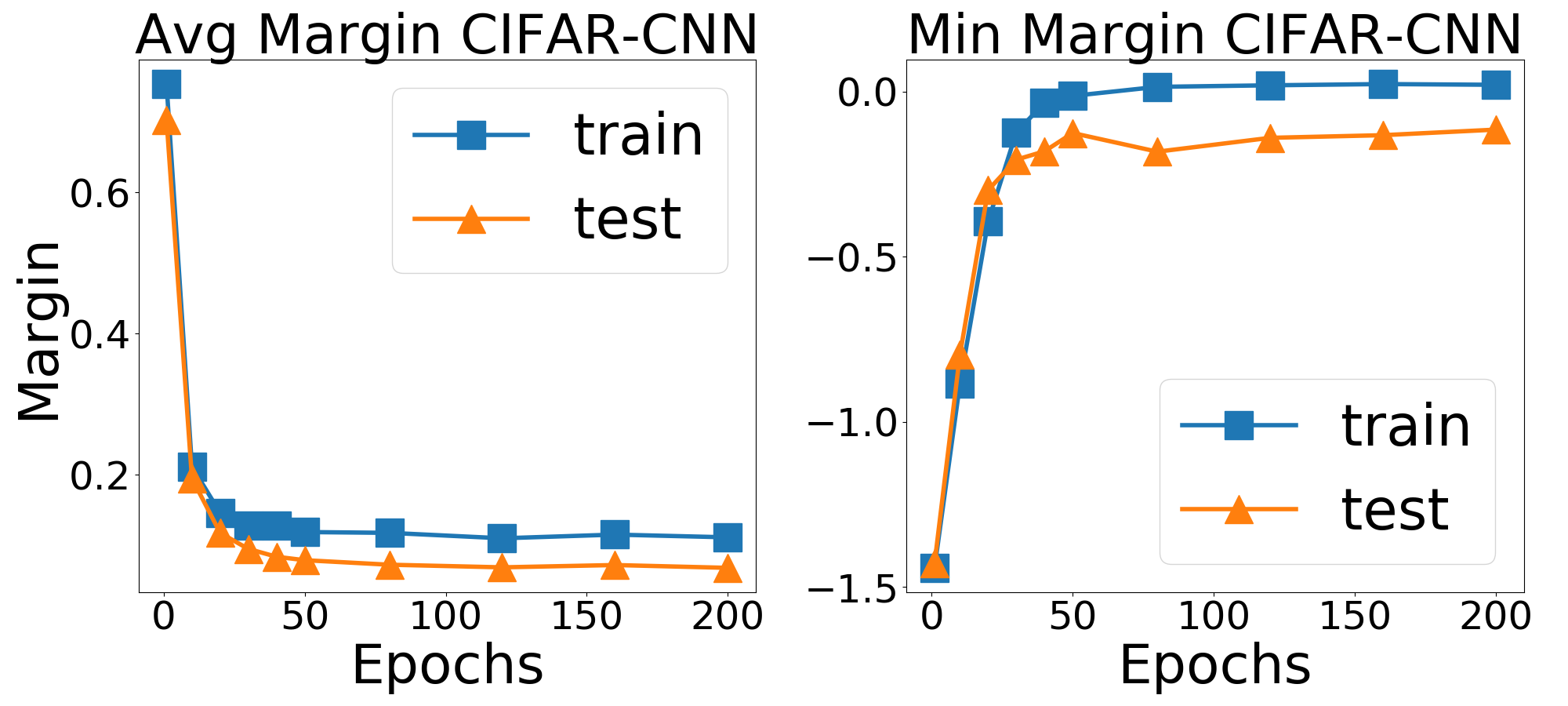}
\caption*{Avg vs Min CIFAR-CNN}
\end{subfigure}
\caption{Average and minimum margin during training of different models on CIFAR10.}
\label{fig:cifar_avg_min_dist}
\end{figure}

\begin{figure}
\begin{subfigure}{0.32\textwidth}
\includegraphics[width = \textwidth, height = 0.10\textheight]{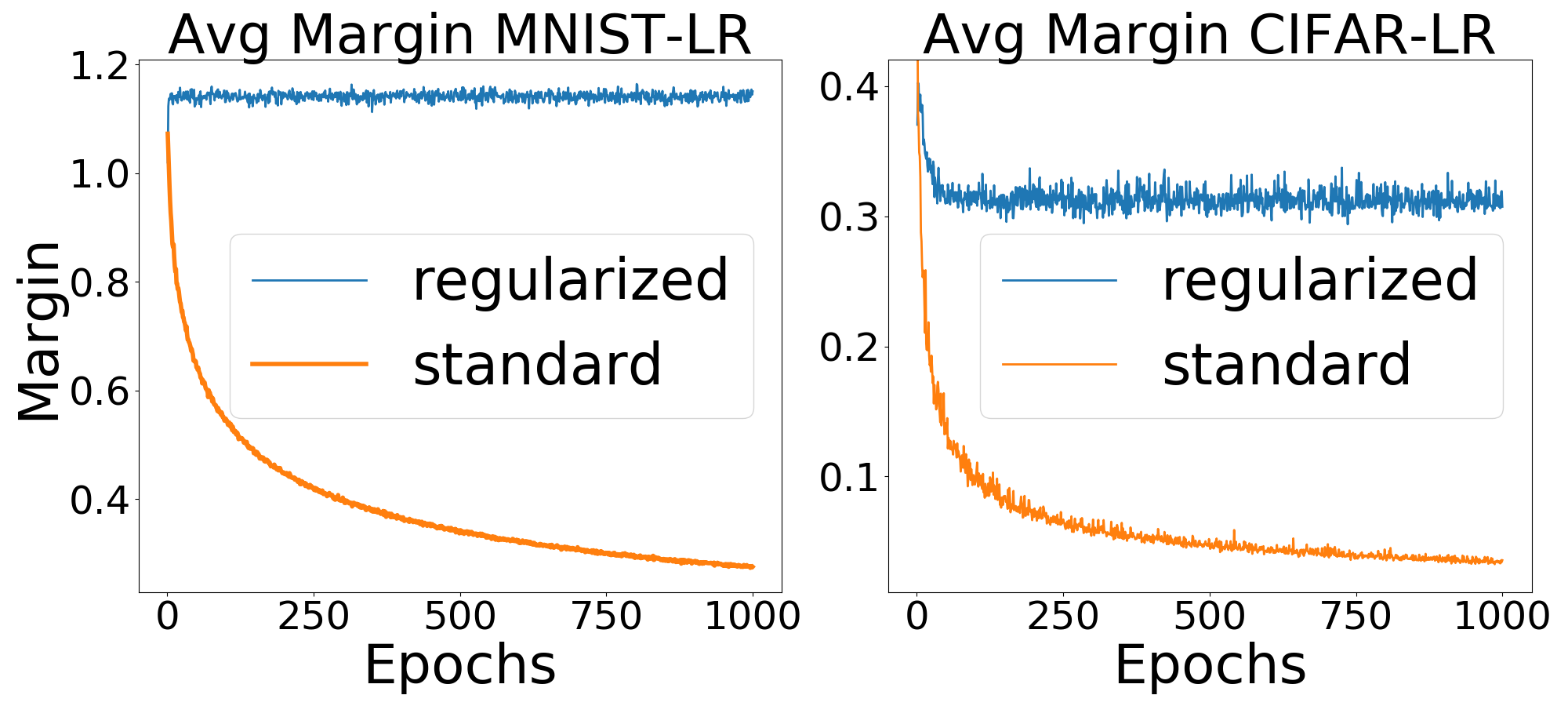}
\caption*{Avg Margin of LR}
\end{subfigure}
\begin{subfigure}{0.32\textwidth}
\includegraphics[width = \textwidth, height = 0.10\textheight]{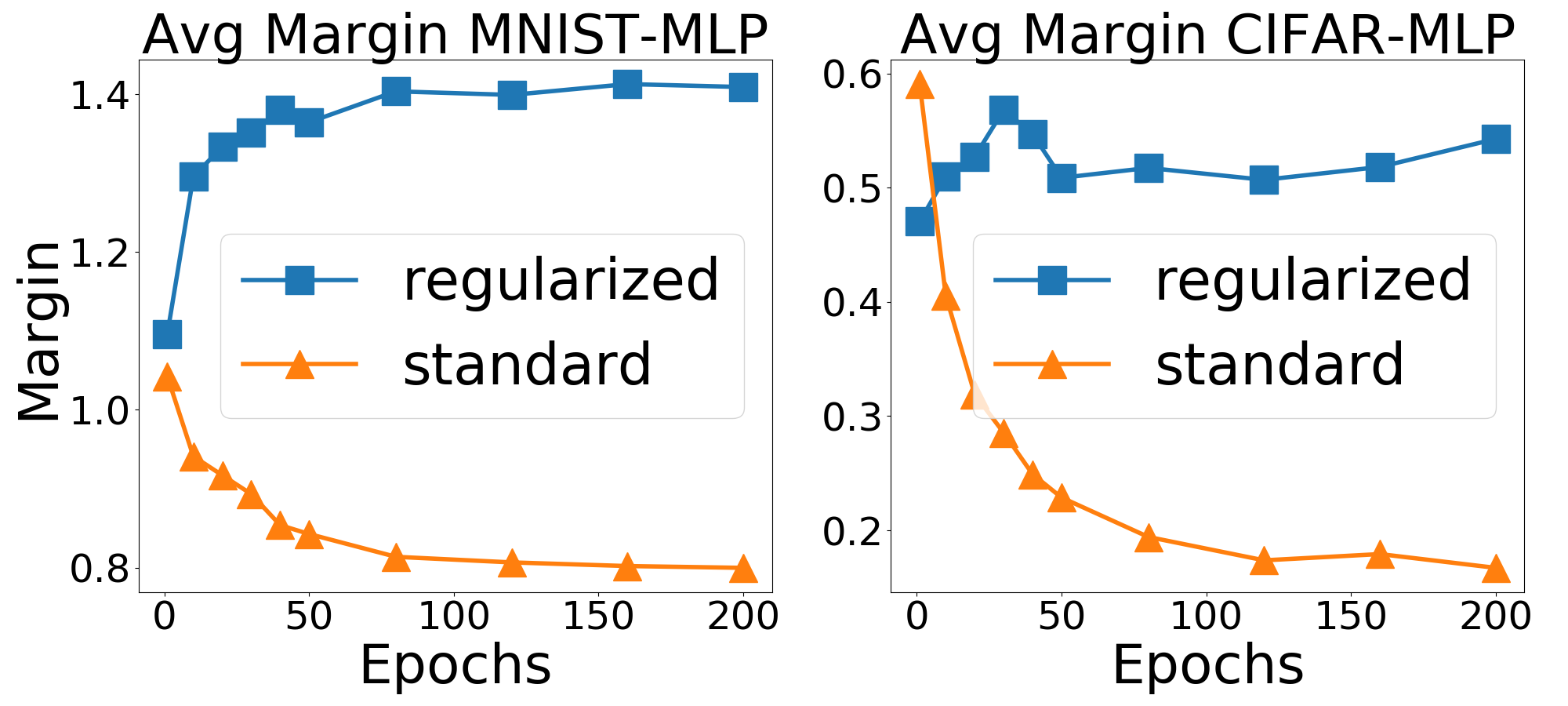}
\caption*{Avg Margin of MLP}
\end{subfigure}
\begin{subfigure}{0.32\textwidth}
\includegraphics[width = \textwidth, height = 0.10\textheight]{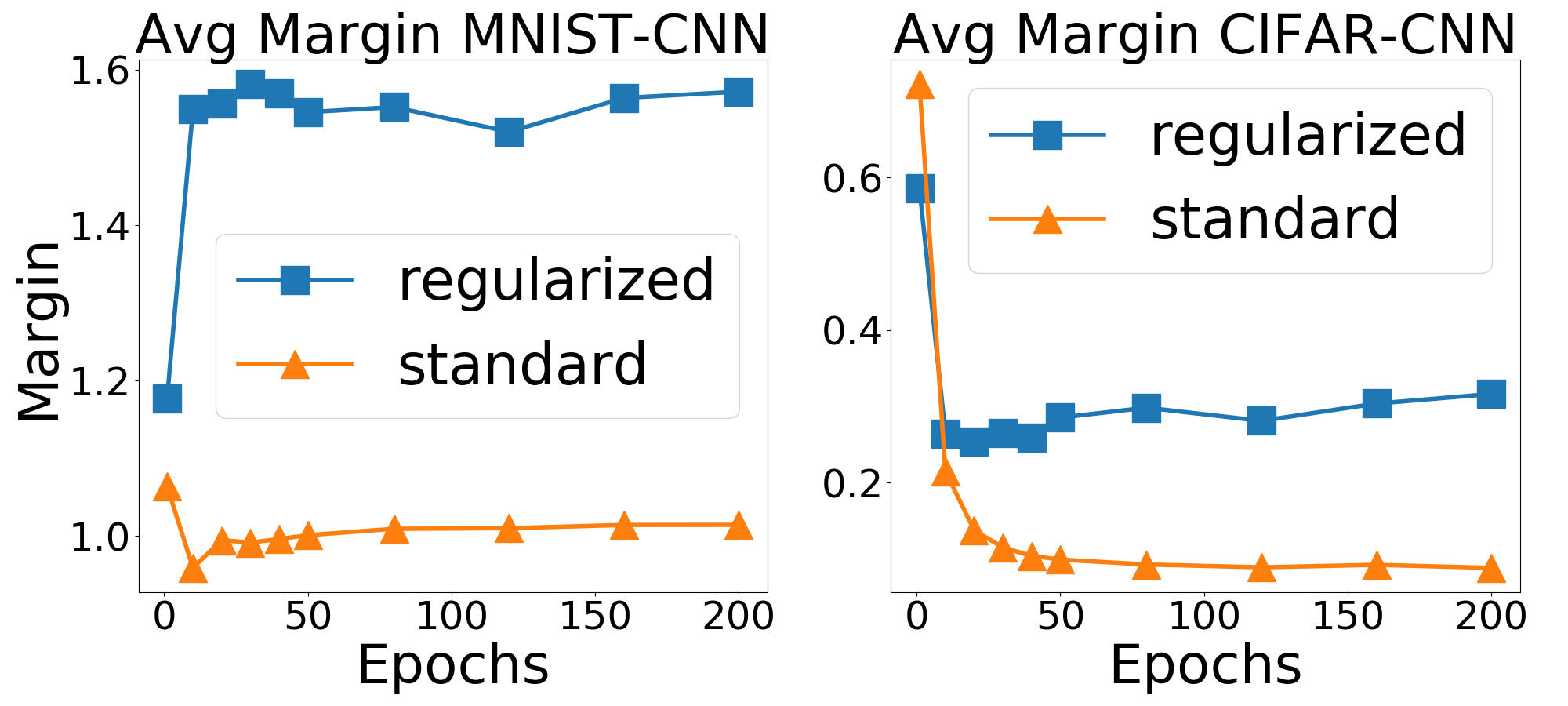}
\caption*{Avg Margin of CNN}
\end{subfigure}
\caption{Average margin of regularized and standard training.}
\label{fig:regularizer_vs_normal}
\end{figure}

\subsection{Average Margin Regularizer}
\label{subsec:amr}

We retrain six models augmented with our average margin regularizer, and compare the average margin with standard training in \Cref{fig:regularizer_vs_normal}. It is clear that with our average margin regularizer, the regularized models no longer sacrifice the average margin during training. For all six models, adding the average margin regularizer significantly improves adversarial robustness. For example, the average margin of LR and MLP become $2 \sim 3$ times more than those of models with standard training. For CNN, the average margin is increased by a factor of $1.5 \sim 2$.




Next, we compare our regularizer with Lipschitz constant regularization. In our regularizer, we use the orthogonal constraint $\| W W ^ \top - I\|_{\mathrm{F}}^2$ to control the Lipschitz constant. The idea of using orthogonal constraint as Lipschitz constant regularizer has been proposed before (\eg Paserval networks \citep{cisse2017parseval} and defensive quantization \citep{lin2018defensive}). Here, we prvide a comparison between Lipschitz constant regularization and our average margin regularizer. We evaluate models using clean accuracy, robust accuracy under $l_2$ norm projected gradient descent (PGD) attack \cite{GoodfellowSS15}, and the average margin. We run $1000$ iterations of PGD to generate adversarial examples, whose success serves as a good approximation of robust accuracy. The results can be found in \Cref{tb:lcr_amr}. As we can see, when $\epsilon$, the amount of allowed perturbations, is small, the accuracy of our average margin regularizer is comparable to that of the Lipschitz constant regularizer (LCR); but as $\epsilon$ becomes larger, the average margin regularizer consistently outperforms the latter. Interestingly, LCR alone can also improve adversarial robustness marginally \parencite{cisse2017parseval, lin2018defensive}, when $\epsilon$ is small. This may be caused by \emph{implicit} max-margin effect of the loss function used in training. However, as $\epsilon$ becomes larger, the robust accuracy of LCR is inferior to our method, sometimes even worse than standard training,
which is also confirmed in the experimental results in \citep{cisse2017parseval,yan2018deep}. This implies that controlling Lipschitz constant alone may not be enough to train a robust model. Instead, one needs to maximize the average margin explicitly.


We also compare our method with deep defense (DD) \citep{yan2018deep}, PGD adversarial training (Adv) \citep{MadryMSTV18} and the maximum margin regularizer (MMR) \citep{croce2019provable}, which maximizes the linear region of ReLU netowrks. The results are shown in \Cref{tb:comparison}. Compared with standard training, our method can significantly improve  model robustness. On MNIST, our method achieves overall better robust accuracy. On CIFAR10, our method achieves comparable results as adversarial training and deep defense. In addition, our method consistently outperforms MMR on both datasets. Moreover, our method is more efficient than adversarial training and deep defense. In fact, adversarial training would increase the training time by the number of PGD iterations (typically tens or hundreds), and deep defense is even more time consuming, to a point where it can only be applied as fine-tuning \citep{yan2018deep}, while the training time of our method is roughly the same as the standard training.

\begin{table}[t]
\scriptsize
\centering
\caption{Comparision between AMR (our average margin regularizer) and LCR (Lipschitz constant regularization). \textbf{Column 4-9:} clean accuracy. \textbf{Column 5-8:} robust accuracy under $l_2$ PGD attack. \textbf{The last column:} average margin.
}
\begin{tabular}{c c c c c c c c c c}
\toprule
& Models & Method & Clean & $\epsilon = 0.5$ & $\epsilon = 1.0$ & $\epsilon = 1.5$ & $\epsilon = 2.0$ & Avg Margin \\
\toprule
\multirow{6}{*}{MNIST} &
\multirow{3}{*}{MLP} & Std & \bf{98.24} & 89.26 & 49.23 & 15.78 & 5.06 & 0.80 \\
                     & & LCR & 95.99 &  91.12 &  80.62 & 60.56 & 34.07 & 1.36 \\
                     & & AMR & 96.01 & \bf{91.18} &  \bf{81.01} & \bf{62.93} & \bf{38.44} & \bf{1.41} \\
\cmidrule{2-9}
& \multirow{3}{*}{CNN} & Std & 99.14 & 95.95 & 90.48 & 88.72 & 87.52 & 1.01 \\ 
                     & & LCR & \bf{99.29} & 97.21 & 87.83 & 58.30 & 26.96 & 1.34 \\
                     & & AMR & 99.25 & \bf{97.90} & \bf{97.60} & \bf{97.51} & \bf{97.33} & \bf{1.40} \\
\toprule
& Models & Method & Clean & $\epsilon = 0.1$ & $\epsilon = 0.2$ & $\epsilon = 0.3$ & $\epsilon = 0.4$ & Avg Margin \\
\toprule
\multirow{6}{*}{CIFAR} &
\multirow{3}{*}{MLP} & Std & \bf{54.04} & 39.77 & 26.67 & 16.77 & 9.95 &  0.17 \\
                     & & LCR & 50.09 & 46.23 & 41.65 & 37.09 & 32.62 & 0.45 \\
                     & & AMR & 50.36 & \bf{46.28} & \bf{42.81} & \bf{39.06} & \bf{35.67} & \bf{0.54}\\
\cmidrule{2-9}
& \multirow{3}{*}{CNN} & Std & 78.77 & 54.32 & 39.81 & 37.14 & 36.27 & 0.09 \\
                     & & LCR & \bf{80.54} & \bf{70.84} &  58.72 & 44.82 & 32.54 & 0.26 \\
                     & & AMR & 78.12 & 69.59 & \bf{59.84} & \bf{49.02} & \bf{38.65} & \bf{0.31} \\
\bottomrule
\end{tabular}
\label{tb:lcr_amr}
\end{table}

\begin{table}[t]
\scriptsize
\centering
\caption{Comparison between 5 different training methods: Std (standard training), DD (deep defense \citep{yan2018deep}), Adv (PGD adversarial training \citep{MadryMSTV18}), MMR (maximum margin regularizer \citep{croce2019provable}), and AMR (ours). \textbf{Column 4:} clean accuracy. \textbf{Column 5-8:} robust accuracy under $l_2$ PGD attack.
}
\begin{tabular}{c c c c c c c c}
\toprule
& & Method & Clean & $\epsilon = 0.5$ & $\epsilon = 1.0$ & $\epsilon = 1.5$ & $\epsilon = 2.0$ \\
\toprule
\multirow{6}{*}{\begin{sideways}MNIST\end{sideways}} & 
\multirow{4}{*}{LeNet} & Std & 98.98 & 95.75 & 80.59 & 44.36 & 22.64 \\
                                                    & & DD & 99.34 & \bf{97.64} & 92.09 & 83.53 & 79.35 \\
                                                    & & Adv & \bf{99.48} & 97.45 & 91.99 & 88.88 & 87.51 \\
                                                    & & AMR & 99.01 & 96.80 & \bf{94.03} & \bf{93.61} & \bf{93.12} \\
\cmidrule{2-8}
& \multirow{2}{*}{LeNetSmall} & MMR & 97.43 & 89.90 & 58.86 & 23.95 & 6.80 \\
                                                    & & AMR & \bf{97.83} & \bf{91.94} & \bf{73.70} & \bf{32.56} & \bf{8.79} \\
\toprule
& & Method & Clean & $\epsilon = 0.1$ & $\epsilon = 0.2$ & $\epsilon = 0.3$ & $\epsilon = 0.4$ \\
\toprule
\multirow{6}{*}{\begin{sideways}CIFAR\end{sideways}} &
\multirow{4}{*}{ConvNet} & Std & 76.61 & 56.47 & 37.37 & 30.17 & 28.89 \\
                                                    & & DD & \bf{79.10} & \bf{68.95} & 59.51 & \bf{56.84} & \bf{56.38} \\
                                                    & & Adv & 75.27 & 68.79 & \bf{61.79} & 54.35 & 47.00 \\
                                                    & & AMR & 76.77 & 68.00 & 57.87 & 52.38 & 50.03 \\
\cmidrule{2-8}
& \multirow{2}{*}{LeNetSmall} & MMR & 59.07 & 49.01 & 39.42 & 30.41 & 22.53 \\
                                                    & & AMR & \bf{68.40} & \bf{61.33} & \bf{53.84} & \bf{46.03} & \bf{39.05} \\
\bottomrule
\end{tabular}
\label{tb:comparison}
\end{table}

\section{Conclusion}

In this work, we studied adversarial robustness from a margin perspective. We discovered the intrinsic trade-off between minimum and average margin, which appears across different models and datasets. We gave strong empirical evidence that deep models maximize the minimum margin during training to achieve high accuracy, but at the expense of decreasing the average margin significantly hence becoming more susceptible to adversarial attacks. To address the issue, we designed a new regularizer to explicitly maximize average margin and to retain Fisher consistency. Our extensive experiments confirmed the effectiveness of our regularizer. In the future, we will theoretically analyze the trade-off to provide further insights to the phenomenon. Moreover, we plan to go beyond just maximizing average margin, by controlling other margin distribution statistics such as variance.





\newpage


\bibliographystyle{unsrtnat}

\bibliography{main}

\newpage

\appendix


\newpage

\appendix

\section{Proofs}
\label{sec:proof}

\begin{proof}[Proof of \Cref{thm:geometric}]
Indeed, according to the definitions in \eqref{eq:class_set} and \eqref{eq:rob_org} we have
\begin{align}
\rsf(\xv) 
&= \inf\{ \|\zv\|: \xv + \zv \not\in \Fs_{\yhat(\xv)}, ~ \xv+\zv\in\Xc \} 
\\
&= \inf\{ \| (\xv+\zv) - \xv\|: \xv + \zv \not\in \Fs_{\yhat(\xv)}, ~ \xv+\zv\in\Xc \} 
\\
&= \dsf(\xv, \bar \Fs_{\yhat(\xv)}) = \dsf(\xv, \bd \Fs_{\yhat(\xv)})
\end{align}
\end{proof}


\begin{proof}[Proof of \Cref{lma:regularizer_calibrated}]
Since both $\phi$ and $\psi$ are bounded below, $\ell$ is also bounded below. Denote $m = \inf_{\alpha \in \RR} C_{\eta}^{\ell}(\alpha)$. We need to show $\inf_{\alpha(2\eta - 1) \leq 0}C_{\eta}^{\ell}(\alpha)$ is strictly greater than $m$ for all $\eta \neq \frac12$. We first consider the case $\eta > \frac12$. For the case $\eta < \frac12$, the proof is similar.

The derivative of $C_{\eta}^{\phi}(\alpha)$ at zero is $(2\eta - 1)\phi^{\prime}(0) < 0$. Thus $\exists \delta_1 > 0$ such that $\forall \alpha \in (0, \delta_1)$ satisfies $C_{\eta}^{\phi}(\alpha) < C_{\eta}^{\phi}(0)$.

The right hand derivative of $\psi$ at zero is negative, thus there $\exists \delta_2 > 0$ such that $\forall \alpha \in (0, \delta_2)$ satisfies $\psi(\alpha) < \psi(0)$. Notice that $\psi(\alpha)$ is constant when $\alpha \leq 0$. Thus $\forall \alpha \in (0, \delta_2)$ satisfies $C_{\eta}^{\psi}(\alpha) < C_{\eta}^{\psi}(0)$.

Notice that $C_{\eta}^{l}(\alpha) = C_{\eta}^{\phi}(\alpha) + \lambda C_{\eta}^{\psi}(\alpha)$. Combining above arguments, there exists a sufficiently small $\alpha_0$, such that $0 < \alpha_0 < \min(\delta_1, \delta_2)$, and $C_{\eta}^{\ell}(\alpha_0) < C_{\eta}^{\ell}(0)$. Splitting the interval $(-\infty, 0]$ into $(-\infty, -\alpha_0]$ and $[-\alpha_0, 0]$, it is sufficient to show that both
\begin{equation}
\inf_{\alpha \leq - \alpha_0} C_{\eta}^{\ell}(\alpha) > m
\end{equation}
and
\begin{equation}
\inf_{- \alpha_0 \leq \alpha \leq 0} C_{\eta}^{\ell}(\alpha) > m
\end{equation}
hold.

For $\alpha \leq - \alpha_0$,
\begin{align}
C_{\eta}^{\ell}(\alpha) - C_{\eta}^{\ell}(- \alpha) & = (2\eta - 1)(\ell(\alpha) - \ell(-\alpha)) \\
\label{eq:nonincreasing1}
& \geq (2\eta - 1)(\ell(0) - \ell(- \alpha)) \\
\label{eq:nonincreasing2}
& \geq (2\eta - 1)(\ell(0) - \ell(\alpha_0)) \\
\label{eq:alpha_0}
& > 0.
\end{align}
The inequalities \eqref{eq:nonincreasing1} and \eqref{eq:nonincreasing2} are because $l$ is a decreasing function. The inequality \eqref{eq:alpha_0} is because the specific choice of $\alpha_0$ strictly decreases the value of $l$ at zero. Namely, flipping the sign of $\alpha$ can strictly decrease the value of $C_{\eta}^{\ell}(\alpha)$. Thus,
\begin{align}
\inf_{\alpha \leq -\alpha_0} C_{\eta}^{\ell}(\alpha) & \geq (2\eta - 1)(\ell(0) - \ell(\alpha_0)) + C_{\eta}^{\ell}(- \alpha) \\
& \geq (2\eta - 1)(\ell(0) - \ell(\alpha_0)) + m \\
& > m.
\end{align}
For $-\alpha_0 \leq \alpha \leq 0$, by continuity there exists a minimizer $\alpha^\star$ for $C_{\eta}^{\ell}(\alpha)$ on this compact set. If $\alpha^\star \neq 0$, then we can again flip the sign of $\alpha$ to get a strictly smaller value. Thus $C_{\eta}^{\ell}(\alpha^\star) > C_{\eta}^{\ell}(- \alpha^\star) \geq m$. If $\alpha^\star = 0$, then
\begin{equation}
C_{\eta}^{\ell}(0) > C_{\eta}^{\ell}(\alpha_0) \geq m.
\end{equation}
Hence the theorem holds.
\end{proof}

\clearpage

\section{Training Curves}
\Cref{fig:lr_loss} shows the training loss, training error and test error of logistic regression in \Cref{sec:study}. As training goes, the error rate on both training set and test set never increase, thus the logistic regression is not overfitting, although trained with excessive number of epochs. This again hightlights that the trade-off between minimum and average margin cannot be caused by overfitting.

\begin{figure}[h]
\centering
\includegraphics[width = \textwidth, height = 0.20\textheight]{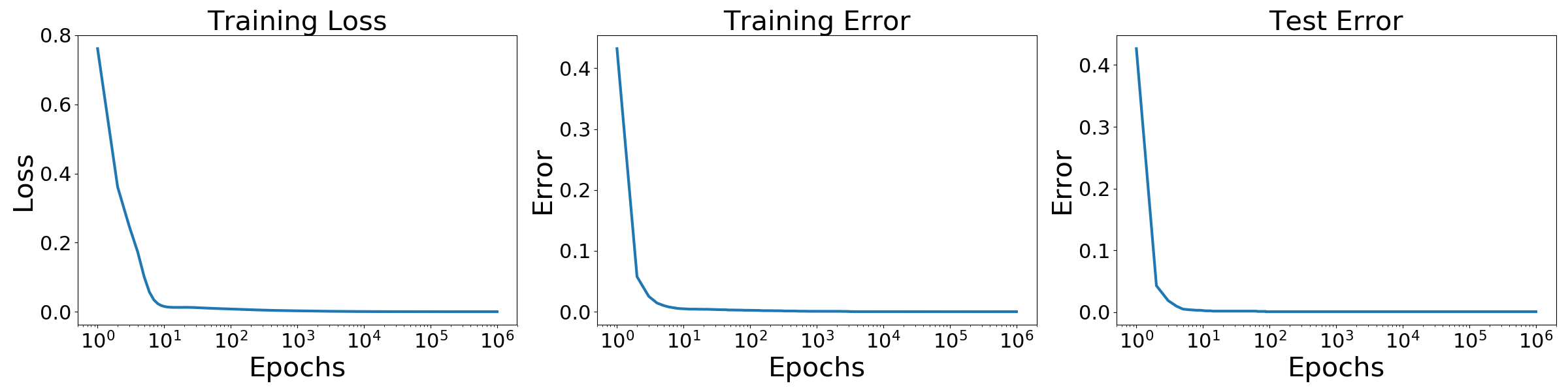}
\caption{Training curves of logistic regression in \Cref{sec:study} classifying $0$ and $1$ on MNIST. For the ease of illustration, x-axis is in log scale. \textbf{Left:} Training loss w.r.t. epochs. \textbf{Middle:} Training error w.r.t. epochs. \textbf{Right:} Test error w.r.t. epochs. The training error reaches zero, thus the subset of MNIST consisting of $0$'s and $1$'s is linear separable. Hence, this dataset strictly satisfies the contition in \Cref{thm:implicit_bias}.}
\label{fig:lr_loss}
\end{figure}

\Cref{fig:training_curves} shows the training curves of six models in \Cref{sec:exp} by standarded training. Although there exists a generalization gap between training errors and test errors, the test errors never increase, thus the trade-off between minimum and average margin cannot be caused by overfitting.

\begin{figure}[h]
\centering
\begin{subfigure}{\textwidth}
\centering
\includegraphics[width = 0.31\textwidth, height = 0.18\textheight]{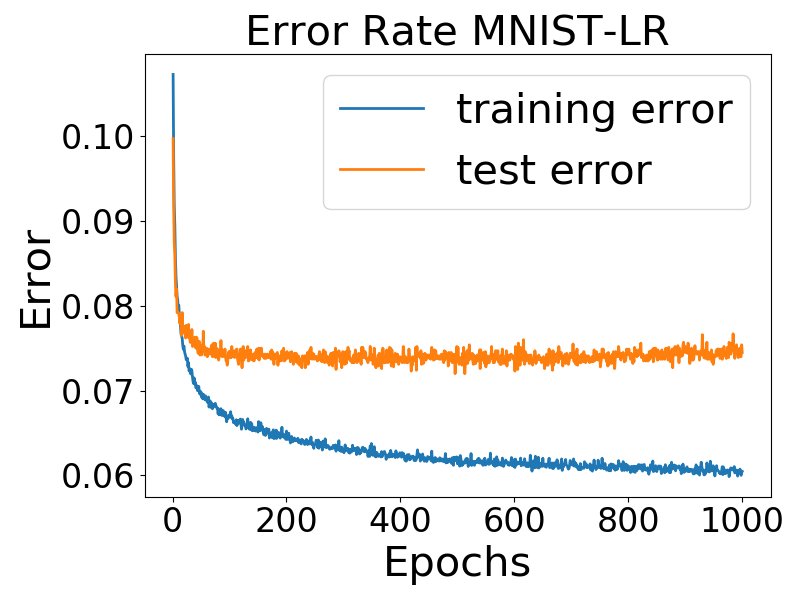}
\includegraphics[width = 0.31\textwidth, height = 0.18\textheight]{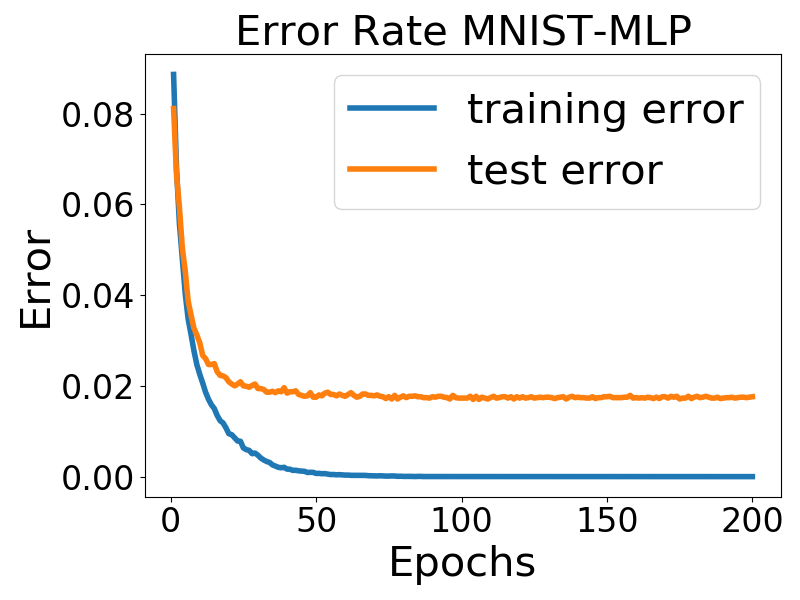}
\includegraphics[width = 0.31\textwidth, height = 0.18\textheight]{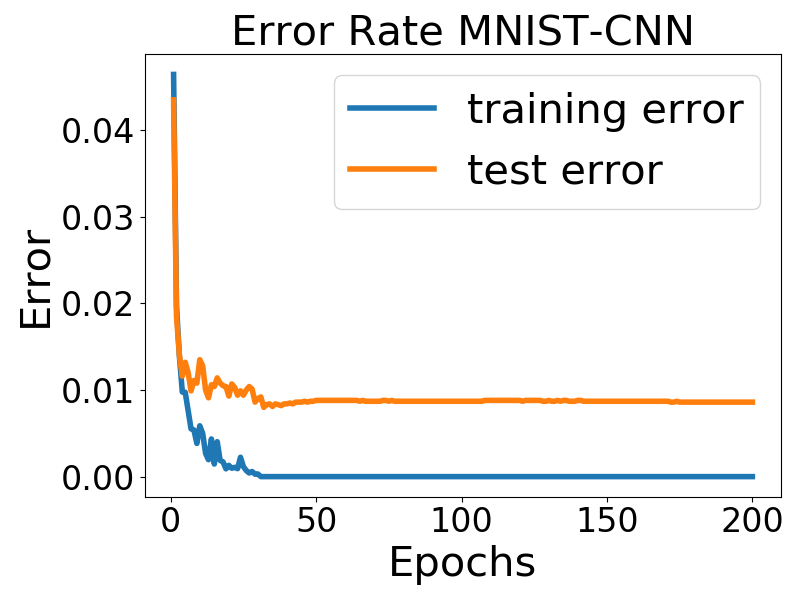}
\end{subfigure}
\begin{subfigure}{\textwidth}
\centering
\includegraphics[width = 0.31\textwidth, height = 0.18\textheight]{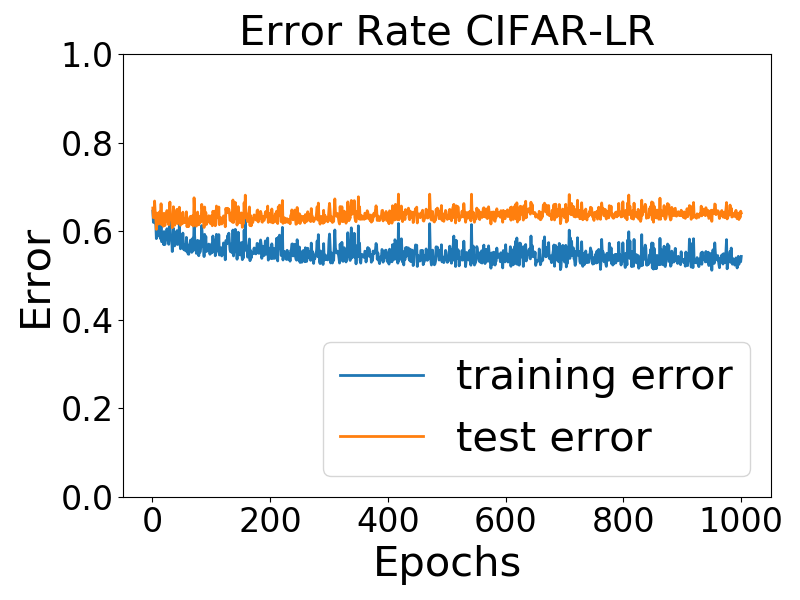}
\includegraphics[width = 0.31\textwidth, height = 0.18\textheight]{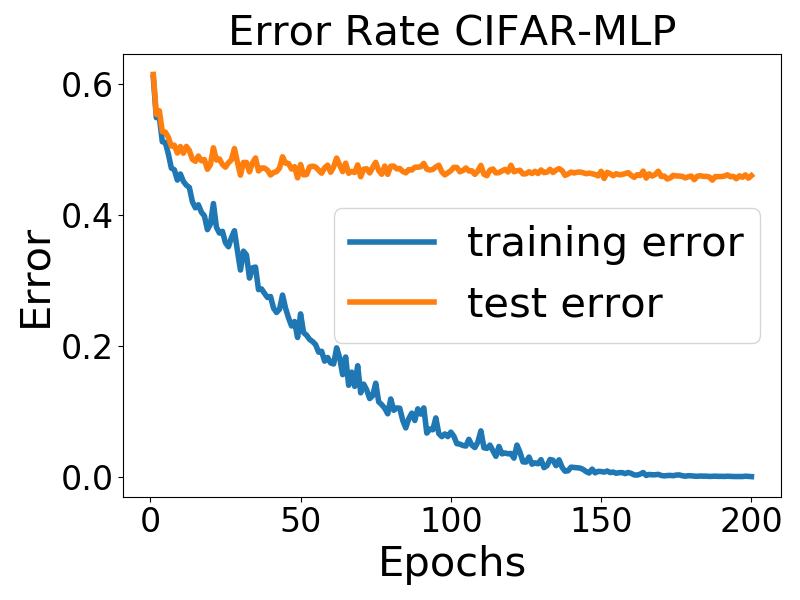}
\includegraphics[width = 0.31\textwidth, height = 0.18\textheight]{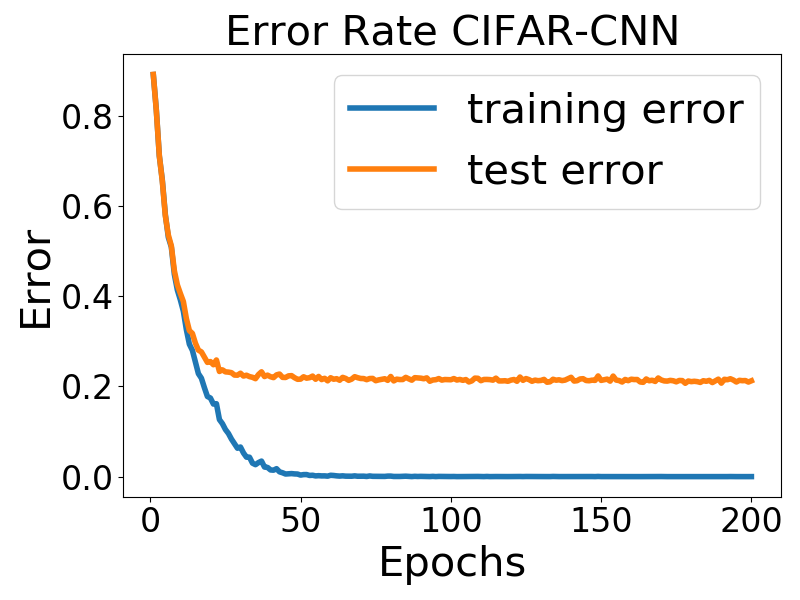}
\end{subfigure}
\caption{Training curves of 3 models on MNIST and 3 models on CIFAR10 by standard training. \textbf{First Row :} MNIST models. \textbf{Second Row :} CIFAR10 models.}
\label{fig:training_curves}
\end{figure}

\clearpage

\section{Minimum-average Margin Trade-off for MNIST Models}
\Cref{fig:mnist_avg_min_dist} shows the minimum and average margin trade-off for MNIST models \footnote{In the figure of average margin of MNIST-LR (top left), we shift the curve of average margin on training set by a small constant $0.001$, to avoid overlapping of two curves. This is only for illustration purpose.}. A similar trade-off between minimum and average margin can be observed as discussed in \Cref{sec:study,sec:exp}. The minimum margin keeps increasing while the average margin keeps decreasing. The only exception is the average margin of MNIST-CNN. But the range of margin in the figure is very small (from $0.94$ to $1.08$), in which case the Lipschitz constant estimation in \Cref{thm:approx_margin} could be inaccurate.

\begin{figure}[h]
\centering
\includegraphics[width = 0.8\textwidth, height = 0.20\textheight]{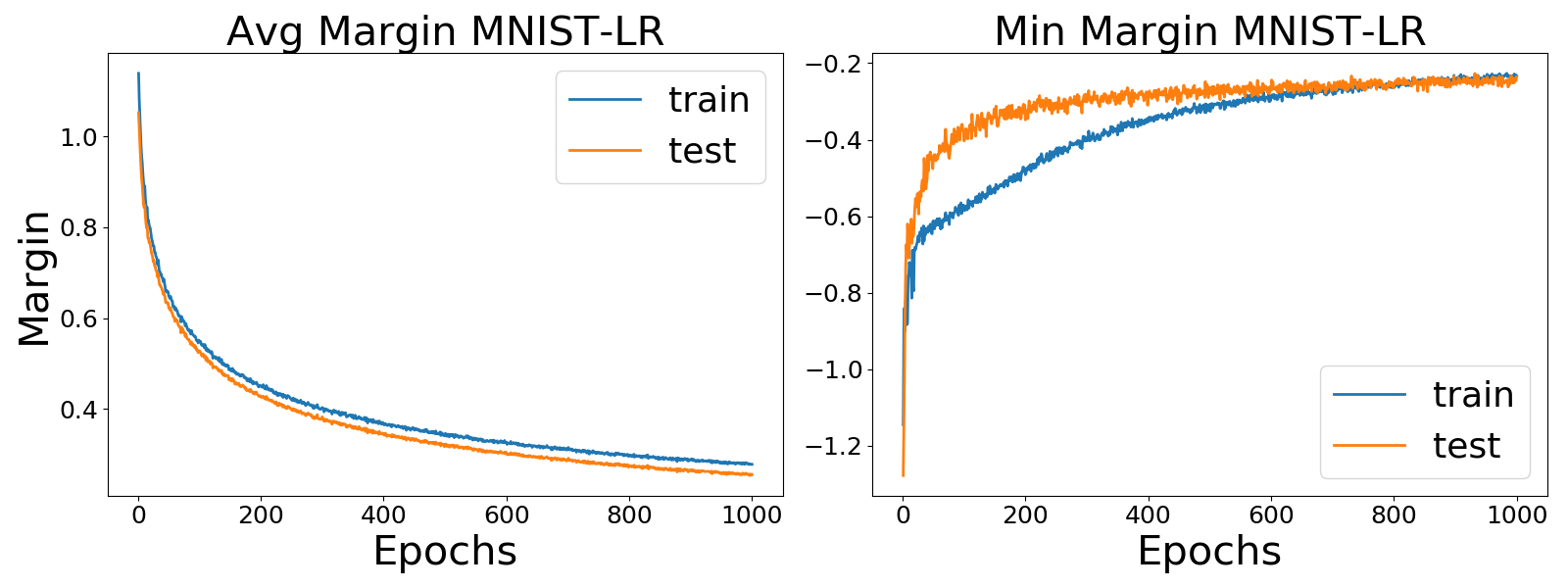}

\includegraphics[width = 0.8\textwidth, height = 0.20\textheight]{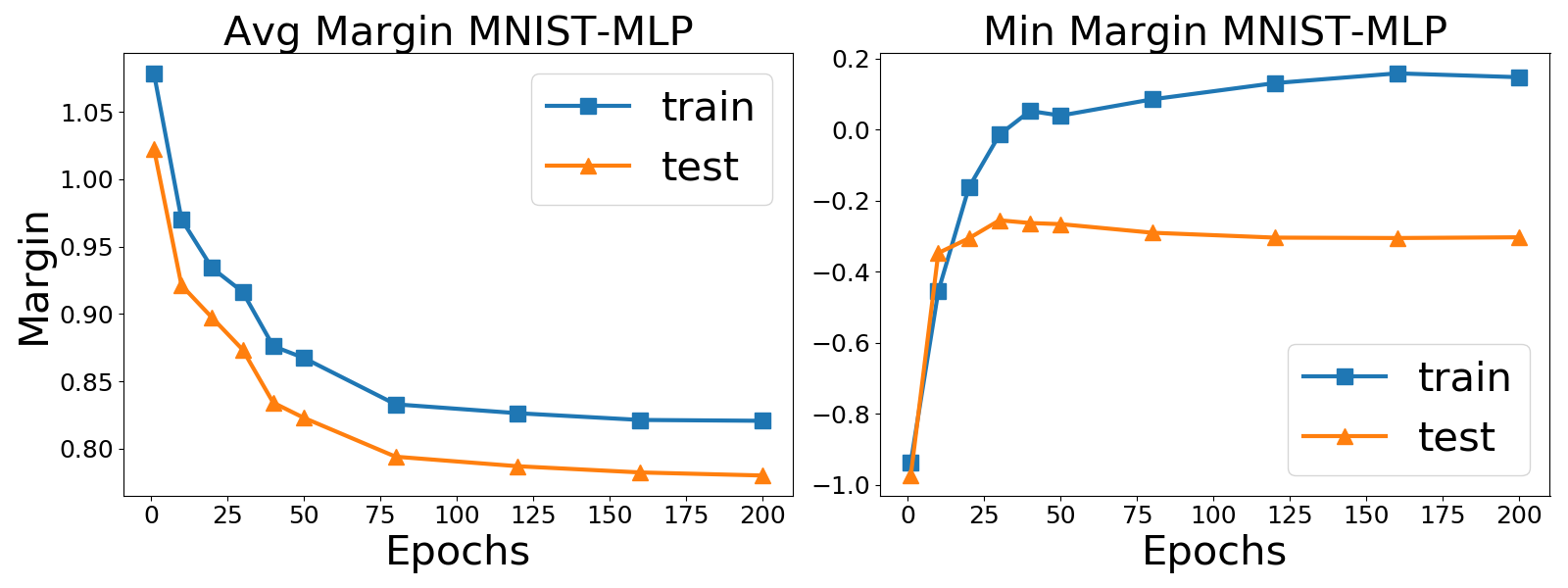}

\includegraphics[width = 0.8\textwidth, height = 0.20\textheight]{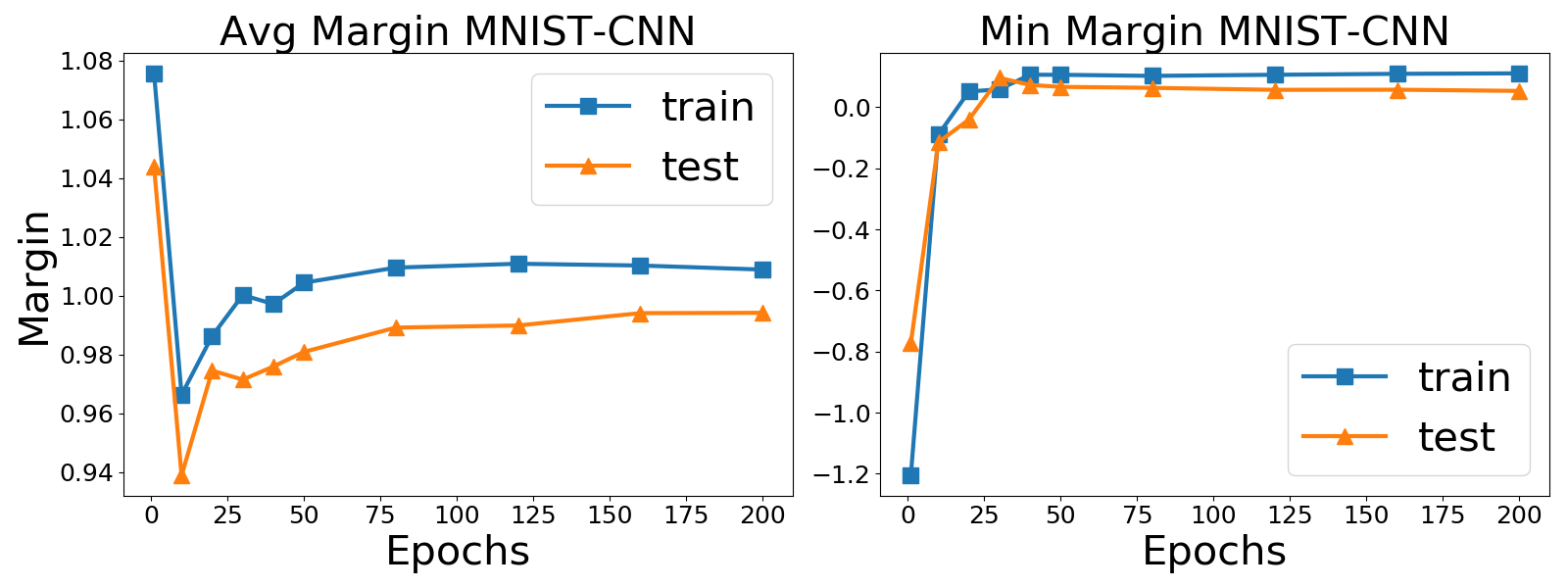}
\caption{Average margin and minimum margin during training of 3 MNIST models.}
\label{fig:mnist_avg_min_dist}
\end{figure}

\clearpage

\section{Margin Histograms of MNIST Models during Training}

We plot margin histograms for three MNIST models in \Cref{fig:mnist_hist}.


\begin{figure}[h]
\begin{subfigure}{0.48\textwidth}
\centering
\includegraphics[width = \textwidth, height = 0.20\textheight]{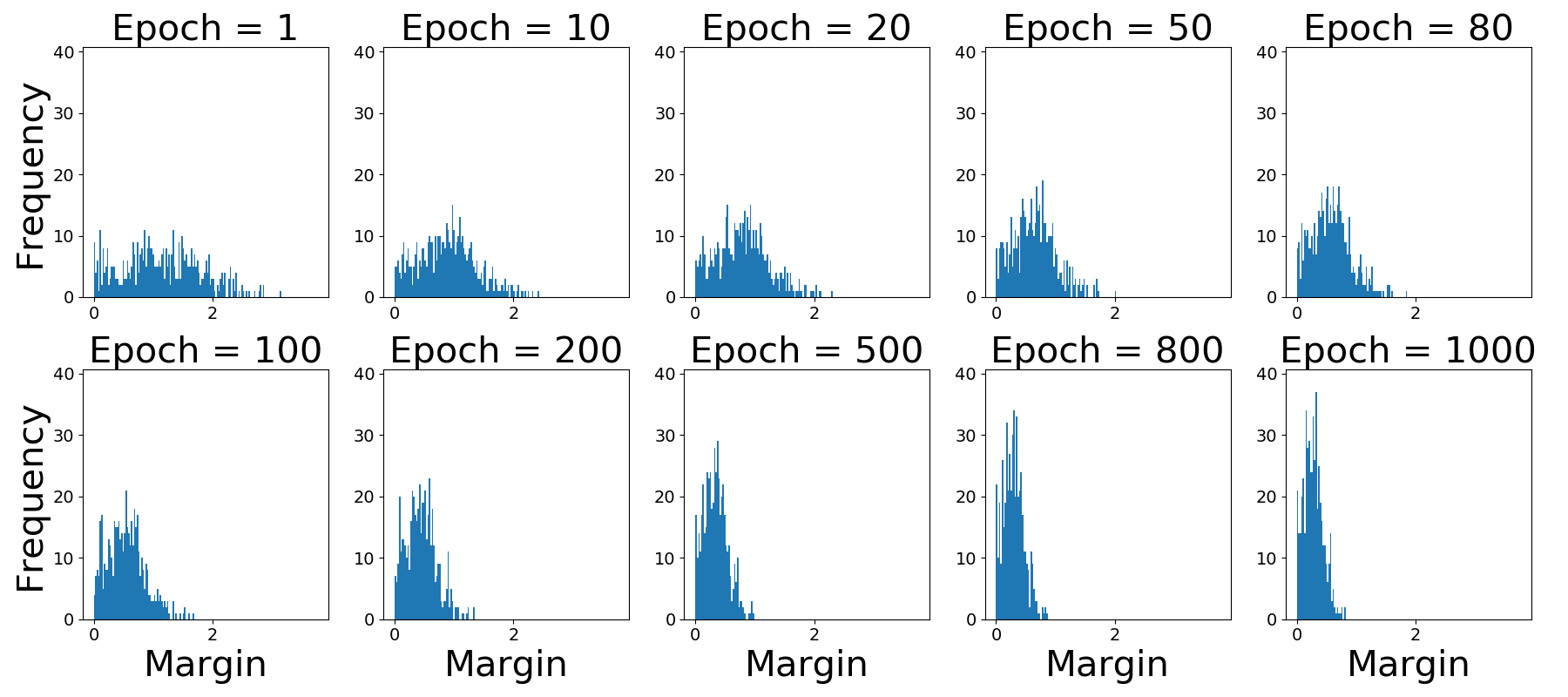}
\caption{Margin histograms of MNIST-LR on training set}
\label{fig:mnistlr_train_hist}
\end{subfigure}
~
\begin{subfigure}{0.48\textwidth}
\centering
\includegraphics[width = \textwidth, height = 0.20\textheight]{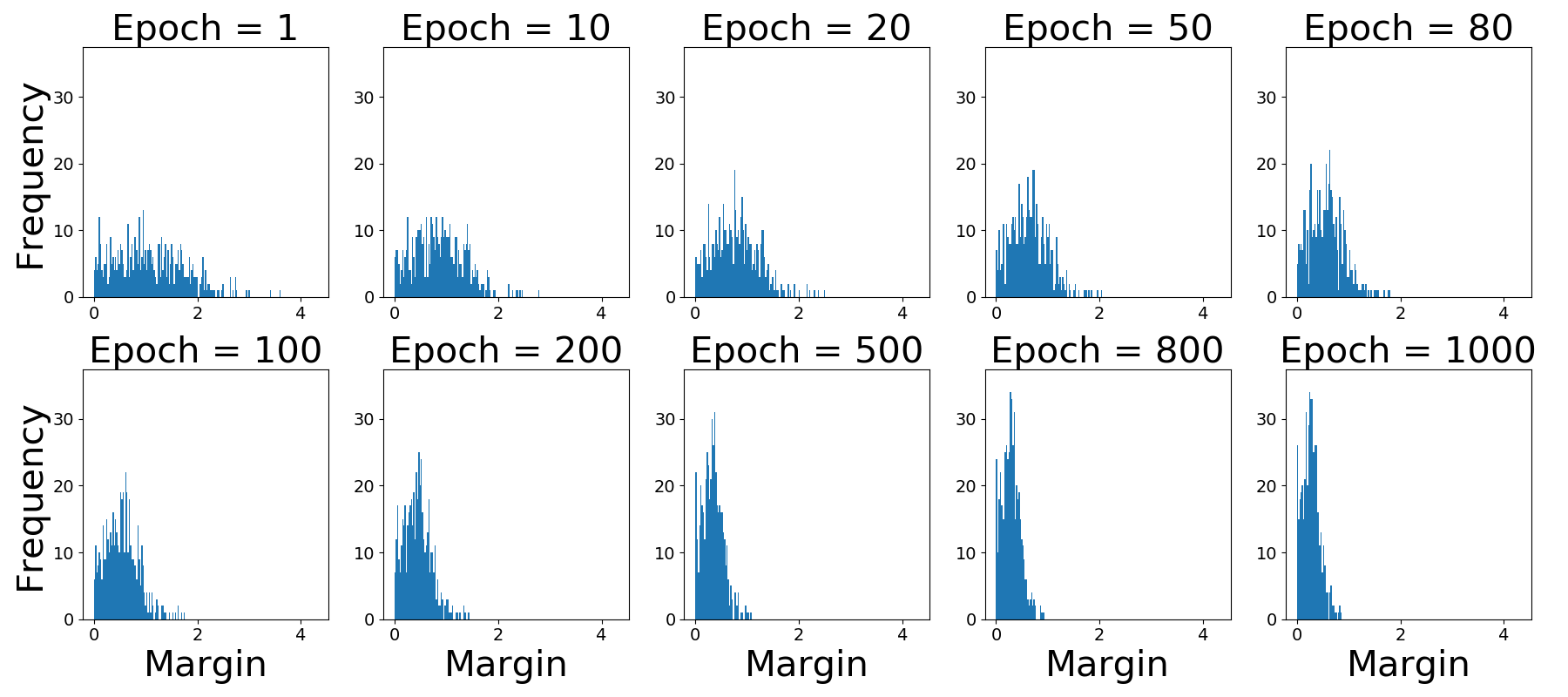}
\caption{Margin histograms of MNIST-LR on test set}
\label{fig:mnistlr_test_hist}
\end{subfigure}

\begin{subfigure}{0.48\textwidth}
\centering
\includegraphics[width = \textwidth, height = 0.20\textheight]{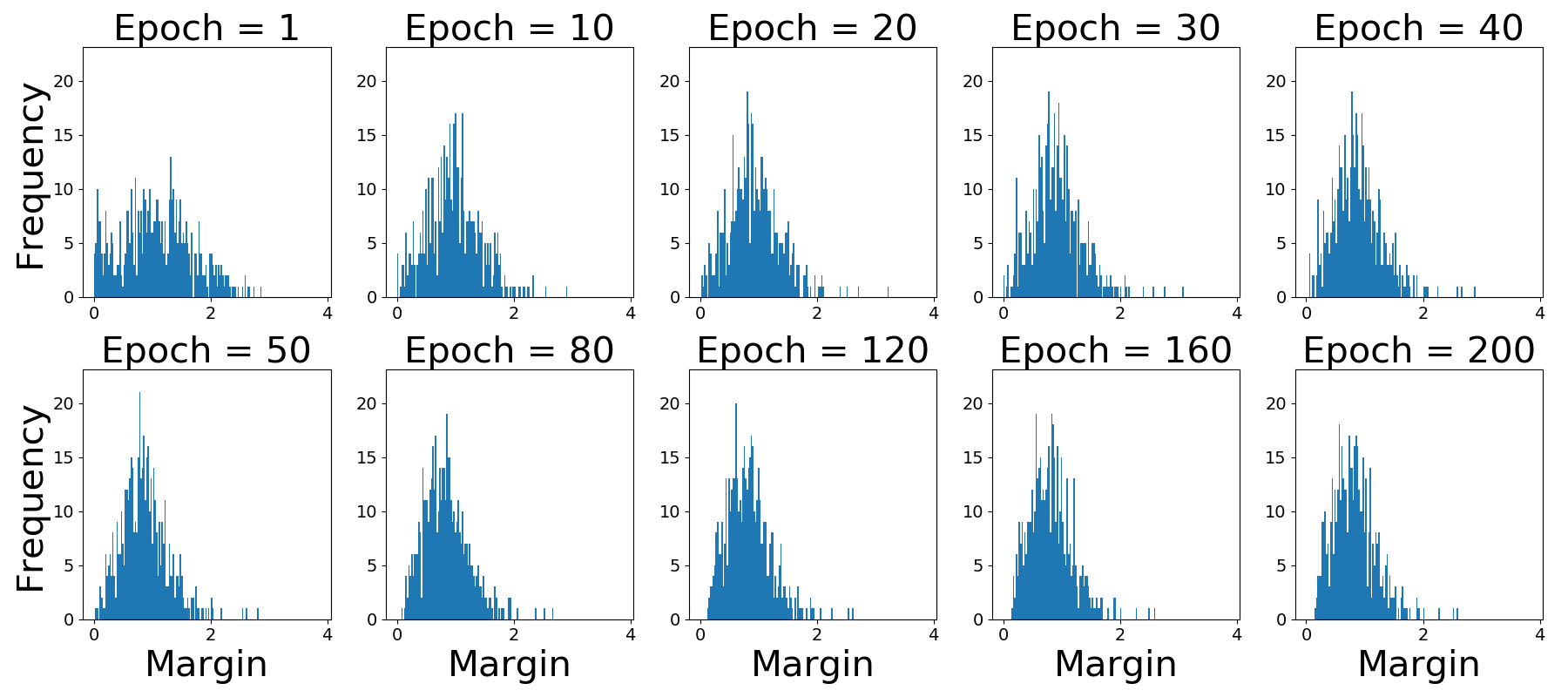}
\caption{Margin histograms of MNIST-MLP on training set}
\label{fig:mnistmlp_train_hist}
\end{subfigure}
\begin{subfigure}{0.48\textwidth}
\centering
\includegraphics[width = \textwidth, height = 0.20\textheight]{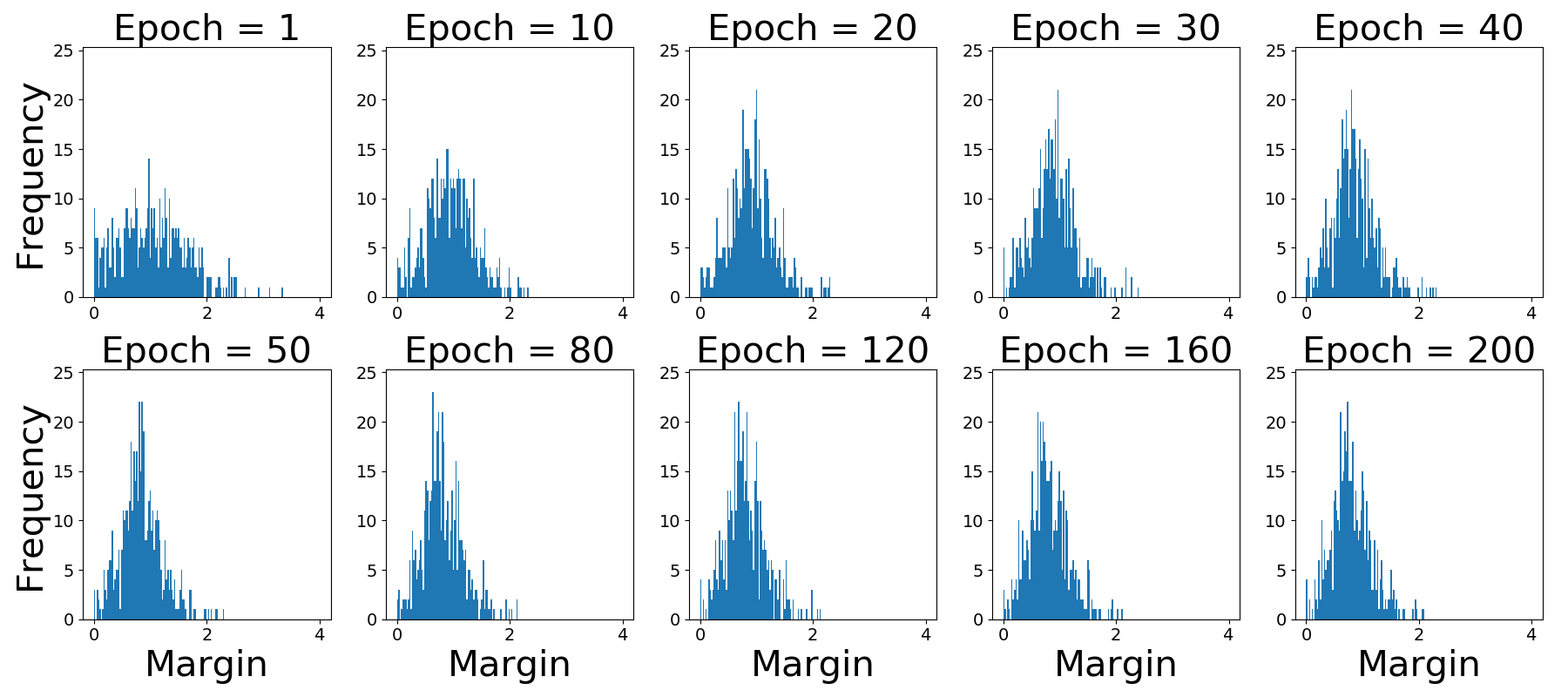}
\caption{Margin histograms of MNIST-MLP on test set}
\label{fig:mnistmlp_test_hist}
\end{subfigure}

\begin{subfigure}{0.48\textwidth}
\centering
\includegraphics[width = \textwidth, height = 0.20\textheight]{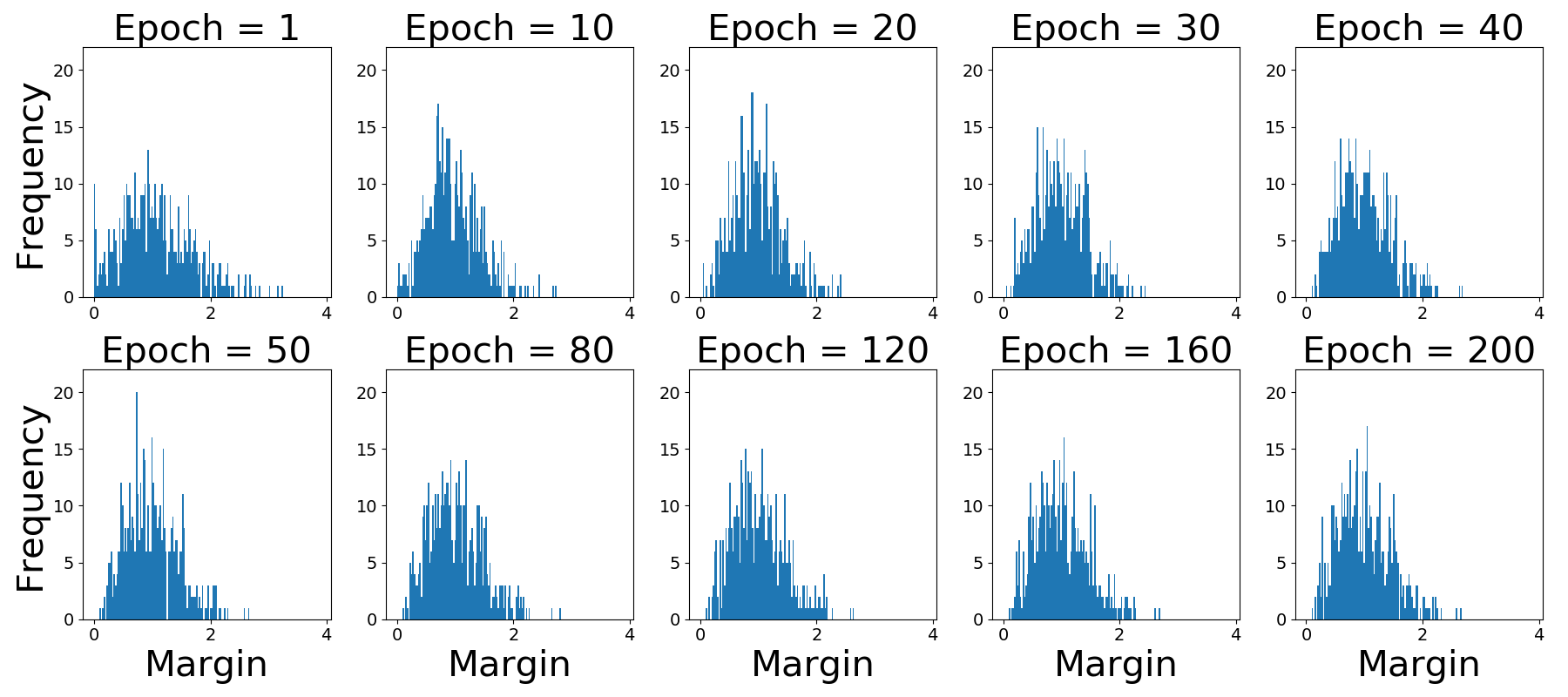}
\caption{Margin histograms of MNIST-CNN on training set}
\label{fig:mnistcnn_train_hist}
\end{subfigure}
\begin{subfigure}{0.48\textwidth}
\centering
\includegraphics[width = \textwidth, height = 0.20\textheight]{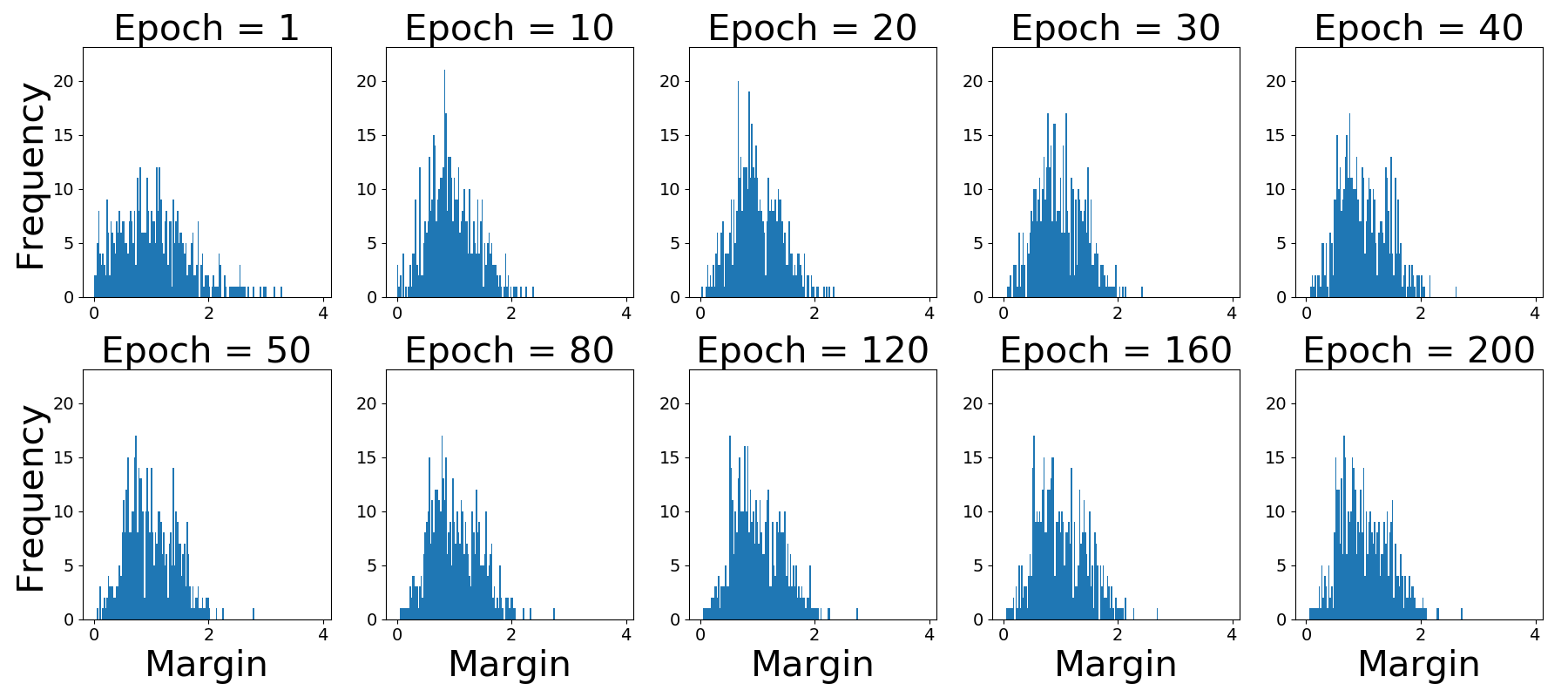}
\caption{Margin histograms of MNIST-CNN on test set}
\label{fig:mnistcnn_test_hist}
\end{subfigure}
\caption{Margin histograms of MNIST models at different epochs during training. \textbf{Top:} Histograms of MNIST-LR. \textbf{Mid:} Histograms of MNIST-MLP. \textbf{Bottom:} Hostograms of MNIST-CNN.}
\label{fig:mnist_hist}
\end{figure}

\clearpage

\section{Margin Histograms of CIFAR Models during Training}
We plot margin histograms for three CIFAR models in \Cref{fig:cifar_hist}.

\begin{figure}[h]
\centering
\begin{subfigure}{0.48\textwidth}
\centering
\includegraphics[width = \textwidth, height = 0.20\textheight]{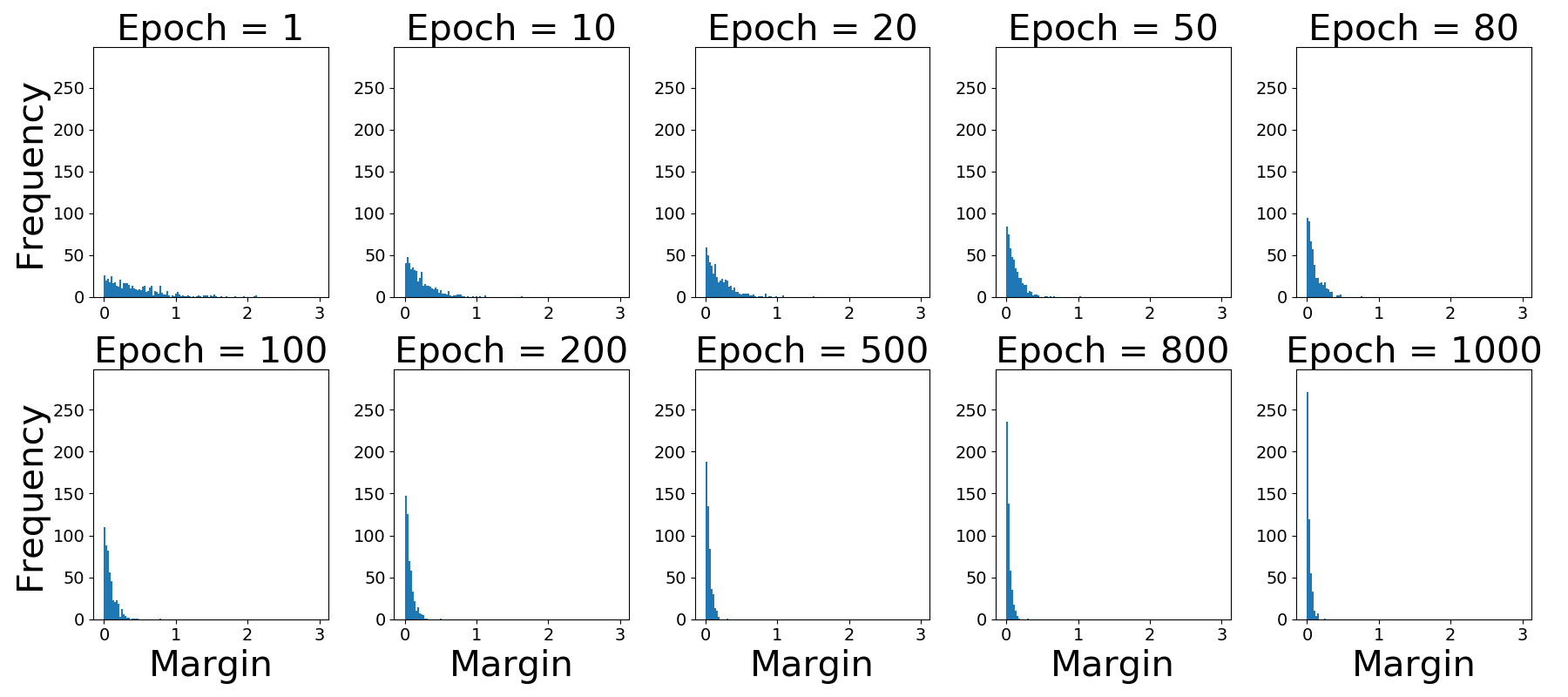}
\caption{Margin histograms of CIFAR-LR on training set}
\label{fig:cifarlr_train_hist}
\end{subfigure}
\begin{subfigure}{0.48\textwidth}
\centering
\includegraphics[width = \textwidth, height = 0.20\textheight]{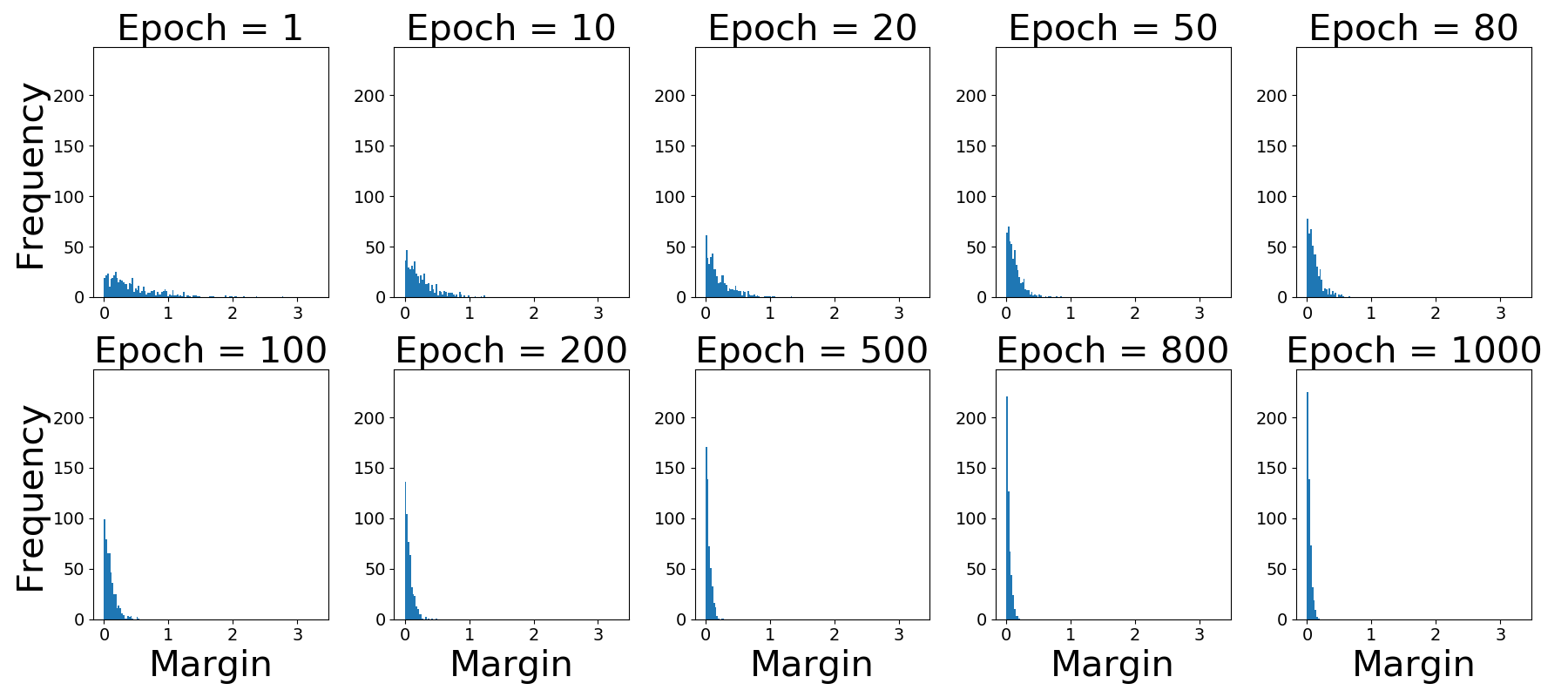}
\caption{Margin histograms of CIFAR-LR on test set}
\label{fig:cifarlr_test_hist}
\end{subfigure}

\begin{subfigure}{0.48\textwidth}
\centering
\includegraphics[width = \textwidth, height = 0.20\textheight]{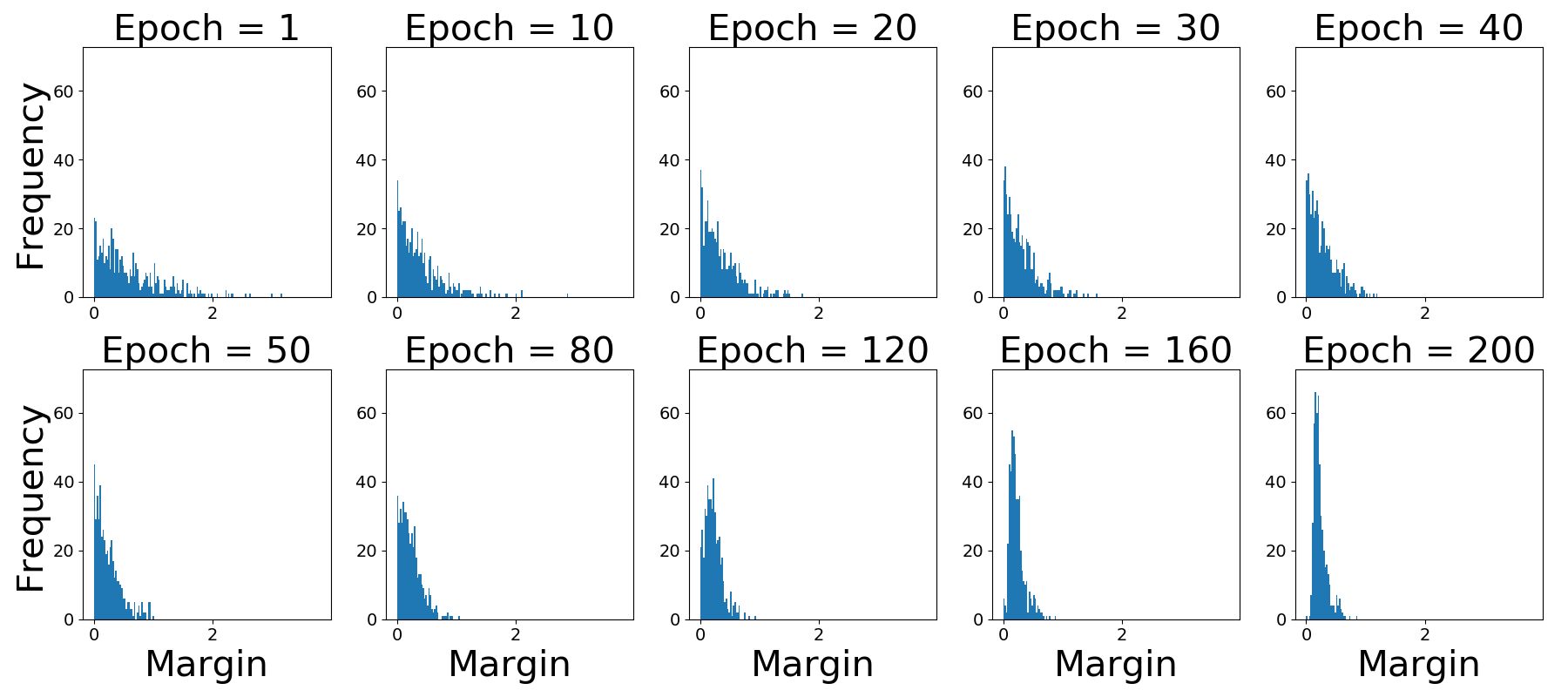}
\caption{Margin histograms of CIFAR-MLP on training set }
\label{fig:cifarmlp_train_hist}
\end{subfigure}
\begin{subfigure}{0.48\textwidth}
\centering
\includegraphics[width = \textwidth, height = 0.20\textheight]{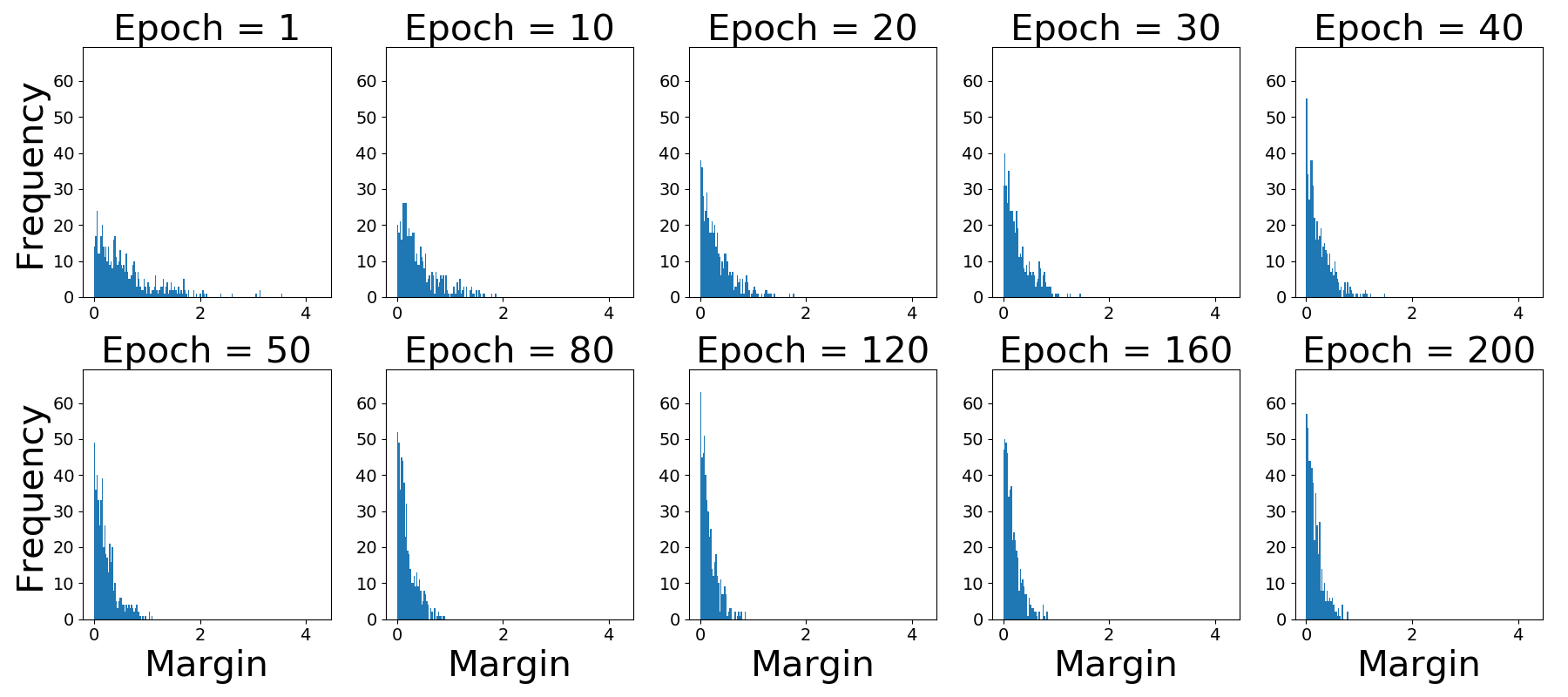}
\caption{Margin histograms of CIFAR-MLP on test set }
\label{fig:cifarmlp_test_hist}
\end{subfigure}

\begin{subfigure}{0.48\textwidth}
\centering
\includegraphics[width = \textwidth, height = 0.20\textheight]{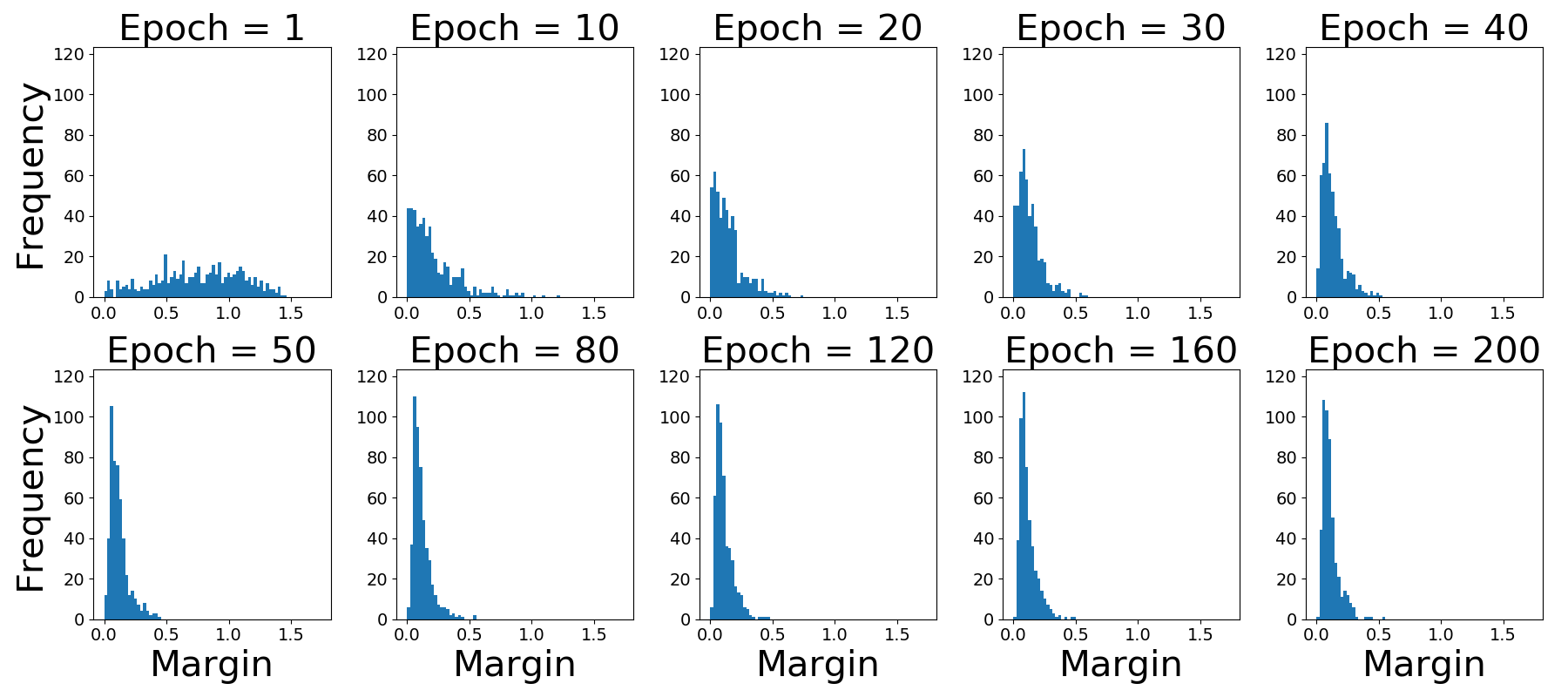}
\caption{Margin histograms of CIFAR-CNN on training set}
\label{fig:cifarcnn_train_hist}
\end{subfigure}
\begin{subfigure}{0.48\textwidth}
\centering
\includegraphics[width = \textwidth, height = 0.20\textheight]{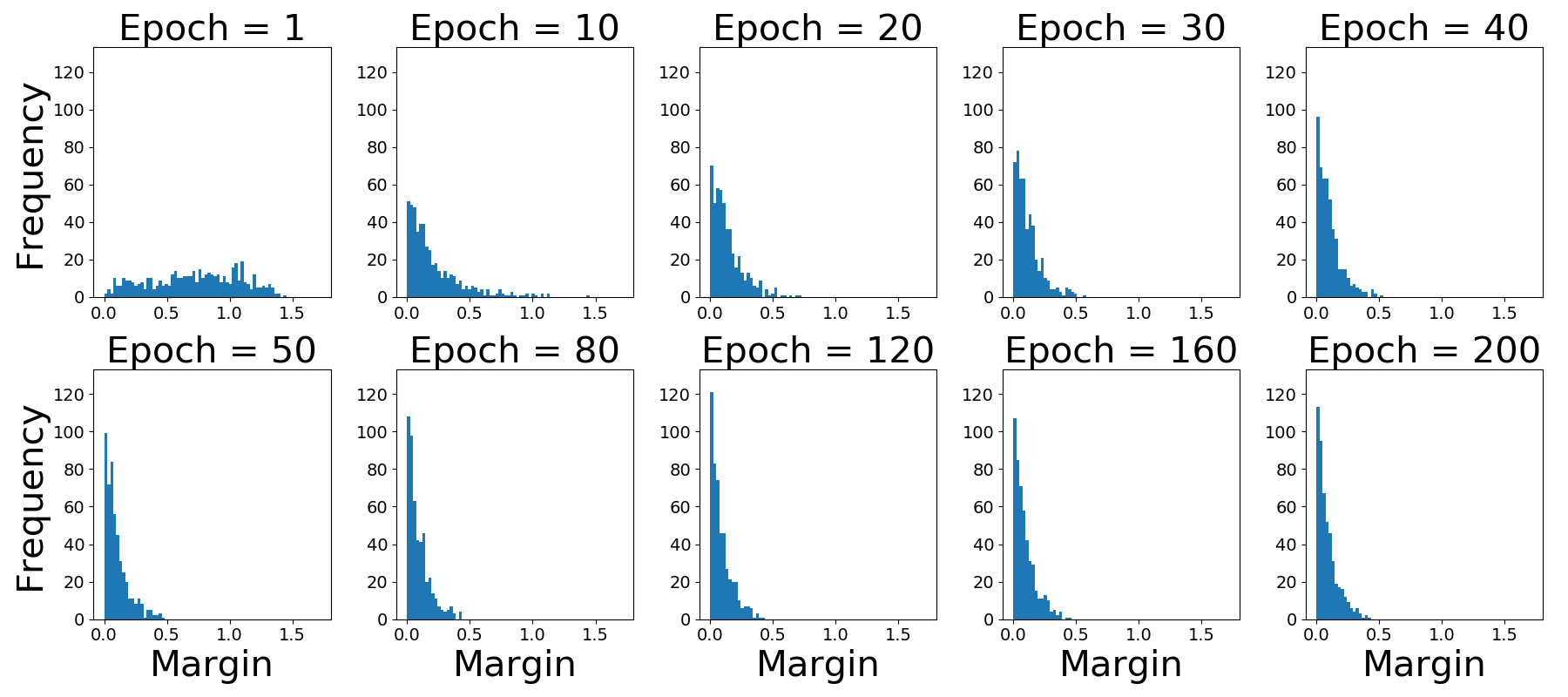}
\caption{Margin histograms of CIFAR-CNN on test set}
\label{fig:cifarcnn_test_hist}
\end{subfigure}
\caption{Margin histograms of CIFAR models at different epochs during training. \textbf{Top:} Histograms of CIFAR-LR. \textbf{Mid:} Histograms of CIFAR-MLP. \textbf{Bottom:} Histograms of CIFAR-CNN.}
\label{fig:cifar_hist}
\end{figure}

\clearpage

\section{Detailed Experiment Setting}


\textbf{Model Architecture:}
\begin{itemize}
\item \textbf{MNIST-LR and CIFAR-LR:} Multiclass logistic regression.
\item \textbf{MNIST-MLP and CIFAR-MLP:} Neural network with one hidden layer. The hidden layer has $1024$ neurons and relu activation.
\item \textbf{MNIST-CNN and MNIST-CNN:} Two AlexNet-Like convolutional neural networks with a slight difference. They are exactly the same as the models used by \textcite{WengZCYSGHD18, CarliniWagner17a}. Dropout is used to alleviate overfitting on CIFAR10. No data augmentation is used.
\item \textbf{ConvNet:} A convolutional neural network with $5$ convolutional layers, used in \citep{yan2018deep}.
\item \textbf{LeNet:} LeNet, consists of two convolutional layers and two fully connected layers.
\item \textbf{LeNetSmall:} A convolutional neural network with similar structure to LeNet, but with fewer filers. This model is same as the one used in \citep{croce2019provable}. On CIFAR10, data augmentation (random crop and random flip) is used.
\end{itemize}

\textbf{Standard Training:} We use SGD (learning rate $0.01$) with momentum ($0.9$) and nestorv to train all six models. Batch size is set to $128$. Linear models (LR) are trained for $1000$ epochs. Nonlinear models (MLP, CNN) are trained for $200$ epochs.

\textbf{Average Margin Regularizer:} $\lambda = 0.1$ and $\beta = 10^{-3}$ are used for all models, while different $\tau$'s are used for different models. $\tau$ is tuned such that it is on the same order as the Lipschitz constant of the model. MNIST-LR and MNIST-MLP use $\tau = 5$. MNIST-CNN use $\tau = 10$. CIFAR-LR use $\tau = 5$. CIFAR-MLP and CIFAR-CNN use $\tau = 10$. The idea is using larger $\tau$ for deeper models. The spectral norm of weights across layers may not be exactly $1$, even after adding orthogonal constraint. In fact, the product of spectral norm will become larger as the model becomes deeper. Thus, one should use larger truncation parameter, since the range of margin in feature space may be larger. Following same idea, LeNet and LeNetSmall use $\tau = 40$; ConvNet uses $\tau = 50$.


In addition, for fair comparison, we set the same $\beta = 10^{-3}$ to train models with Lipschitz constant regularization.

\textbf{Adversarial Training:} We use PGD attack to perform adversarial training. We set the number of iterations $40$ and step size $0.01$. On MNIST, the perturbation budget $\epsilon$ is set to $2.0$; on CIFAR, $\epsilon$ is set to $0.4$.

\textbf{Attack Method:} Through out the experiment, we use projected gradient descent (PGD) attack to evaluate the robust accuracy. We set step size $0.01$ and number of iterations $1000$. Increasing the number of iterations does not change the robust accuracy much, thus $1000$ iterations is sufficient to generate strong adversarial examples.

\textbf{Approximating Distance}
When using \Cref{thm:approx_margin} to approximate the distance to decision boundary, we need to estimation the Lipchitz constant of the network in a small neighbourhood around the input. Following \citep{tjeng2018evaluating}, we simply take the maximum norm of gradient in that neighbourhood. Throughout the experiments, we set $r = 5$ and the sampling size equal to $1024\times 50$ (except for linear classifiers, whose sampling size is $1$). For linear logistic regression, we perform estimation in every epoch; for nonlinear classifiers, we only perform estimation at certain epochs: $1$, $10$, $20$, $30$, $40$, $50$, $80$, $120$, $160$ and $200$, where we perform more estimations at the beginning, as distances may change drastically in initial stages. Due to efficiency concerns, the estimation is performed on a subset of size $500$, which is randomly chosen (with a fixed random seed) from the original training and test sets.

\if01
\section{MNIST-LR}
\begin{figure*}[h!]
\centering
\includegraphics[width = 0.8\textwidth, height = 0.12\textheight]{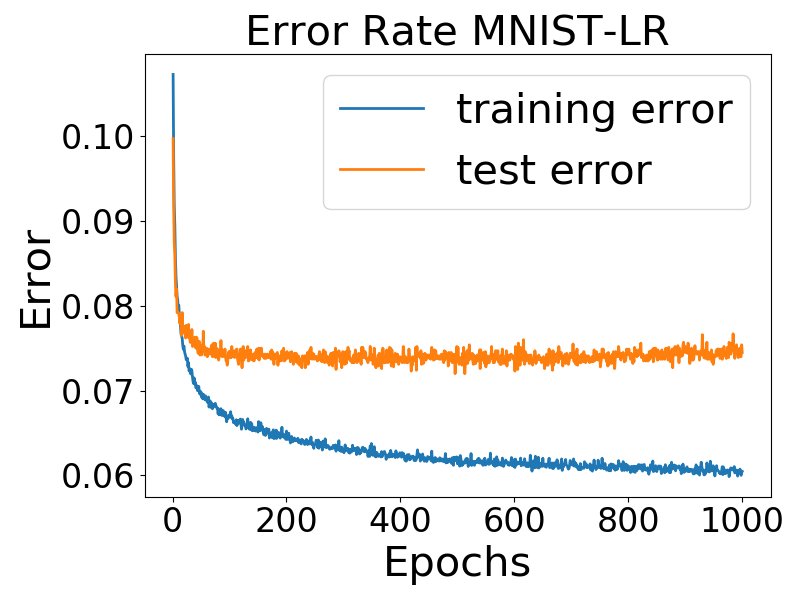}
\caption{Training loss and error rate of MNISTLR}
\end{figure*}

\begin{figure*}[h!]
\centering
\includegraphics[width = 0.8\textwidth, height = 0.12\textheight]{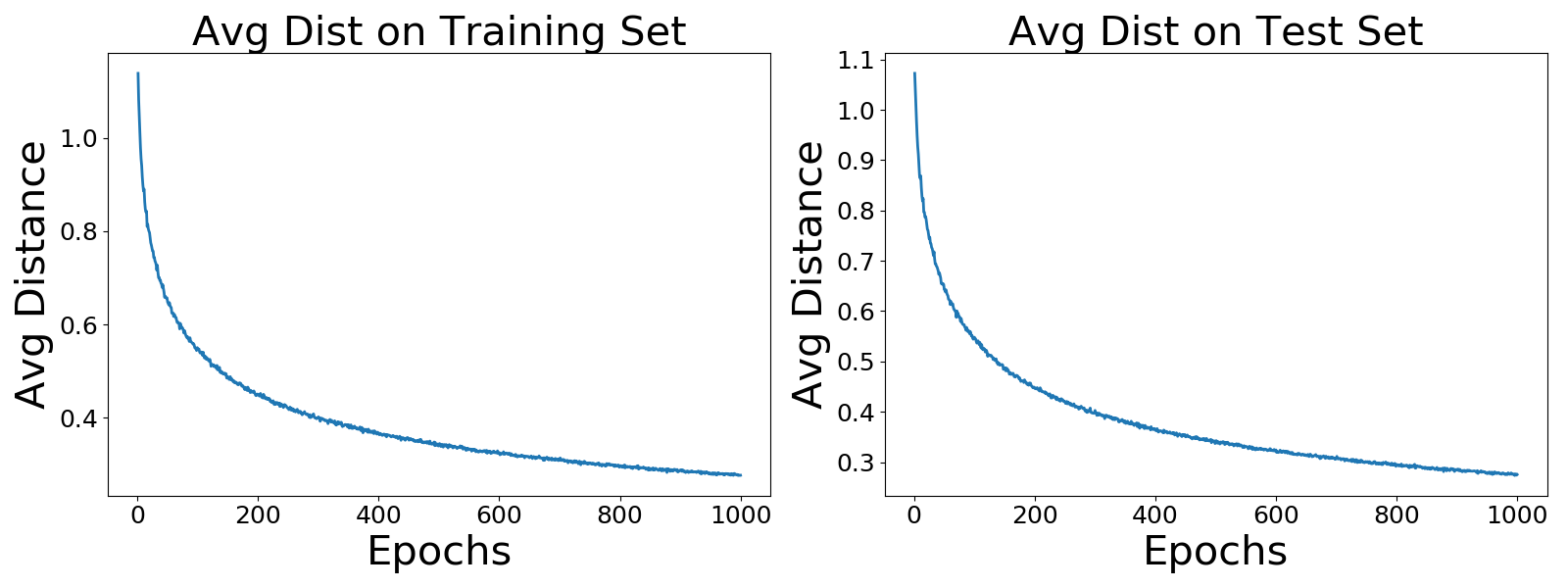}
\caption{Average distance during training of MNISTLR}
\end{figure*}

\begin{figure*}[h!]
\centering
\includegraphics[width = 0.8\textwidth, height = 0.12\textheight]{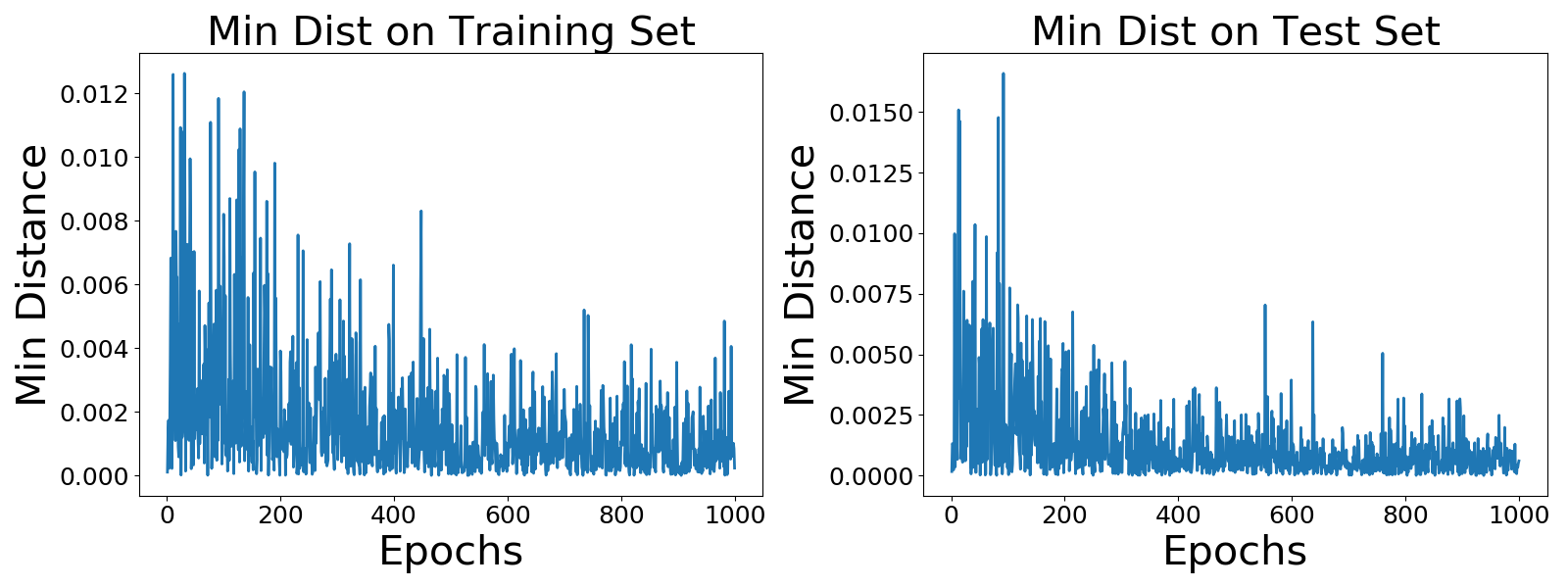}
\caption{Minimum distance during training of MNISTLR}
\end{figure*}

\begin{figure*}[h!]
\centering
\includegraphics[width = 0.8\textwidth, height = 0.20\textheight]{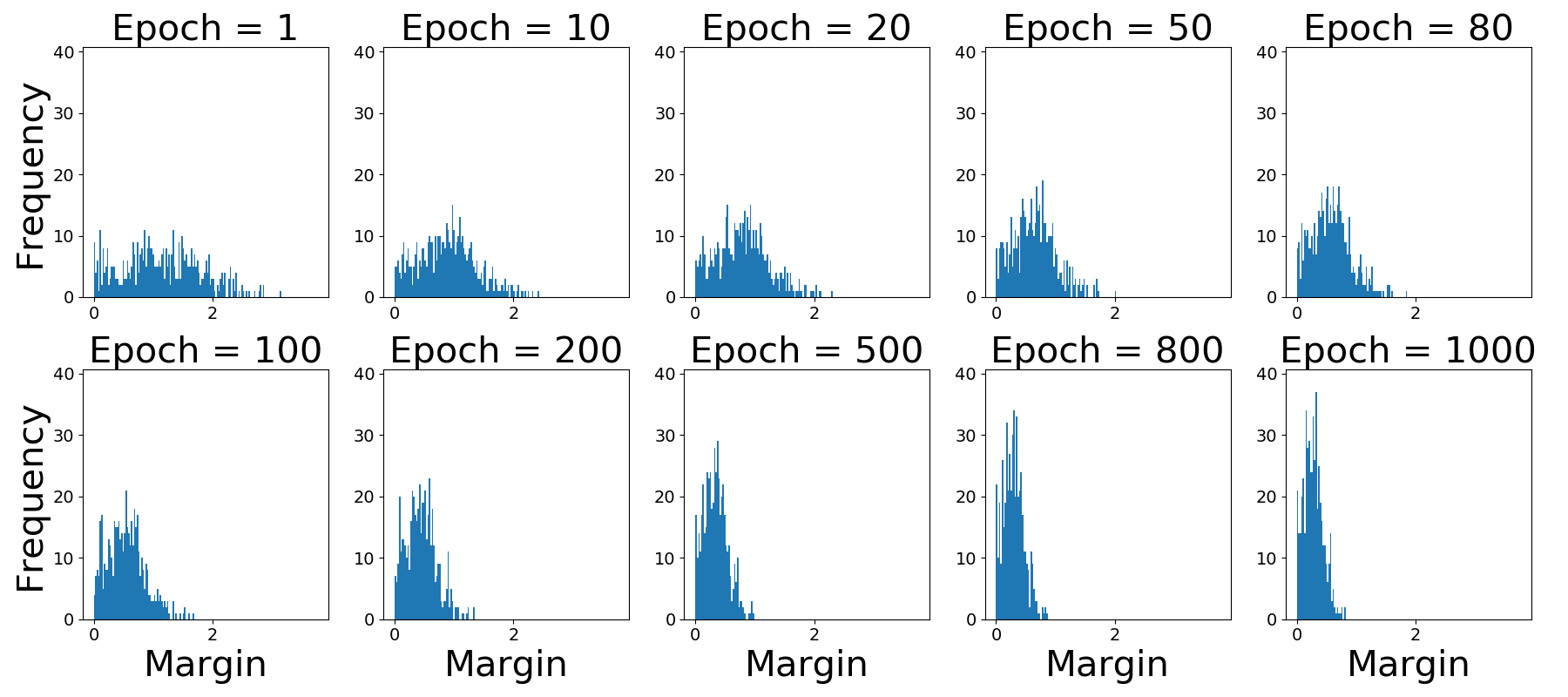}
\caption{Histogram of distance on training set during training of MNISTLR}
\end{figure*}

\begin{figure*}[h!]
\centering
\includegraphics[width = 0.8\textwidth, height = 0.20\textheight]{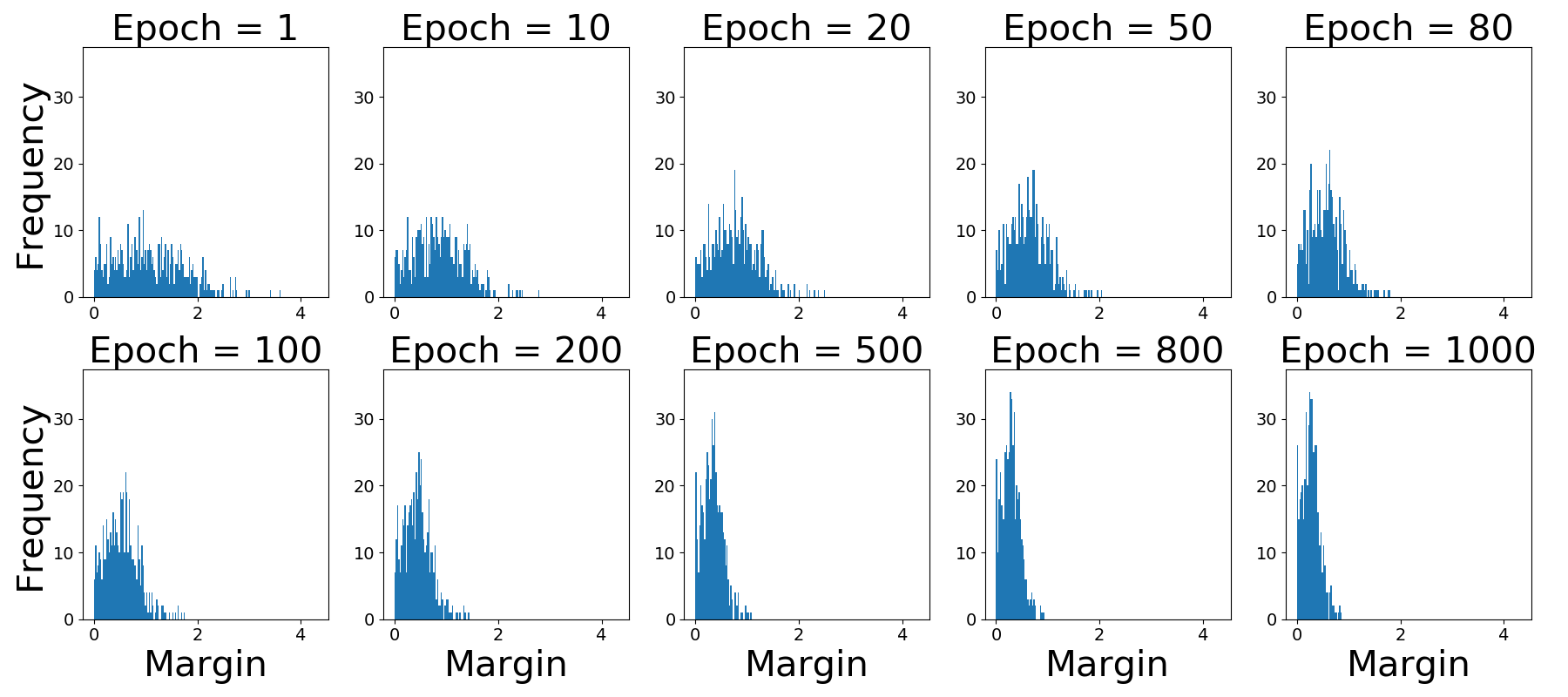}
\caption{Histogram of distance on test set during training of MNISTLR}
\end{figure*}

In addition, we visualize a set of adversarial examples.
\begin{figure*}[h!]
\centering
\includegraphics[width = 0.8\textwidth, height = 0.4\textheight]{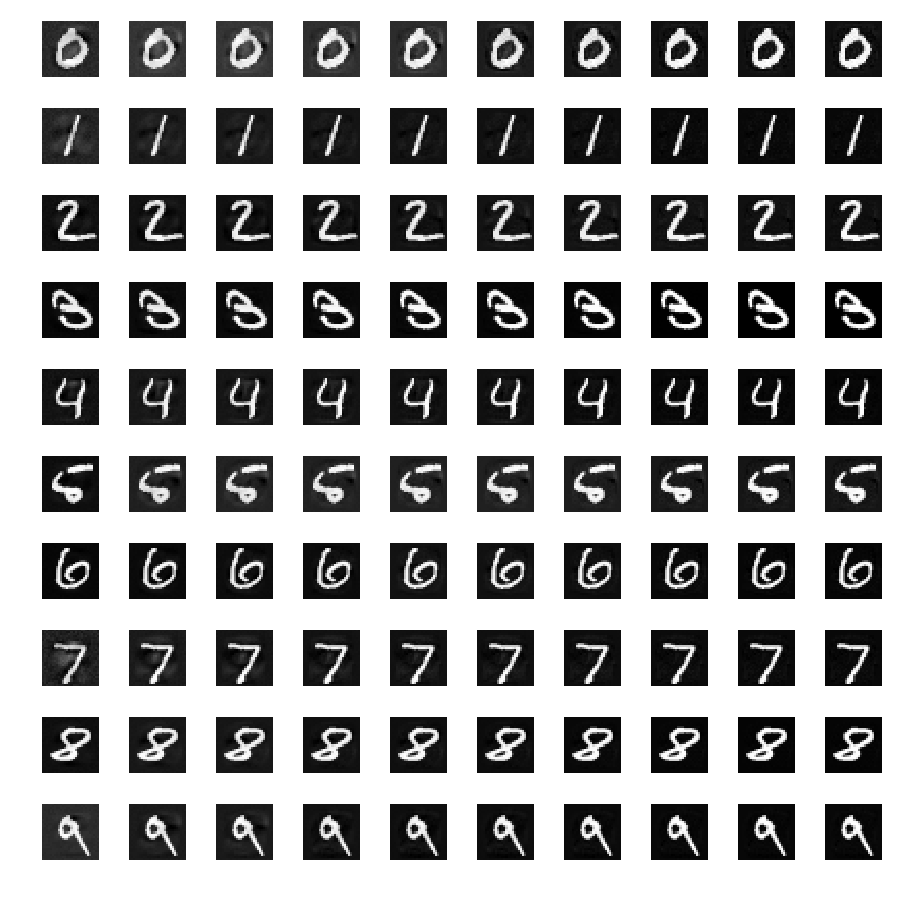}
\caption{Adversarial examples for MNISTLR during different epochs of training}
\end{figure*}

\section{MNIST-MLP}
\begin{figure*}[h!]
\centering
\includegraphics[width = 0.8\textwidth, height = 0.12\textheight]{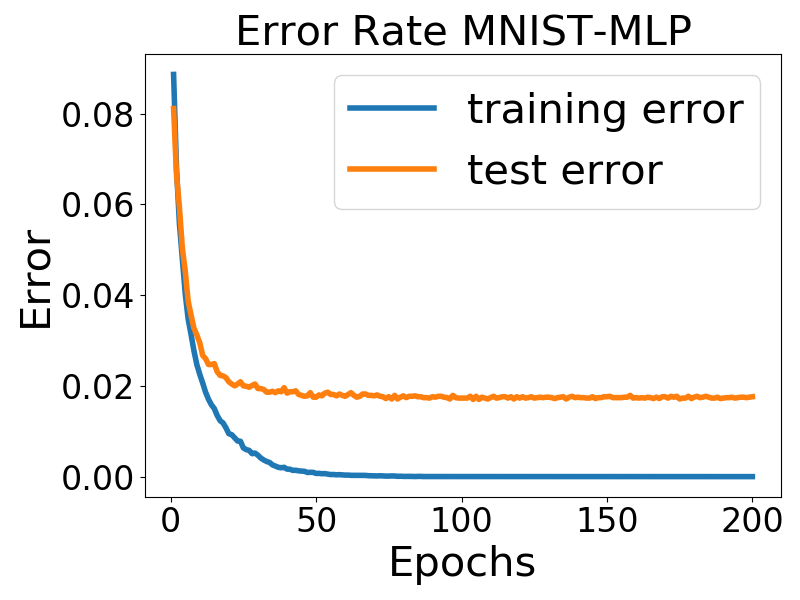}
\caption{Training loss and error rate of MNISTMLP}
\end{figure*}

\begin{figure*}[h!]
\centering
\includegraphics[width = 0.8\textwidth, height = 0.12\textheight]{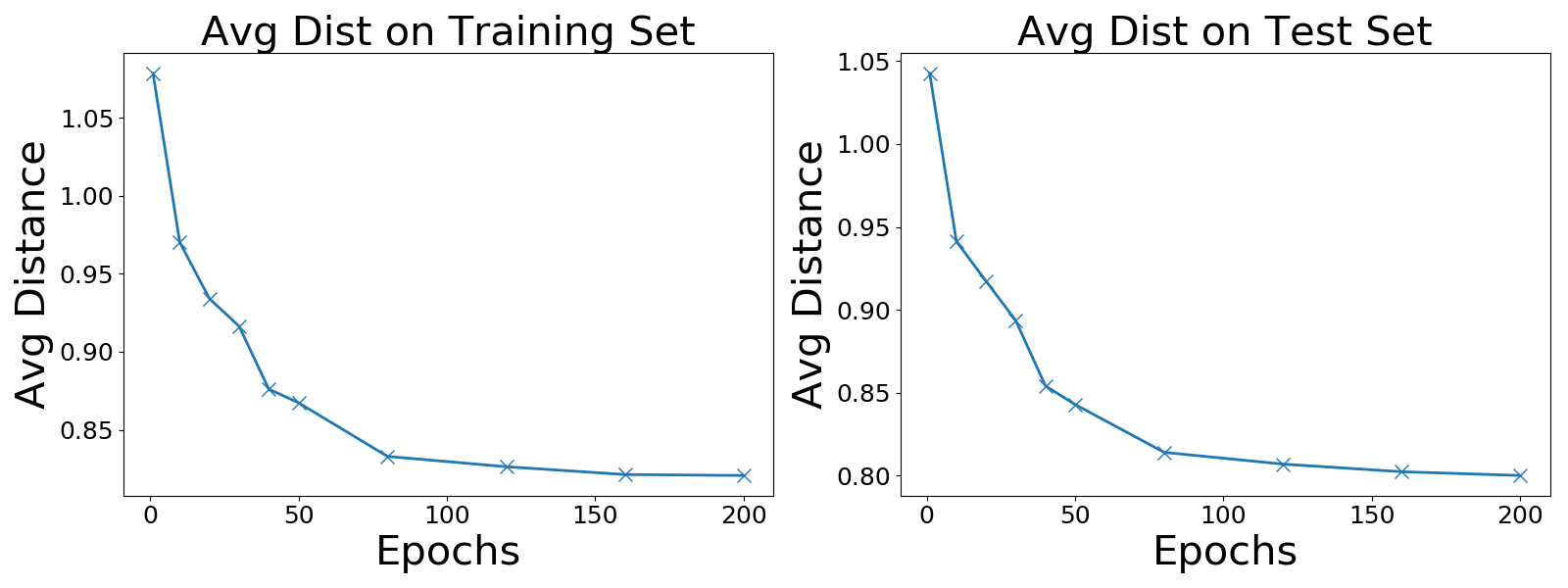}
\caption{Average distance during training of MNISTMLP}
\end{figure*}

\begin{figure*}[h!]
\centering
\includegraphics[width = 0.8\textwidth, height = 0.12\textheight]{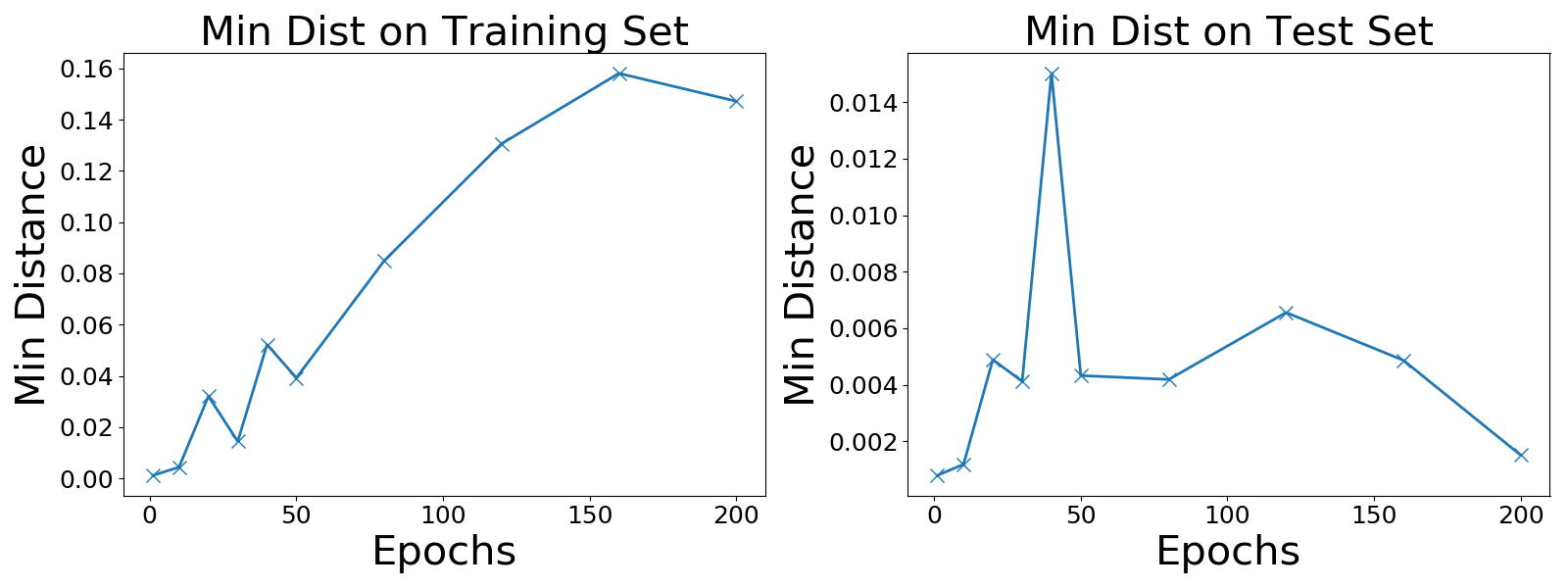}
\caption{Minimum distance during training of MNISTMLP}
\end{figure*}

\begin{figure*}[h!]
\centering
\includegraphics[width = 0.8\textwidth, height = 0.20\textheight]{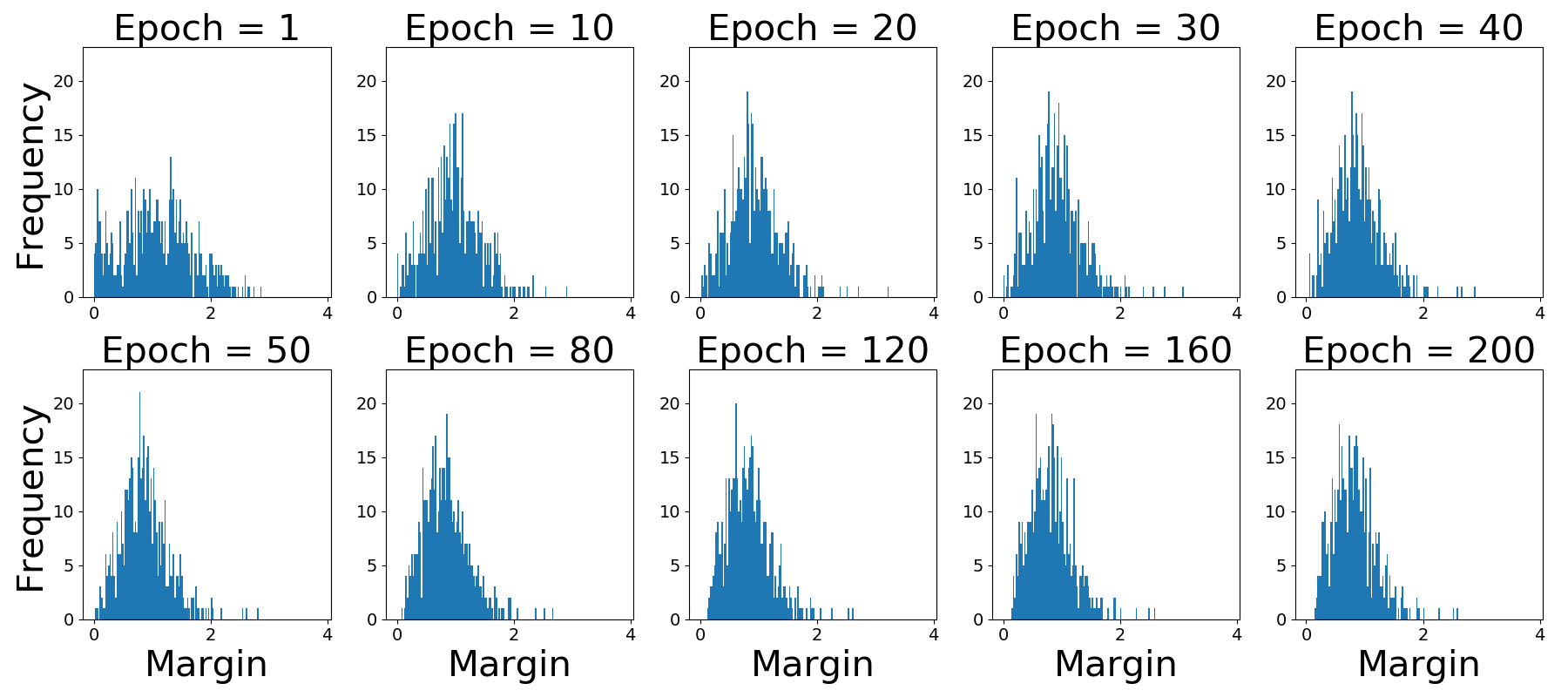}
\caption{Histogram of distance on training set during training of MNISTMLP}
\end{figure*}

\begin{figure*}[h!]
\centering
\includegraphics[width = 0.8\textwidth, height = 0.20\textheight]{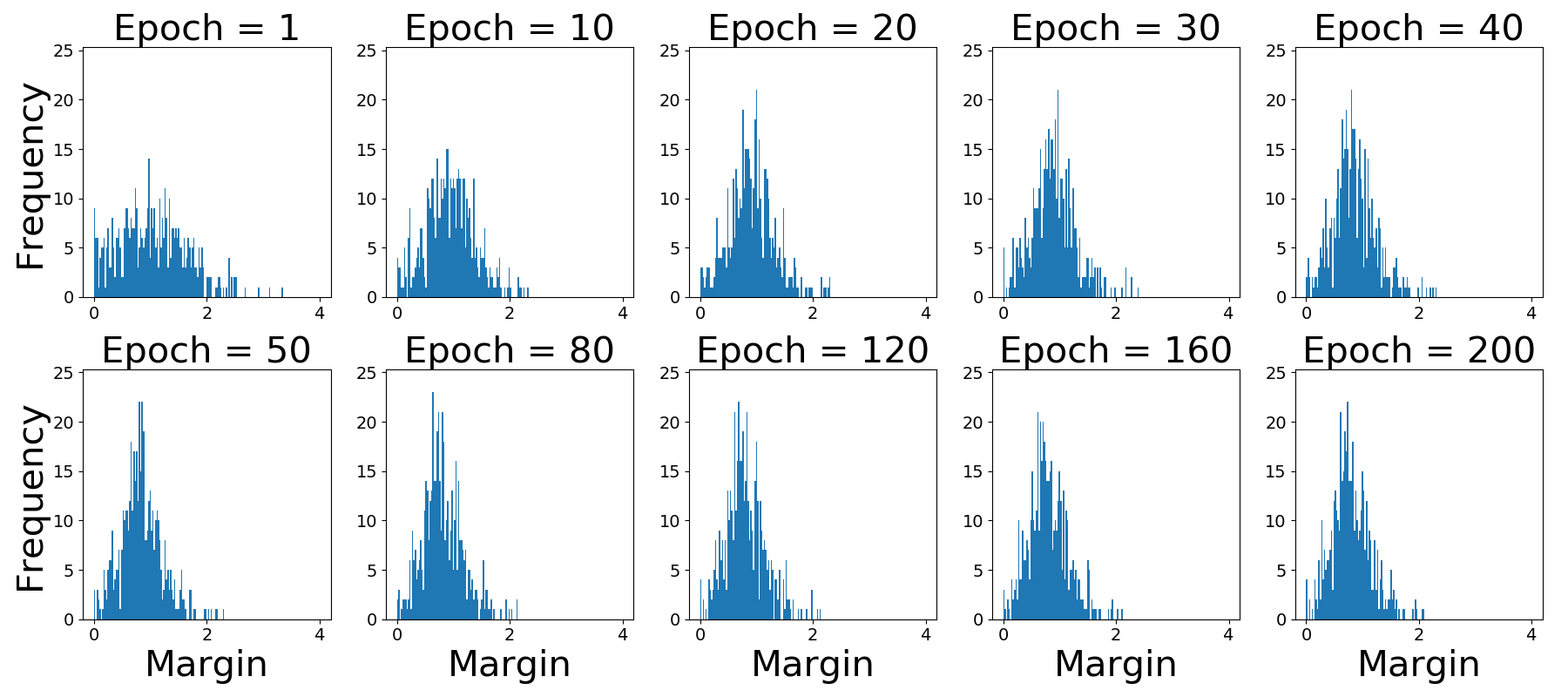}
\caption{Histogram of distance on test set during training of MNISTMLP}
\end{figure*}

\section{MNIST-CNN}
\begin{figure*}[h!]
\centering
\includegraphics[width = 0.8\textwidth, height = 0.12\textheight]{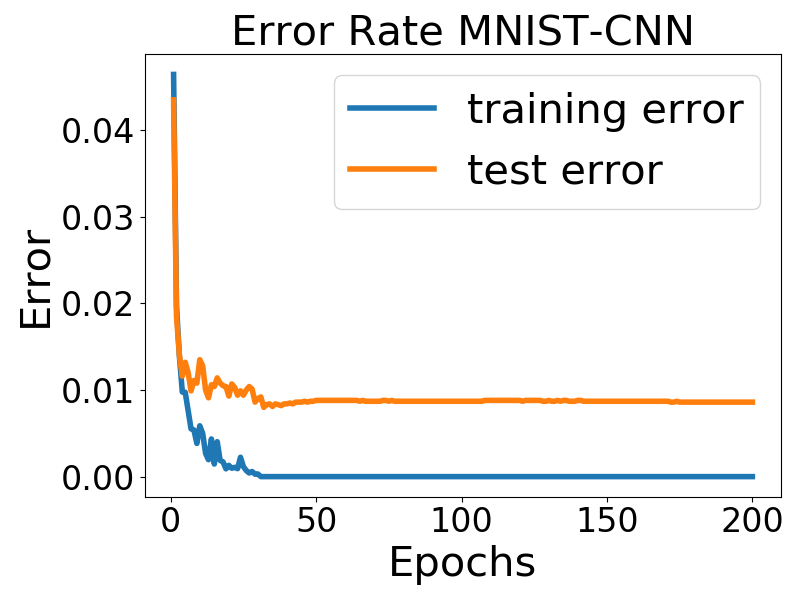}
\caption{Training loss and error rate of MNISTCNN}
\end{figure*}

\begin{figure*}[h!]
\centering
\includegraphics[width = 0.8\textwidth, height = 0.12\textheight]{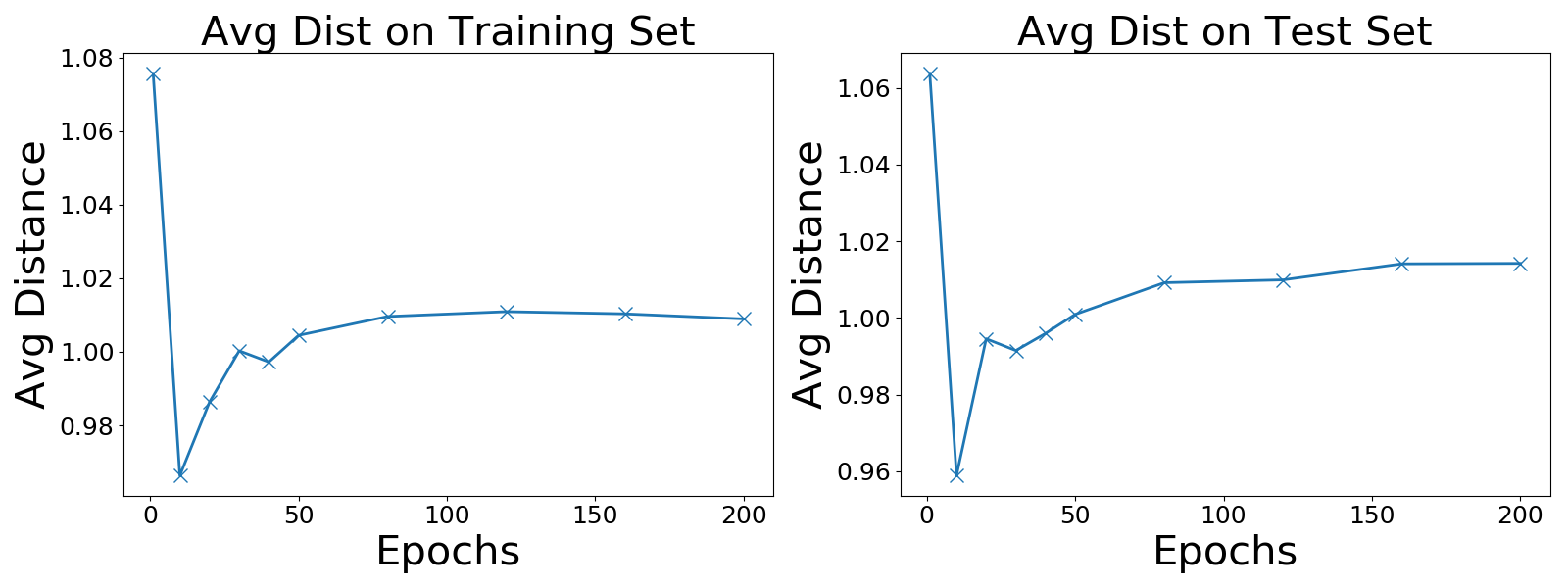}
\caption{Average distance during training of MNISTCNN}
\end{figure*}

\begin{figure*}[h!]
\centering
\includegraphics[width = 0.8\textwidth, height = 0.12\textheight]{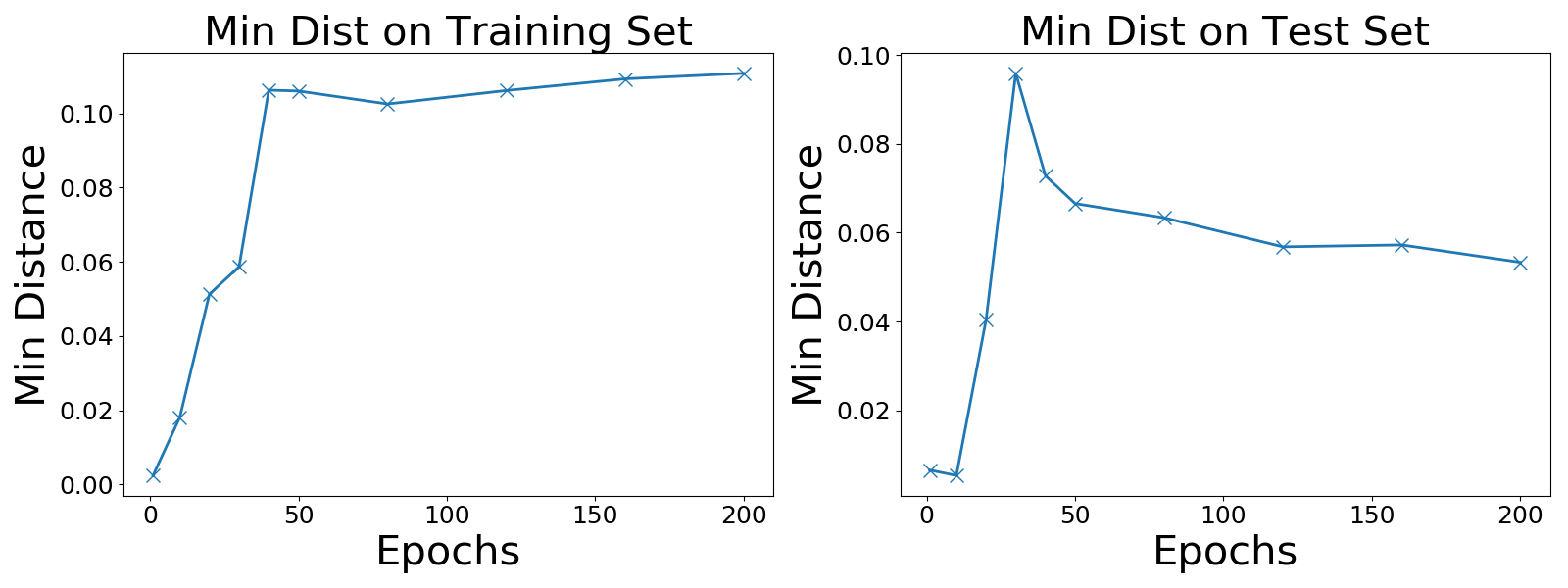}
\caption{Minimum distance during training of MNISTCNN}
\end{figure*}

\begin{figure*}[h!]
\centering
\includegraphics[width = 0.8\textwidth, height = 0.20\textheight]{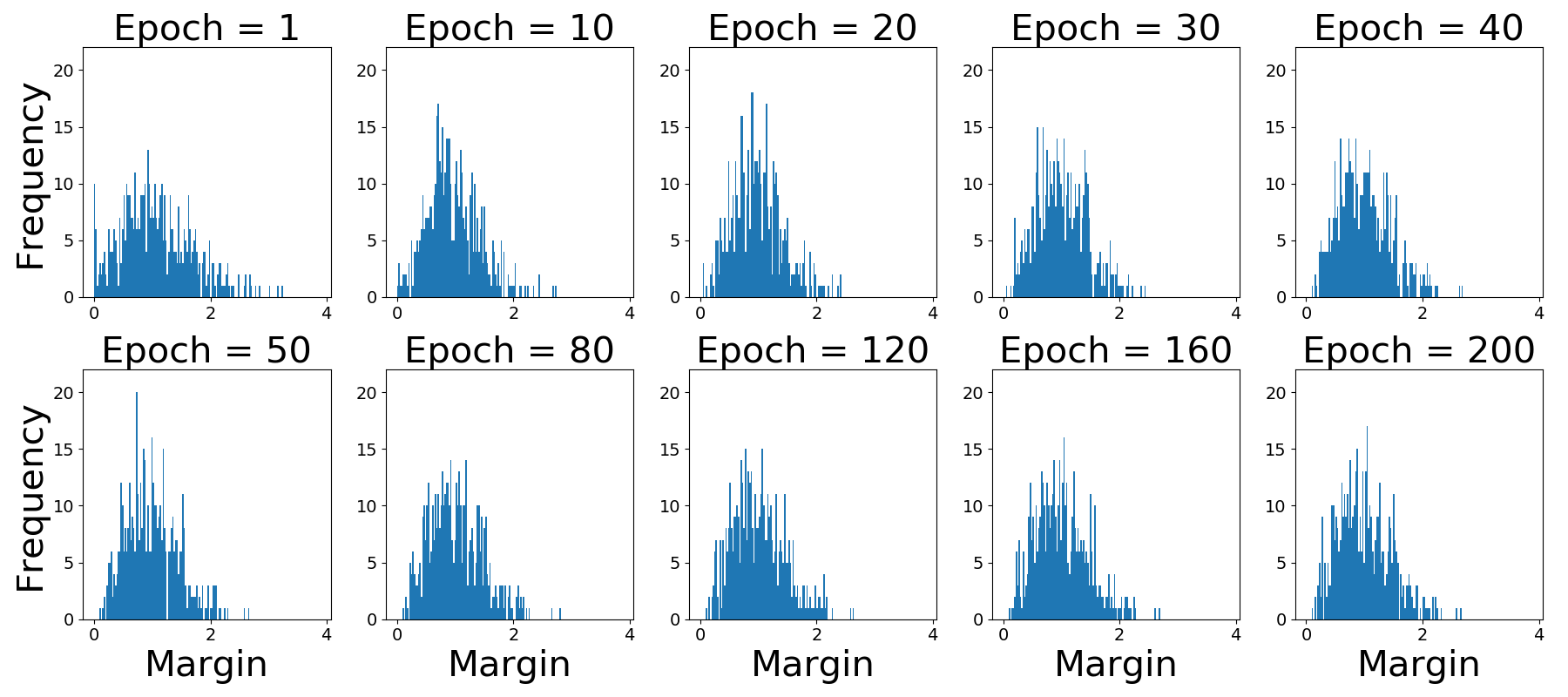}
\caption{Histogram of distance on training set during training of MNISTCNN}
\end{figure*}

\begin{figure*}[h!]
\centering
\includegraphics[width = 0.8\textwidth, height = 0.20\textheight]{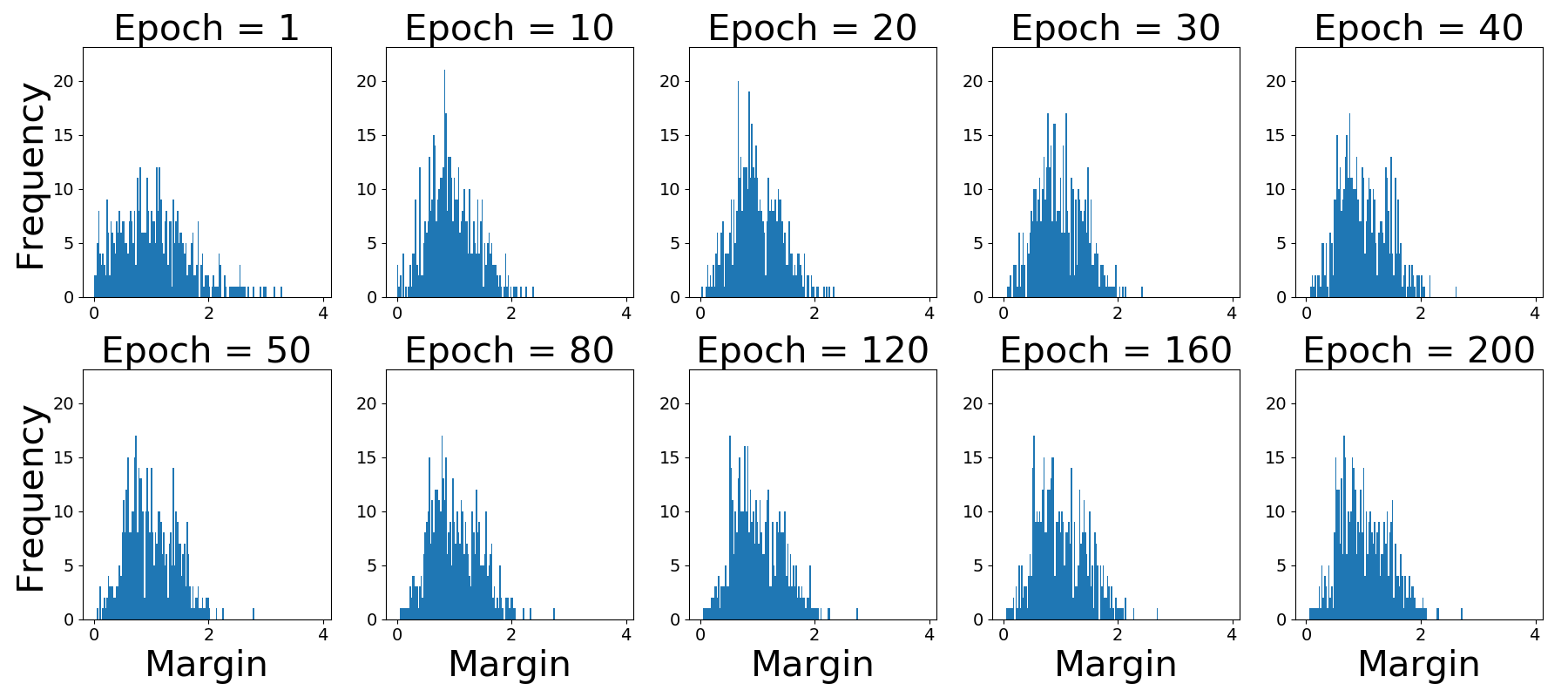}
\caption{Histogram of distance on test set during training of MNISTCNN}
\end{figure*}

\section{CIFAR-LR}
\begin{figure*}[h!]
\centering
\includegraphics[width = 0.8\textwidth, height = 0.12\textheight]{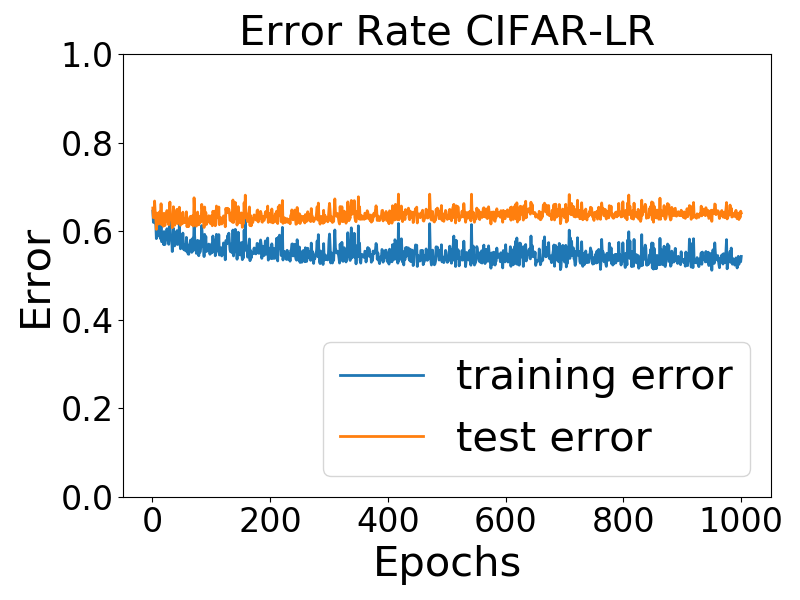}
\caption{Training loss and error rate of CIFARLR}
\end{figure*}

\begin{figure*}[h!]
\centering
\includegraphics[width = 0.8\textwidth, height = 0.12\textheight]{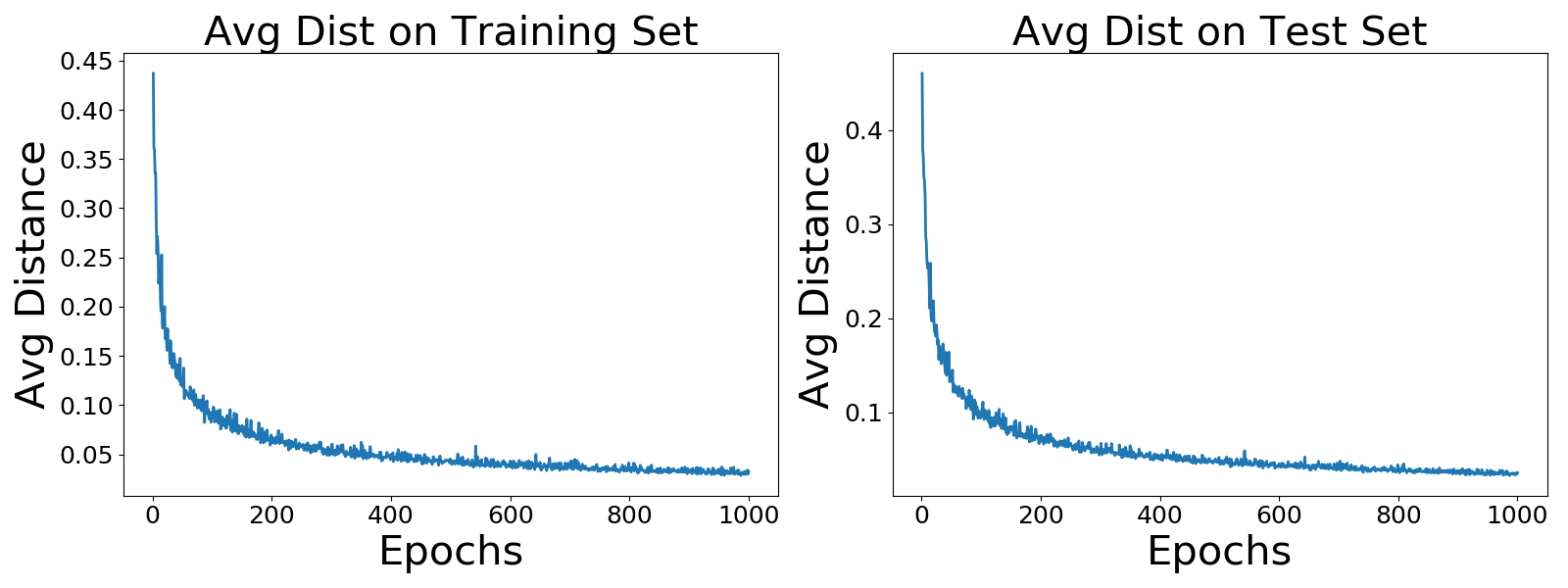}
\caption{Average distance during training of CIFARLR}
\end{figure*}

\begin{figure*}[h!]
\centering
\includegraphics[width = 0.8\textwidth, height = 0.12\textheight]{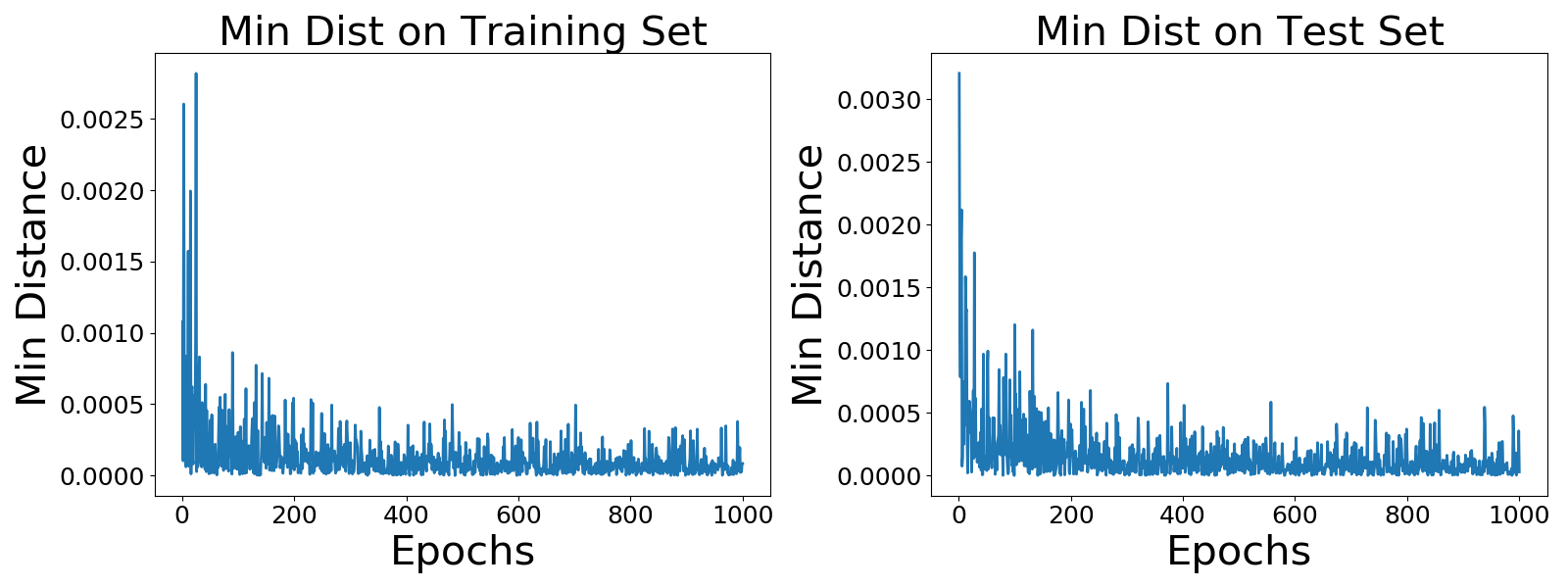}
\caption{Minimum distance during training of CIFARLR}
\end{figure*}

\begin{figure*}[h!]
\centering
\includegraphics[width = 0.8\textwidth, height = 0.20\textheight]{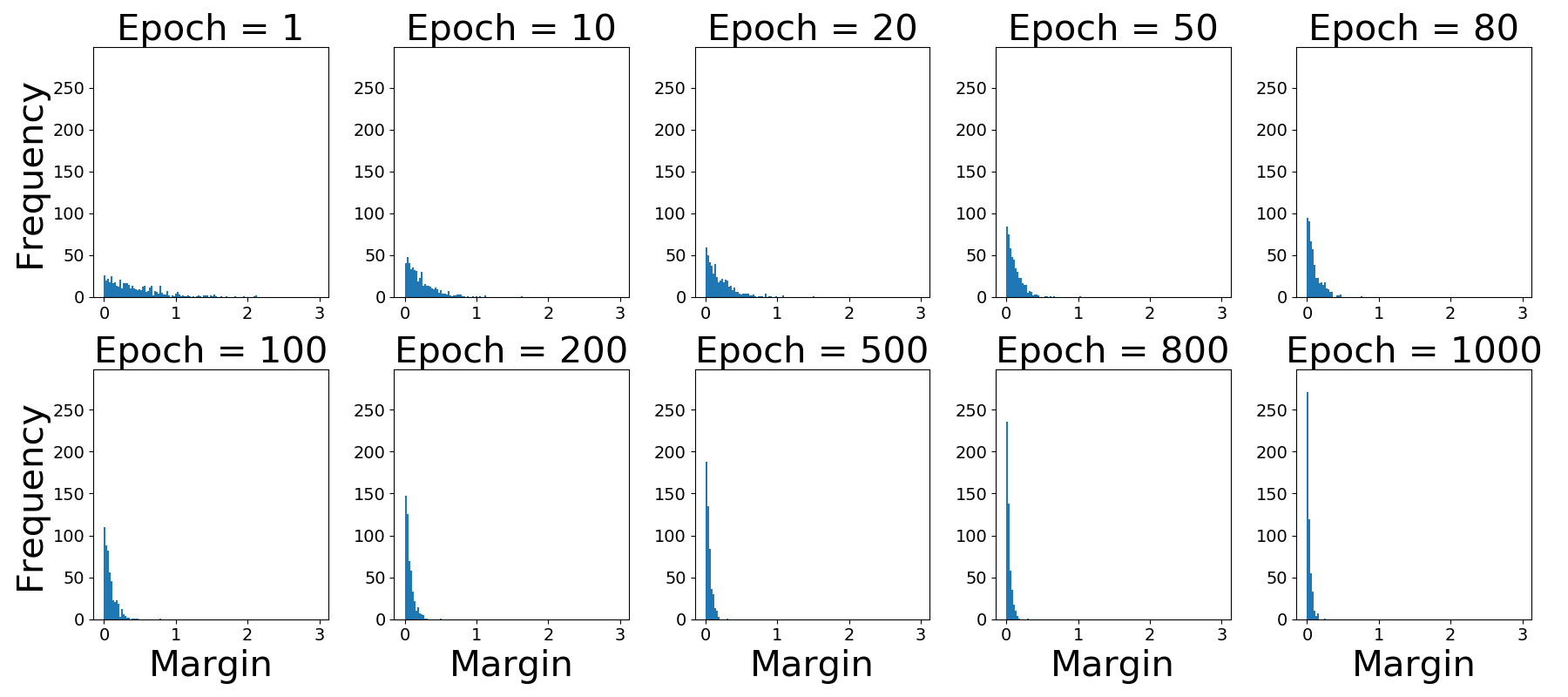}
\caption{Histogram of distance on training set during training of CIFARLR}
\end{figure*}

\begin{figure*}[h!]
\centering
\includegraphics[width = 0.8\textwidth, height = 0.20\textheight]{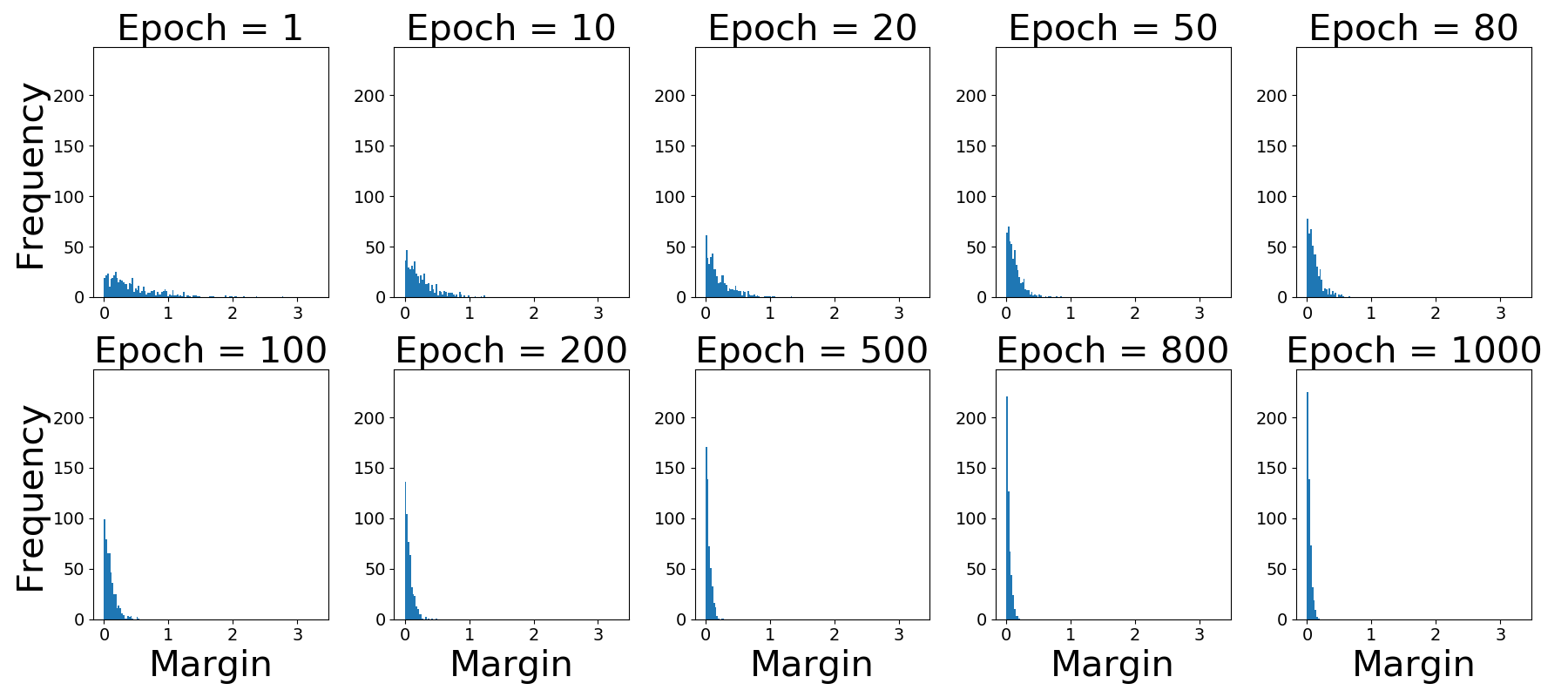}
\caption{Histogram of distance on test set during training of CIFARLR}
\end{figure*}

\section{CIFAR-MLP}
\begin{figure*}[h!]
\centering
\includegraphics[width = 0.8\textwidth, height = 0.12\textheight]{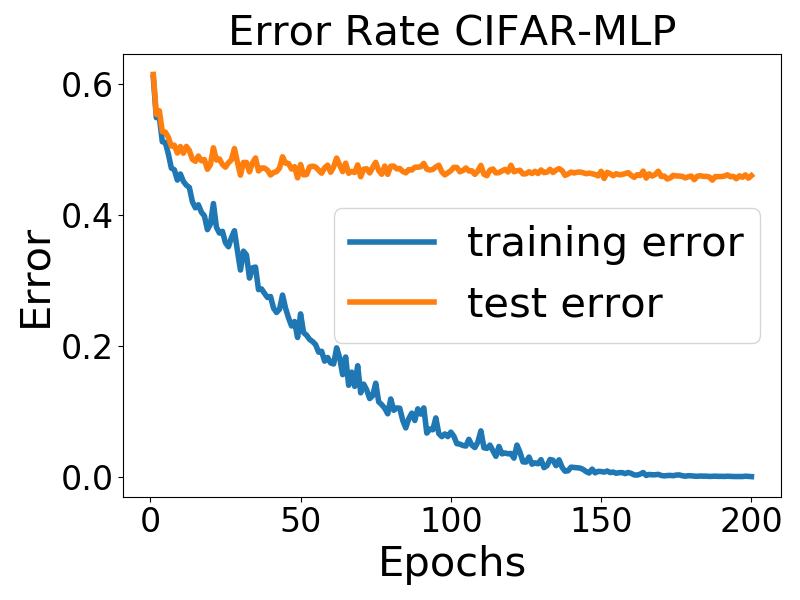}
\caption{Training loss and error rate of CIFARMLP}
\end{figure*}

\begin{figure*}[h!]
\centering
\includegraphics[width = 0.8\textwidth, height = 0.12\textheight]{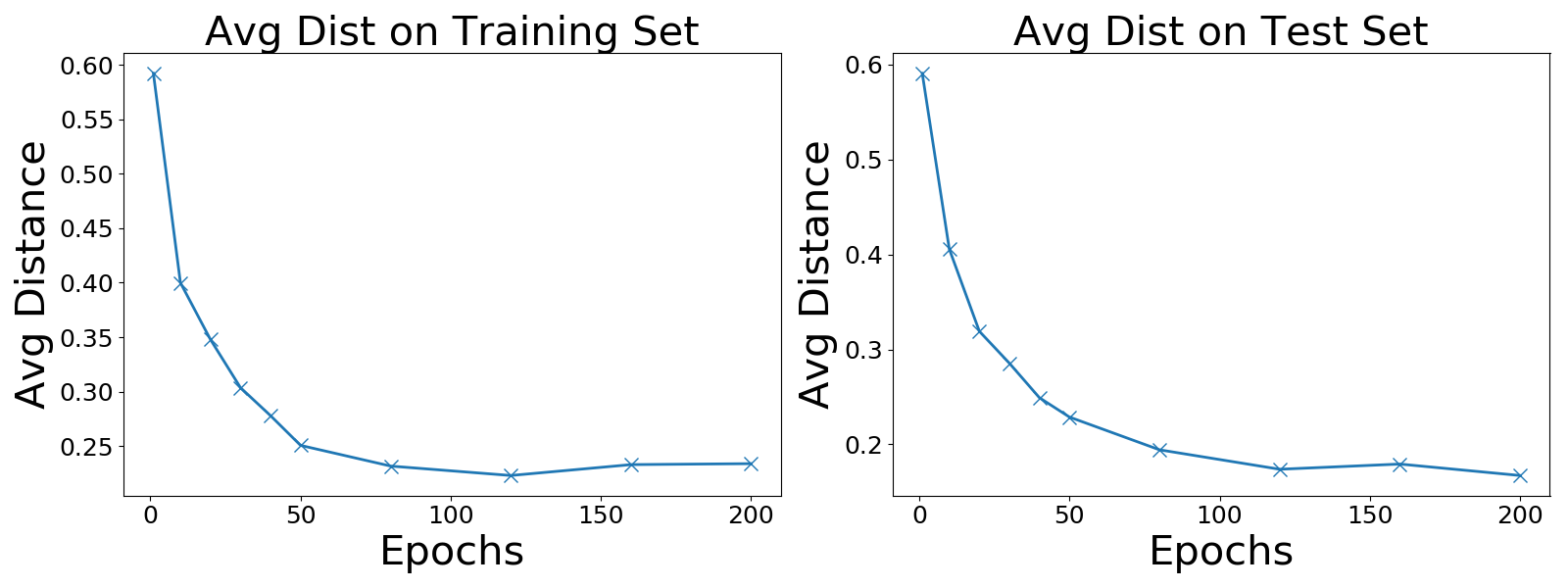}
\caption{Average distance during training of CIFARMLP}
\end{figure*}

\begin{figure*}[h!]
\centering
\includegraphics[width = 0.8\textwidth, height = 0.12\textheight]{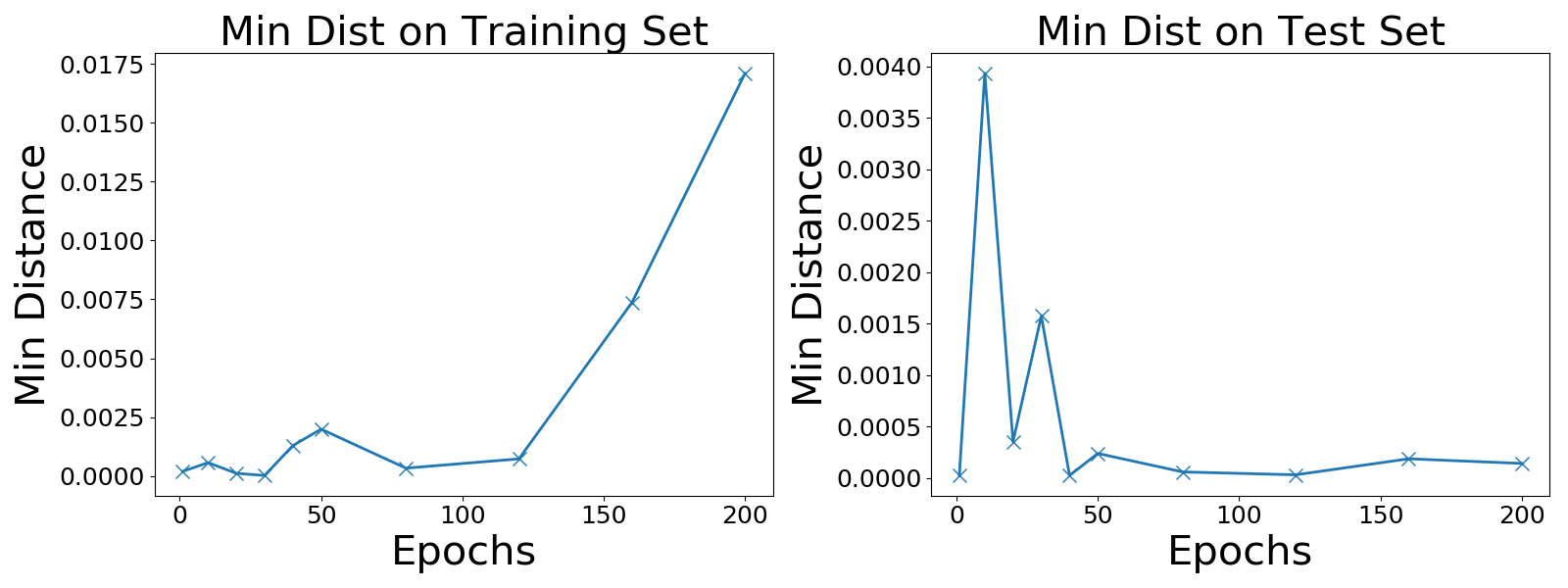}
\caption{Minimum distance during training of CIFARMLP}
\end{figure*}

\begin{figure*}[h!]
\centering
\includegraphics[width = 0.8\textwidth, height = 0.20\textheight]{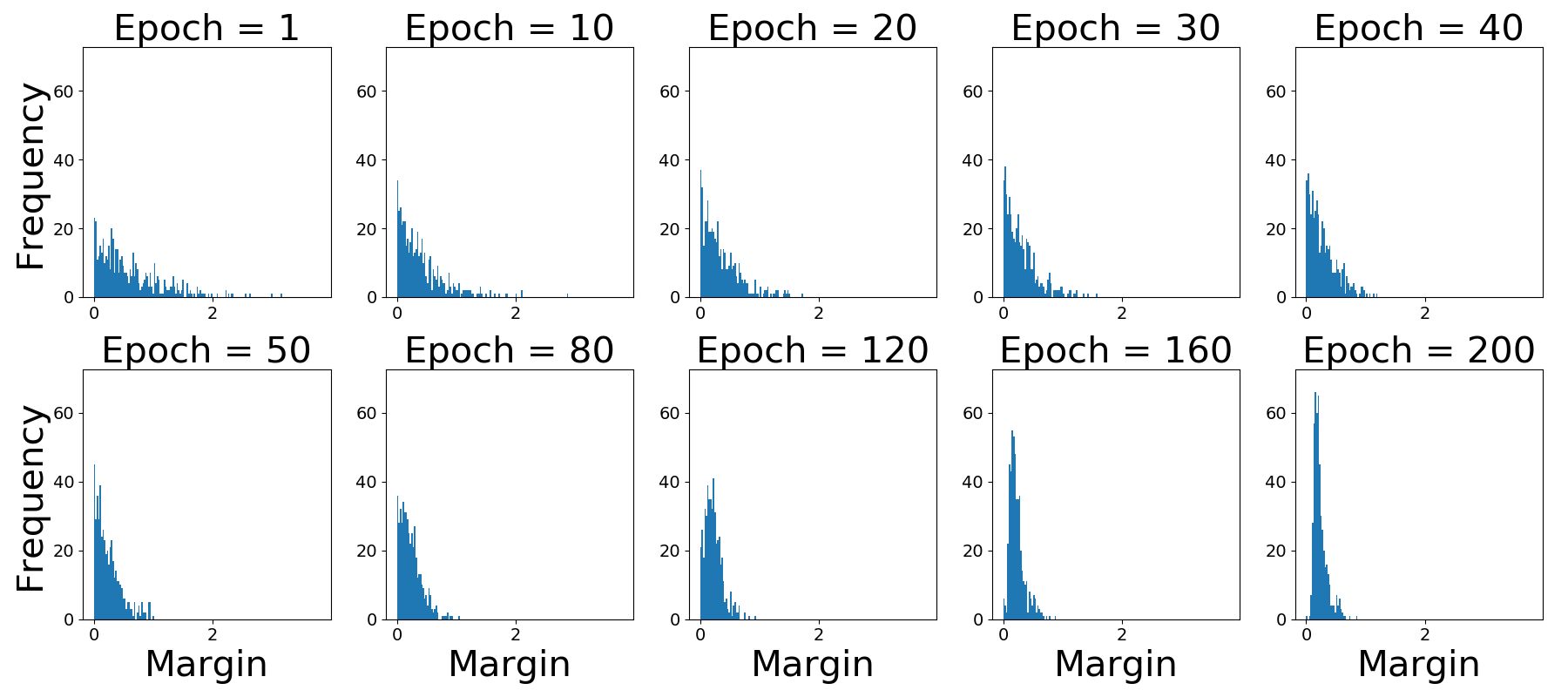}
\caption{Histogram of distance on training set during training of CIFARMLP}
\end{figure*}

\begin{figure*}[h!]
\centering
\includegraphics[width = 0.8\textwidth, height = 0.20\textheight]{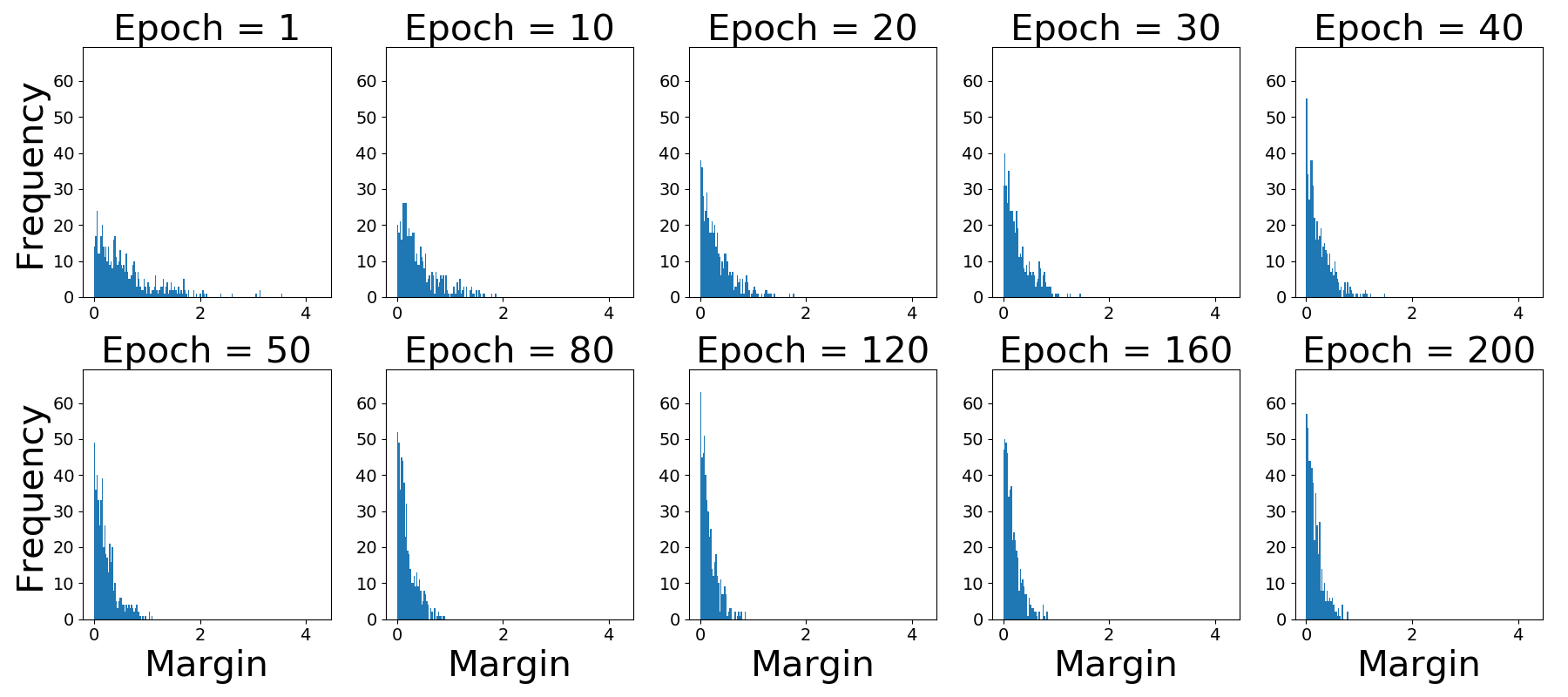}
\caption{Histogram of distance on test set during training of CIFARMLP}
\end{figure*}

\section{CIFAR-CNN}
\begin{figure*}[h!]
\centering
\includegraphics[width = 0.8\textwidth, height = 0.12\textheight]{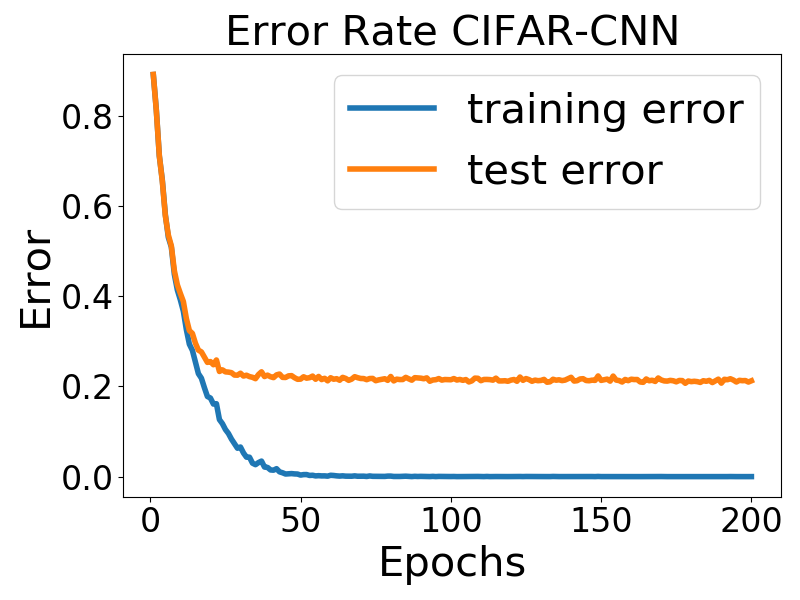}
\caption{Training loss and error rate of CIFARCNN}
\end{figure*}

\begin{figure*}[h!]
\centering
\includegraphics[width = 0.8\textwidth, height = 0.12\textheight]{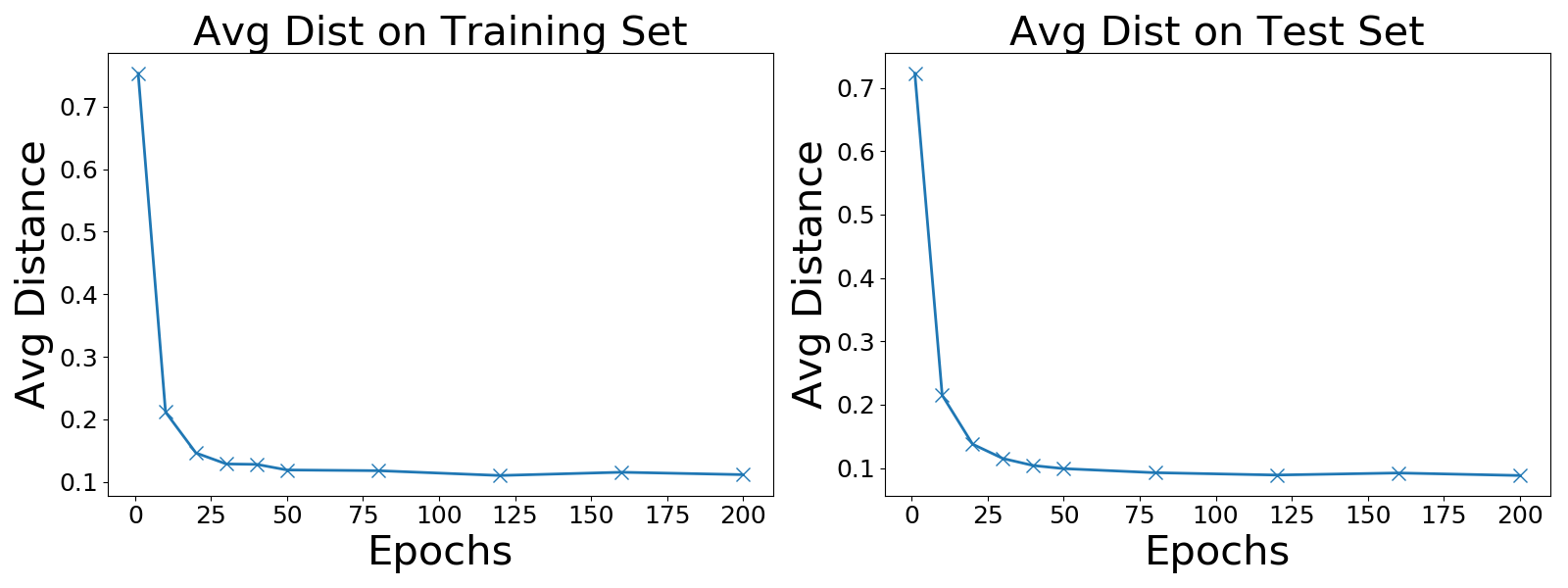}
\caption{Average distance during training of CIFARCNN}
\end{figure*}

\begin{figure*}[h!]
\centering
\includegraphics[width = 0.8\textwidth, height = 0.12\textheight]{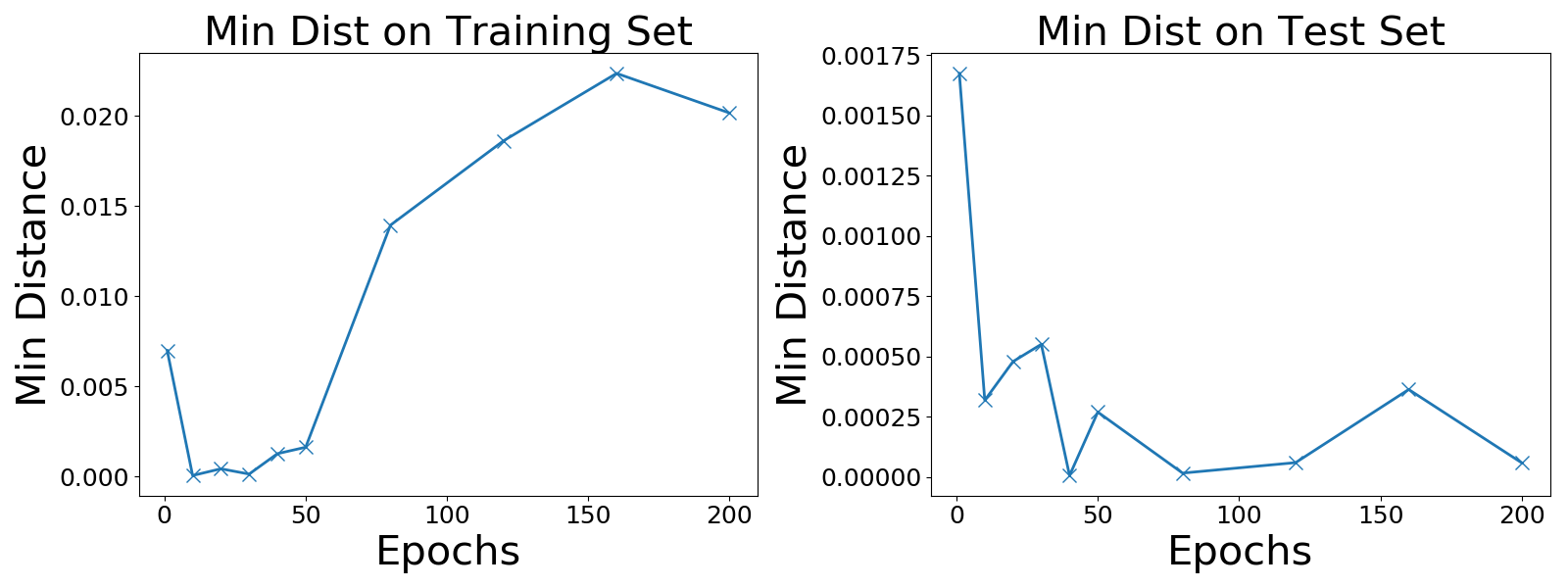}
\caption{Minimum distance during training of CIFARCNN}
\end{figure*}

\begin{figure*}[h!]
\centering
\includegraphics[width = 0.8\textwidth, height = 0.20\textheight]{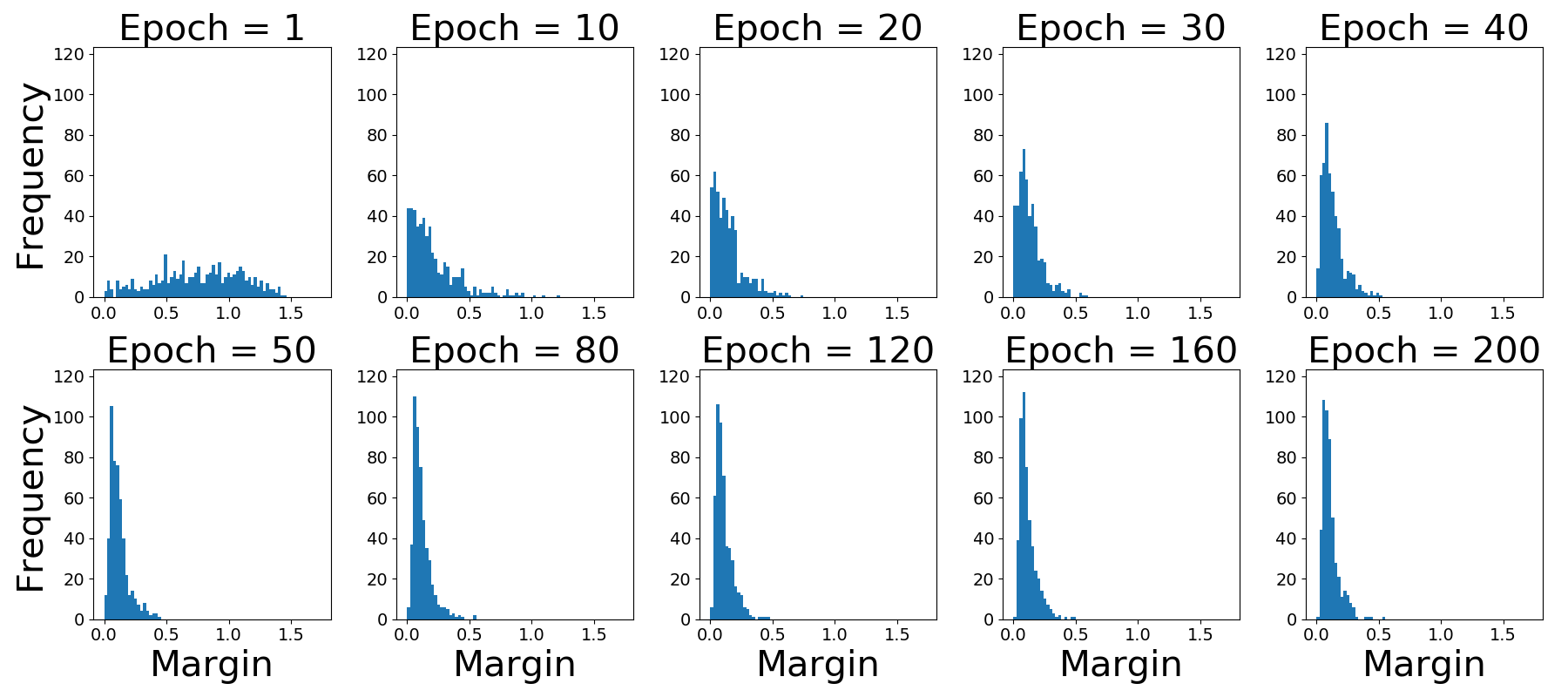}
\caption{Histogram of distance on training set during training of CIFARCNN}
\end{figure*}

\begin{figure*}[h!]
\centering
\includegraphics[width = 0.8\textwidth, height = 0.20\textheight]{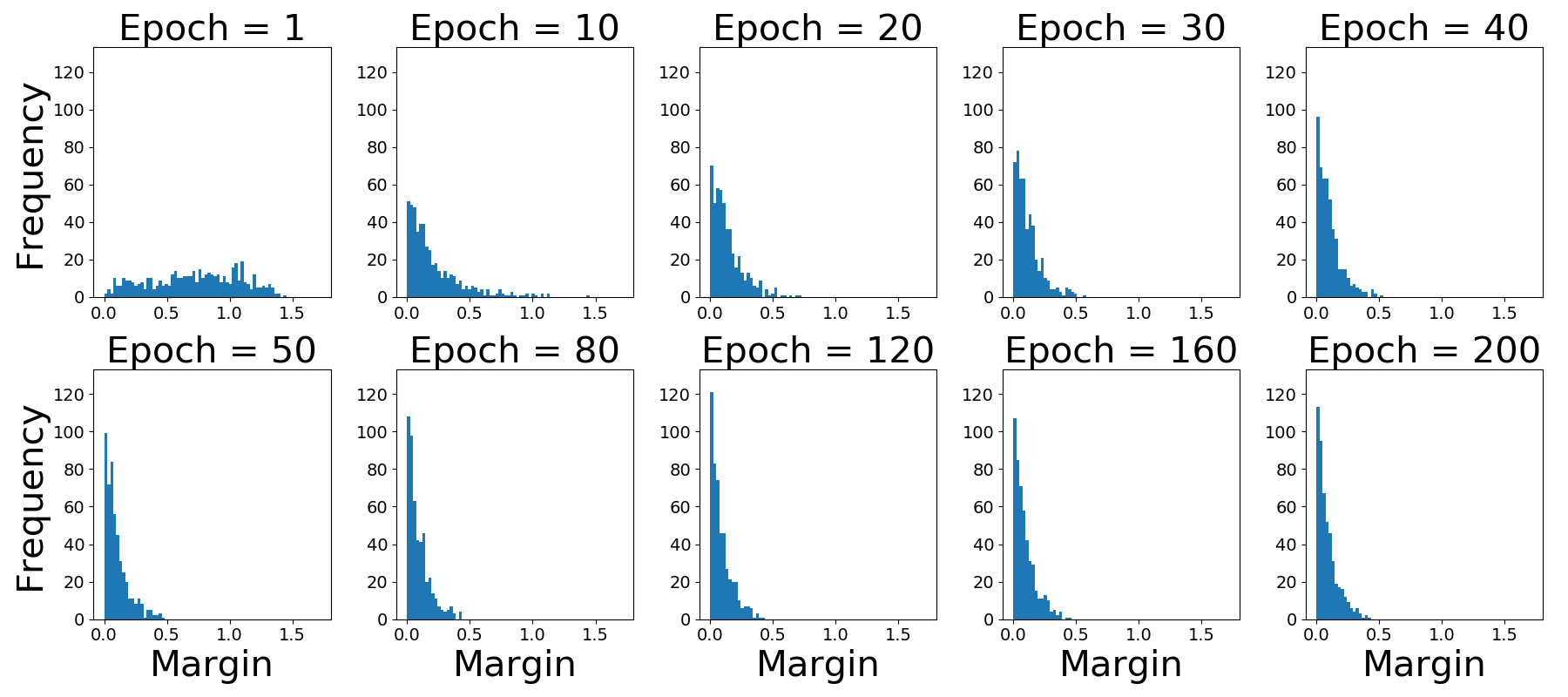}
\caption{Histogram of distance on test set during training of CIFARCNN}
\end{figure*}

\section{Visualization of Adversarial Examples}

\section{Discussion}
So far, the experiment result is promising. We can observe that average distance keeps decreasing during training for all models (except for MNIST-CNN). Plots of histogram of distance confirm that the majority of data is pushed towards desicion boundary.

However, unlike binary linear case, there are two points that are different:
\begin{itemize}
\item We did not oberserve the increase of robustness at the very beginning. Instead, average distance decrease monotonically. This could be a result of insufficient numeber of estimations.
\item The minimum distance keeps oscillating, unlike increasing minimum distance in binary logistic regression. This may be a result of nonseparable dataset. However, several nonlinear models indeed reach zero training error after certain number of epochs. This may also be a result of insufficient number of estimations, i.e. the minimum distance may continue to increase after $200$ epochs.
\end{itemize}

Again, for nonlinear classifiers, we only use the Lipschitz lower bound as a approximation of distance to decision boundary. This may be the reason for the unsatisfactory result for MNIST-CNN.

\fi

\end{document}